\documentclass{article} 
\usepackage{tencent_ailab_tech_report}
\usepackage[colorlinks = true,
            linkcolor = blue,
            urlcolor  = blue,
            citecolor = blue,
            anchorcolor = blue]{hyperref}

\usepackage{graphicx}
\usepackage{caption}
\usepackage{subcaption}
\usepackage{booktabs} 

\usepackage[colorinlistoftodos]{todonotes}

\usepackage{listings}
\usepackage{xcolor}
\usepackage{natbib}
\usepackage{bbding}
\lstdefinelanguage{json}{
    basicstyle=\ttfamily\small,
    numbers=left,
    numberstyle=\tiny,
    stepnumber=1,
    numbersep=5pt,
    showstringspaces=false,
    breaklines=true,
    frame=lines,
    backgroundcolor=\color{gray!10},
    literate=
     *{0}{{{\color{blue}0}}}{1}
      {1}{{{\color{blue}1}}}{1}
      {2}{{{\color{blue}2}}}{1}
      {3}{{{\color{blue}3}}}{1}
      {4}{{{\color{blue}4}}}{1}
      {5}{{{\color{blue}5}}}{1}
      {6}{{{\color{blue}6}}}{1}
      {7}{{{\color{blue}7}}}{1}
      {8}{{{\color{blue}8}}}{1}
      {9}{{{\color{blue}9}}}{1}
      {:}{{{\color{red}:}}}{1}
      {,}{{{\color{red},}}}{1}
      {"}{{{\color{orange}"}}}{1}
      {[}{{{\color{green}[}}}{1}
      {]}{{{\color{green}]}}}{1}
      {\{}{{{\color{green}\{}}}{1}
      {\}}{{{\color{green}\}}}}{1}
}

\usepackage{microtype}
\usepackage{hyperref}
\usepackage{url}
\usepackage{booktabs}
\usepackage{enumitem}
\usepackage{multicol}
\usepackage{multirow}
\usepackage{CJKutf8}
\usepackage{amsmath}
\usepackage{siunitx}
\usepackage{floatflt}
\usepackage{wrapfig}
\usepackage{graphicx}
\usepackage{booktabs}
\usepackage{wrapfig}
\usepackage{authblk}
\usepackage{lipsum}
\usepackage{dsfont}
\usepackage{bbm}
\usepackage{microtype}
\usepackage{graphicx}
\usepackage{subcaption}
\usepackage{multirow}
\usepackage{booktabs} 
\usepackage{pifont} 
\usepackage{graphicx}  
\usepackage{subcaption}
\usepackage{colortbl}
\usepackage{hyperref}

\usepackage[ruled,vlined,linesnumbered,noend]{algorithm2e}
\usepackage{amsmath,amssymb,xcolor,setspace}

\SetKwInput{KwIn}{Require}
\SetKwInput{KwOut}{Return}
\SetKwComment{Comment}{$\triangleright$\ }{}
\SetKw{KwTo}{to}
\SetArgSty{textnormal}
\SetAlCapNameFnt{\normalsize\bfseries}
\SetAlCapFnt{\normalsize\bfseries}
\usepackage{marvosym}
\usepackage[most]{tcolorbox}
\newtcolorbox{alprompt}[1]{
        boxrule = 1pt,
        fontupper = \small\tt,
        fonttitle = \bf\color{black},
        arc = 2pt,
        rounded corners,
        colframe = black,
        colbacktitle = white!97!yellow,
        colback = white!97!yellow,
        title = #1,
}
\newtcolorbox{keyfindingbox}[1]{%
  enhanced,
  colback=gray!5,
  colframe=gray!60!black,
  colbacktitle=gray!40!black,
  coltitle=white,
  fonttitle=\bfseries,
  boxed title style={sharp corners},
  sharp corners,
  title=#1
}
\definecolor{darkgreen}{rgb}{0.0, 0.5, 0.0}
\definecolor{darkgray}{gray}{0.4}
\definecolor{maroon}{rgb}{0.5, 0.0, 0.0}
\definecolor{navy}{rgb}{0.0, 0.0, 0.5}
\definecolor{teal}{rgb}{0.0, 0.5, 0.5}

\definecolor{deepblue}{RGB}{41, 128, 185}

\definecolor{mylightgreen}{RGB}{144,238,144}
\definecolor{mylightblue}{RGB}{173,216,230}

\definecolor{outerboxcolor}{gray}{0.90} 
\definecolor{innerboxcolor}{rgb}{1,1,1}

\definecolor{nred}{RGB}{196, 38, 11}
\definecolor{ngreen}{RGB}{18, 141, 21}
\definecolor{nblue}{RGB}{41, 52, 190}

\usepackage{listings}

\usepackage{array}
\usepackage{amsmath}
\usepackage{mathtools}
\usepackage{amsthm}
\usepackage{arydshln}
\usepackage{aliascnt}
\usepackage[capitalize,noabbrev]{cleveref}
\usepackage{adjustbox} 
\usepackage{enumitem}
\usepackage{longtable}
\usepackage{tcolorbox}
\tcbuselibrary{skins}

\theoremstyle{plain}
\newtheorem{theorem}{Theorem}[section]
\crefname{theorem}{Theorem}{Theorems}
\Crefname{theorem}{Theorem}{Theorems}

\newaliascnt{proposition}{theorem}

\aliascntresetthe{proposition}
\crefname{proposition}{Proposition}{Propositions}
\Crefname{proposition}{Proposition}{Propositions}

\newaliascnt{lemma}{theorem}
\newtheorem{lemma}[lemma]{Lemma}
\aliascntresetthe{lemma}
\crefname{lemma}{Lemma}{Lemmas}
\Crefname{lemma}{Lemma}{Lemmas}

\newaliascnt{corollary}{theorem}
\newtheorem{corollary}[corollary]{Corollary}
\aliascntresetthe{corollary}
\crefname{corollary}{Corollary}{Corollaries}
\Crefname{corollary}{Corollary}{Corollaries}

\theoremstyle{definition}
\newaliascnt{definition}{theorem}

\aliascntresetthe{definition}
\crefname{definition}{Definition}{Definitions}
\Crefname{definition}{Definition}{Definitions}

\newaliascnt{assumption}{theorem}
\newtheorem{assumption}[assumption]{Assumption}
\aliascntresetthe{assumption}
\crefname{assumption}{Assumption}{Assumptions}
\Crefname{assumption}{Assumption}{Assumptions}

\theoremstyle{remark}
\newaliascnt{remark}{theorem}
\newtheorem{remark}[remark]{Remark}
\aliascntresetthe{remark}
\crefname{remark}{Remark}{Remarks}
\Crefname{remark}{Remark}{Remarks}

\DeclareMathOperator*{\argmin}{arg\,min}

\newcommand{\eg}{\textit{e.g.}}

\makeatletter
\renewcommand{\@fnsymbol}[1]{\ensuremath{\ifcase#1\or \dag\or \ddag\or \S\or \P\or \|\or **\fi}}
\makeatother

\colmfinalcopy

\title{
  \begin{minipage}{1.0\textwidth}
Group Distributionally Robust Optimization-Driven Reinforcement Learning for LLM Reasoning   
\end{minipage}
}

\author{
Kishan Panaganti, Zhenwen Liang, Wenhao Yu,
Haitao Mi, Dong Yu\\ 
\vspace{-1em}
Tencent AI Lab in Bellevue, WA, USA\\
Correspondence to: \texttt{kpb@global.tencent.com}}

\begin{document}

\maketitle
\vspace{-2em}
\begin{abstract}
Recent progress in Large Language Model (LLM) reasoning is increasingly driven by the refinement of post-training loss functions and alignment strategies\footnotemark. However, standard Reinforcement Learning (RL) paradigms like Group Relative Policy Optimization (GRPO) remain constrained by \emph{static uniformity}: uniform prompt sampling and a fixed number of rollouts per prompt. For heterogeneous, heavy-tailed reasoning data, this creates structural inefficiencies that waste compute on already-solved patterns while under-training the long tail of hard problems. To address this, we propose \textbf{Multi-Adversary Group Distributionally Robust Optimization (GDRO)}, an optimization-first framework that moves beyond uniform reasoning models by dynamically adapting the training distribution.

We introduce an \emph{Online Difficulty Classifier} that partitions prompts into \emph{dynamic pass@k difficulty groups}. We then propose \textbf{two independent GDRO games} for post-training: (1) \textbf{Prompt-GDRO}, which employs an EMA-debiased multiplicative-weights bandit sampler to target the \textit{intensive} difficulty margin and upweight persistently hard groups without frequency bias; and (2) \textbf{Rollout-GDRO}, which uses a shadow-price controller to reallocate rollouts across groups, maximizing gradient variance reduction on hard tasks under a fixed \emph{mean} budget (compute-neutral). We provide no-regret guarantees for Prompt-GDRO (via an entropy-regularized GDRO surrogate) and a variance-proxy analysis motivating a square-root optimal rollout allocation for Rollout-GDRO. We validate our framework on the DAPO 14.1k dataset using Qwen3-Base models. Prompt-GDRO and Rollout-GDRO achieve average relative gains of \textbf{+10.6\%} and \textbf{+10.1\%}, respectively, in pass@8 accuracy across 1.7B, 4B, and 8B scales compared to the GRPO baseline. Qualitative analysis shows an emergent curriculum: the adversaries shift resources to the evolving reasoning frontier, enhancing the reasoning model's performance.
\end{abstract}

\footnotetext{\scriptsize Adam Marblestone (paraphrased) on Dwarkesh's podcast (\href{https://www.youtube.com/watch?v=_9V_Hbe-N1A}{YouTube}): ``The brain's secret sauce is its loss functions, not its architecture.'' }

\begin{figure*}[h]
    \centering
    \vspace{-2em}
    \includegraphics[width=\textwidth]{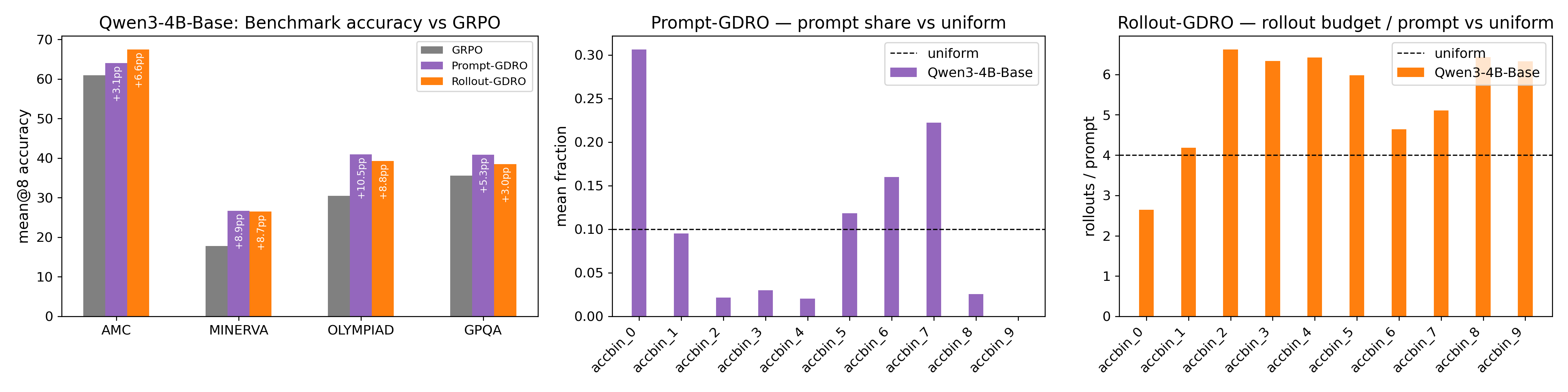}
    \vspace{-2em}
    \caption{\textbf{Beyond Uniform Reasoning—A Multi-Adversary Post-Training Framework.} 
    Plots on the right represent training steps tail averages ($\ge$60th percentile) capturing the curriculum.
    (Left) Our framework significantly outperforms the standard GRPO baseline across mathematical reasoning benchmarks via dynamic adaptation. 
    (Center) \textbf{Prompt-GDRO:} The adversary learns a non-uniform curriculum. Instead of uniform sampling (dashed line), probability mass (purple bars) shifts to the ``reasoning frontier'' (bins 6--8), targeting the specific difficulty level where learning is most efficient. 
    (Right) \textbf{Rollout-GDRO:} The adversary optimizes compute utility. Under a fixed global budget (dashed line), it reallocates rollouts (orange bars) from solved tasks (bin 0) to high-variance tasks, scaling exploration with difficulty. Note: Bars represent rollout count per prompt (policy intensity).}
    \label{fig:frontpage_results}
\end{figure*}

\section{Introduction}
\label{sec:intro}

\begin{figure}[t]
\centering
\vspace{-1em}
\includegraphics[width=1\linewidth]{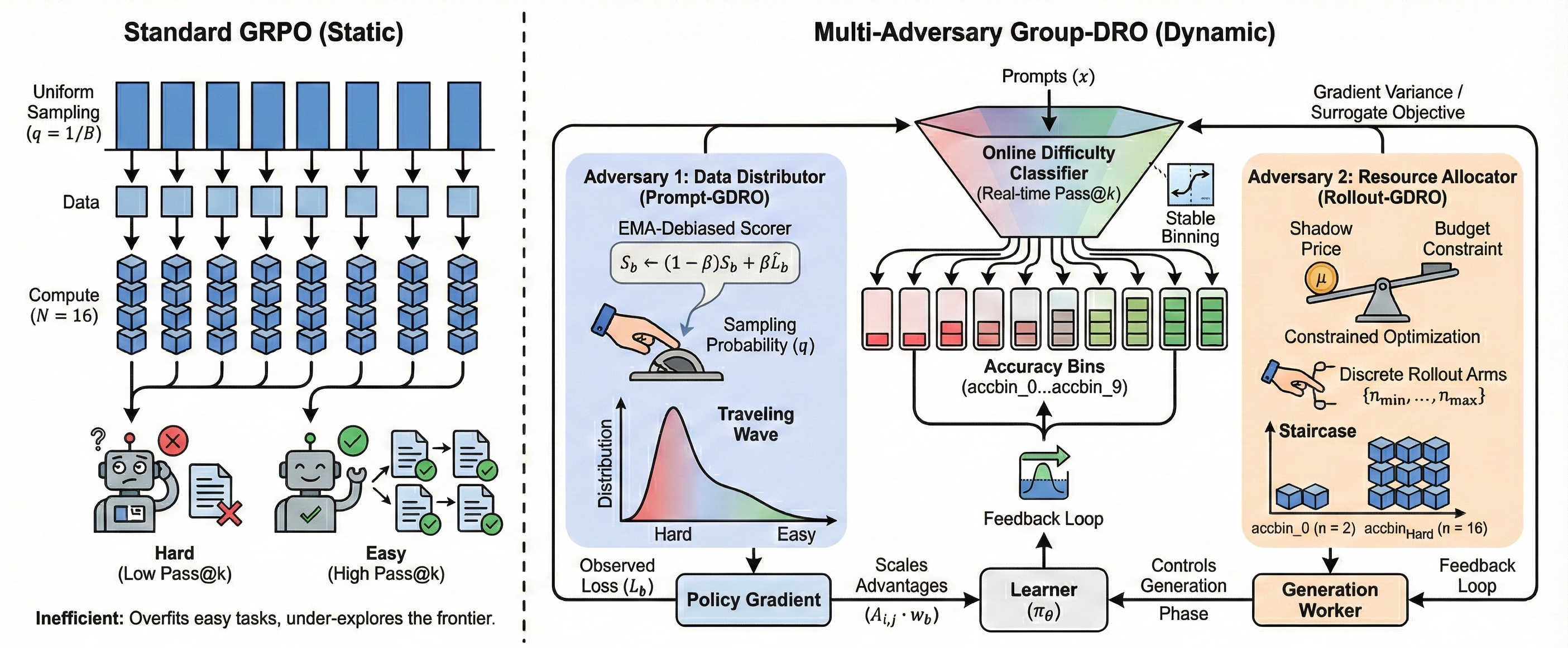}
\vspace{-2em}
\caption{\textbf{Conceptual Illustration: Static Uniformity vs. Multi-Adversary GDRO (Dynamic).} (Left) Standard GRPO samples prompts uniformly ($q=1/B$) and assigns a fixed number of rollouts (schematically $N=16$), causing it to overfit easy tasks while under-exploring the frontier. (Right) Our framework employs an \textbf{Online Difficulty Classifier} to dynamically partition prompts based on real-time pass@k. It introduces two \emph{independent} adversarial feedback loops (not coupled): (1) \textbf{Prompt-GDRO} (Data Distributor) uses an EMA-debiased scorer to shift the sampling distribution toward hard bins, creating a ``traveling wave'' of difficulty; (2) \textbf{Rollout-GDRO} (Resource Allocator) uses a shadow price $\mu$ to solve a constrained optimization problem, allocating discrete rollout arms ($n_{\min} \dots n_{\max}$) to maximize gradient variance reduction on high-uncertainty tasks.}
\vspace{-1.5em}
\label{fig:pull_figure}
\end{figure}

The capabilities of Large Language Models (LLMs) in complex reasoning are increasingly shaped not only by architectural scaling, but by the design of post-training objectives and alignment pipelines. Reinforcement Learning (RL), particularly methods like Proximal Policy Optimization (PPO) \citep{schulman2017proximal} and Group Relative Policy Optimization (GRPO) \citep{shao2024deepseekmath}, has emerged as a standard approach for aligning models with rigorous logical constraints. By optimizing for sparse verifiable rewards or process-based supervision, these methods have enabled significant breakthroughs in mathematical problem solving \citep{wang2023mathshepherd} and code generation.

At a high level, our perspective is that reasoning post-training exposes \emph{two} distinct sources of non-uniformity:
(i) \emph{which prompts} remain unsolved as the model improves (a shifting difficulty landscape), and
(ii) \emph{how much exploration} different prompts require to yield low-variance learning signals.
Standard pipelines treat both knobs as static---uniform prompt sampling and fixed rollouts---which we argue is mismatched to the heavy-tailed structure of reasoning.

Recent discussions in the broader ML community further motivate moving beyond uniformity from a \emph{data value} perspective. Finzi et al.~\citep{finzi2026epiplexity} propose \emph{epiplexity}---a notion of information that aims to quantify the structural content that a computationally bounded learner can extract from data (distinct from time-bounded entropy/noise). From this viewpoint, the ``value'' of a prompt is not determined by its frequency, but by whether it still contains learnable structure \emph{for the current policy} under a fixed compute budget. This perspective is aligned with our Prompt-GDRO mechanism: by steering sampling toward the evolving reasoning frontier, we concentrate updates on prompts that remain high-value for learning, rather than repeatedly training on already-solved, low-value examples.

Concurrently, empirical analyses of RL with verifiable rewards suggest that learning progress can be dominated by a small fraction of high-entropy ``decision'' points, challenging the implicit assumption that uniform averaging over all tokens/prompts is the most compute-efficient choice \citep{wang2025beyond8020}. Related work on ``thinking'' and the accuracy--compute Pareto frontier emphasizes that additional computation is most beneficial when applied selectively, and can be inefficient when applied uniformly across instances \citep{madaan2025rethinking_thinking}. Taken together, these perspectives echo a community intuition that \emph{both} data selection and compute allocation should be adaptive \citep{finzi2026epiplexity_thread}---precisely the two control knobs instantiated by our Prompt-GDRO and Rollout-GDRO adversaries.

However, prevalent RL pipelines rely on a fundamental assumption of \emph{static uniformity}: they sample prompts
uniformly from the training distribution and allocate a fixed computational budget (number of rollouts) to every prompt.
We argue that this rigidity creates structural inefficiencies. As formalized in our analysis (Section
\ref{sec:analysis}), reasoning datasets are inherently heterogeneous, composed of disjoint sub-domains (e.g., elementary
algebra vs. Olympiad number theory) with vastly different difficulty profiles. Under uniform sampling, optimization is
dominated by the most frequent, often easier patterns, so learning signal concentrates on the ``easy core'' while errors
persist in a long, difficult tail. A long line of supervised learning work has addressed this asymmetry by making
training explicitly \emph{difficulty-aware}---from curriculum and self-paced learning that schedule examples from easy to
hard \citep{bengio2009curriculum,kumar2010selfpaced}, to boosting and hard-example mining that upweight misclassified or
high-loss instances \citep{freund1997decision,shrivastava2016training}, and focal losses that downweight well-classified
examples \citep{lin2017focal}. This motivates viewing robustness through \emph{difficulty-defined} groups and optimizing
worst-bin performance in the spirit of GDRO \citep{sagawa2020distributionally}. Furthermore, the value of computational
exploration is non-uniform. In reasoning tasks, ``solved'' prompts yield low-variance gradients, while high-entropy
``frontier'' prompts require massive exploration to reduce gradient variance \citep{setlur2024learning}. A static budget
allocation fails to capture this dynamic, wasting resources on redundant verification while under-exploring critical
failure modes.

To address these limitations, we propose a \textbf{Multi-Adversary Group Distributionally Robust Optimization (GDRO)} framework. Motivated by the biological hypothesis that intelligent systems distinguish between a core representation learning module and a specialized ``steering subsystem'' that optimizes cost functions \citep{marblestone2016toward}, we implement a dynamic, data-agnostic grouping mechanism. Specifically, we replace static uniformity with an \emph{Online Difficulty Classifier} that partitions data based on real-time empirical error rates (pass@k), effectively allowing the optimization process to ``steer'' itself. We then formulate post-training as a zero-sum game and instantiate two complementary adversaries via the \textbf{GDRO-EXP3P} algorithm \citep{soma2022optimal}. Importantly, \emph{these adversaries are designed as independent modules}: Prompt-GDRO plays a GDRO reweighting game against the learner, while Rollout-GDRO plays a separate constrained compute-allocation game. In this work we analyze and evaluate the two games in isolation (no coupling); jointly coupling both adversaries into a single multi-time-scale system is left to future work.

Our contributions are summarized as follows:
\begin{enumerate}[nosep,leftmargin=*]
    \item \textbf{Prompt-GDRO (A Data Adversary):} We employ an adversarial reweighting rule that targets the \textit{intensive} difficulty margin (mean loss) instead of uniform sampling. We introduce an \textit{EMA-Debiased} scoring rule to prevent the adversary from succumbing to frequency bias, ensuring that rare, high-difficulty groups are upweighted effectively. This acts as a regularizer against over-optimizing the easy core, improving worst-bin robustness as the difficulty frontier shifts. \emph{Theory:} In Section~\ref{sec:analysis} (proofs in Appendix~B), we show that exponential-weights Prompt-GDRO corresponds to optimizing an entropy-regularized GDRO surrogate (a log-sum-exp ``soft worst-group'' objective) and admits a no-regret game interpretation.
    \item \textbf{Rollout-GDRO (A Compute Adversary):} We challenge the convention of fixed rollout budgets (e.g., $n=4$). We formulate rollout allocation as a constrained resource allocation game where a second adversary dynamically assigns rollout counts to maximize gradient variance reduction on hard tasks, subject to a global \emph{mean-rollout} compute constraint. This enables the model to efficiently explore the solution space of complex problems without increasing the total training budget. \emph{Theory:} In Section~\ref{sec:analysis} (proofs in Appendix~B), we derive a variance proxy for GRPO rollouts and show that the variance-optimal compute-neutral allocation obeys a square-root law, motivating the shadow-price controller used by Rollout-GDRO.
    \item \textbf{Empirical results.} On the DAPO 14.1k reasoning dataset with Qwen3-Base models (1.7B/4B/8B), Prompt-GDRO improves pass@8 by \textbf{+9.74\%}, \textbf{+13.13\%}, and \textbf{+8.96\%}, and Rollout-GDRO by \textbf{+10.64\%}, \textbf{+10.59\%}, and \textbf{+9.20\%} over GRPO. Qualitative analyses reveal an emergent curriculum that shifts sampling weight and rollout budget toward the evolving reasoning frontier.
\end{enumerate}

\section{Preliminaries}
\label{sec:preli}

\subsection{Reinforcement Learning for Reasoning}
\label{sec:preli_rl}

We formalize post-training for reasoning as Reinforcement Learning (RL) over an autoregressive language policy.
Let $x \sim \mathcal{D}$ denote a prompt and let $y=(y_1,\dots,y_L)$ denote a response sampled from a policy
$\pi_\theta(\cdot\mid x)$ parameterized by $\theta$. The conditional sequence probability factorizes as
\begin{equation}
    \pi_\theta(y\mid x) \;=\; \prod_{t=1}^{L} \pi_\theta\!\left(y_t \mid x, y_{<t}\right).
\end{equation}
The RL objective is to maximize expected reward
\begin{equation}
    J(\theta) \;=\; \mathbb{E}_{x \sim \mathcal{D},\, y \sim \pi_\theta(\cdot\mid x)}\!\left[r(x,y)\right],
\end{equation}
where $r(x,y)$ is a task-dependent, generally non-differentiable reward. In mathematical reasoning,
$r(x,y)$ is typically sparse (e.g., binary correctness) or semi-sparse (e.g., verifier-based signals).

\subsection{Group-Relative Policy Optimization (GRPO)}
\label{sec:preli_grpo}

Group-Relative Policy Optimization (GRPO) \citep{shao2024deepseekmath} is a computationally efficient alternative
to Proximal Policy Optimization (PPO) \citep{schulman2017proximal}. GRPO eliminates a learned value critic by
constructing a baseline from a within-prompt group of rollouts.

\begin{remark}[Two notions of ``group'']
Throughout the paper, the term ``group'' can refer to two different objects:
(i) a \emph{GRPO rollout group} (multiple rollouts for a fixed prompt), and
(ii) a \emph{GDRO group/bin} (a subset of prompts induced by a grouping rule, e.g., an online difficulty bin).
We explicitly index GRPO rollouts by $(i,j)$ and GDRO bins by $b \in \{1,\dots,B\}$.
\end{remark}

\begin{remark}[Notation: $n$ vs.\ $k$]
We use $n$ for the GRPO rollout-group size (train-time rollouts per prompt), and $k$ for ``best-of-$k$'' statistics such as pass@k and mean@k. In our experiments, we use $n=4$ and $k=8$.
\end{remark}

\noindent\textbf{Group sampling and advantage estimation.}
For each prompt $x_i$, we sample $n$ responses $\{y_{i,j}\}_{j=1}^n$ from a behavior policy $\pi_{\theta_{\text{old}}}$
(the policy used to generate rollouts). Each response receives a scalar reward
$r_{i,j} \triangleq r(x_i,y_{i,j})$. GRPO computes a group-relative advantage by standardizing rewards within
the rollout group:
\begin{equation}
    A_{i,j} \;=\; \frac{r_{i,j} - \mu_i}{\sigma_i + \varepsilon},
    \label{eq:grpo_advantage}
\end{equation}
where $\mu_i \triangleq \frac{1}{n}\sum_{j=1}^n r_{i,j}$, $\sigma_i \triangleq \sqrt{\frac{1}{n}\sum_{j=1}^n (r_{i,j}-\mu_i)^2}$,
and $\varepsilon>0$ is a small constant for numerical stability. The scalar advantage $A_{i,j}$ is applied
token-wise to all tokens in $y_{i,j}$ in the surrogate objective below.

\noindent\textbf{The clipped surrogate objective.}
Let
\begin{equation}
    \rho_{i,j,t}(\theta) \;\triangleq\; \frac{\pi_\theta\!\left(y_{i,j,t} \mid x_i, y_{i,j,<t}\right)}
    {\pi_{\theta_{\text{old}}}\!\left(y_{i,j,t} \mid x_i, y_{i,j,<t}\right)}
\end{equation}
denote the token-level importance ratio. The GRPO/PPO-style clipped surrogate for response $y_{i,j}$ is
\begin{equation}
    \mathcal{J}^{\text{CLIP}}_{i,j}(\theta)
    \;=\; \frac{1}{L_{i,j}} \sum_{t=1}^{L_{i,j}}
    \min\!\Big(
        \rho_{i,j,t}(\theta)\, A_{i,j},\;
        \text{clip}\!\big(\rho_{i,j,t}(\theta),\, 1-\epsilon,\, 1+\epsilon\big)\, A_{i,j}
    \Big),
    \label{eq:l_clip}
\end{equation}
where $\epsilon>0$ is the PPO clipping parameter.

\noindent\textbf{Observable loss signal.}
For our robust optimization controllers, we require a scalar ``loss-like'' signal per generated response.
We define the per-response loss as the negative KL-regularized surrogate:
\begin{equation}
    \ell_{i,j}(\theta)
    \;=\;
    -\mathcal{J}^{\text{CLIP}}_{i,j}(\theta)
    \;+\;
    \beta_{\text{KL}}\, D_{\text{KL}}\!\left(\pi_\theta(\cdot\mid x_i) \parallel \pi_{\text{ref}}(\cdot\mid x_i)\right),
    \label{eq:grpo_loss_signal}
\end{equation}
where $\beta_{\text{KL}}>0$ controls the KL penalty to a fixed reference policy $\pi_{\text{ref}}$.
In our experiments we operate in a \emph{zero-SFT} setting, so $\pi_{\text{ref}}$ is simply the \emph{initial base checkpoint}
(i.e., the same Qwen3-\{1.7B,4B,8B\}-Base model we start RL from) and is held frozen during training.
In practice, $D_{\text{KL}}(\pi_\theta(\cdot\mid x_i)\,\|\,\pi_{\text{ref}}(\cdot\mid x_i))$ is evaluated on the sampled responses
using token-level log-probabilities under $\pi_\theta$ and $\pi_{\text{ref}}$.
Finally, we define the \emph{prompt-level} loss as the mean over rollouts:
\begin{equation}
    \ell(x_i;\theta) \;\triangleq\; \frac{1}{n} \sum_{j=1}^{n} \ell_{i,j}(\theta).
\end{equation}
This prompt-level signal will be aggregated by bins and fed to the GDRO adversaries.

\subsection{Group Distributionally Robust Optimization}
\label{sec:preli_gdro}

Standard Empirical Risk Minimization (ERM) minimizes average loss under the empirical mixture, which can be dominated by
high-frequency and/or easy instances. A long line of supervised learning work addresses this imbalance by making training
explicitly \emph{difficulty-aware}---from curriculum and self-paced learning that schedule examples from easy to hard
\citep{bengio2009curriculum,kumar2010selfpaced}, to boosting and hard-example mining that upweight misclassified or
high-loss instances \citep{freund1997decision,shrivastava2016training}, and focal losses that downweight well-classified
examples \citep{lin2017focal}. GDRO provides a complementary, principled objective: treat subpopulations (here, difficulty
bins) as groups and optimize worst-group risk.

We adopt \textbf{Group Distributionally Robust Optimization} (GDRO) \citep{sagawa2020distributionally}. Assume
prompts are partitioned into $B$ disjoint groups (domains) $\{\mathcal{D}_1,\dots,\mathcal{D}_B\}$. Let
\begin{equation}
    L_b(\theta) \;\triangleq\; \mathbb{E}_{x\sim\mathcal{D}_b}\!\left[\ell(x;\theta)\right]
\end{equation}
denote the expected prompt-level loss for group $b$. GDRO optimizes worst-group performance by solving
\begin{equation}
    \min_{\theta}\;\max_{q\in\Delta_B}\;\sum_{b=1}^{B} q_b\, L_b(\theta),
    \label{eq:gdro_general}
\end{equation}
where $\Delta_B$ is the probability simplex. Following \citet{soma2022optimal}, \eqref{eq:gdro_general} admits a
zero-sum game interpretation between a learner ($\theta$) and an adversary ($q$) and can be optimized via no-regret
online learning. Our method instantiates this perspective with multiple adversarial ``levers'' on top of the GRPO loss
signal in \eqref{eq:grpo_loss_signal}.

\section{Method}\label{sec:method}

Standard instruction tuning paradigms are characterized by static uniformity: a uniform distribution over prompts and a fixed computational budget per prompt. We argue that this rigidity creates structural inefficiencies in learning, particularly for heterogeneous reasoning tasks where difficulty varies significantly across domains.

To address this, we propose a \textbf{Multi-Adversary Framework} that decomposes the training process into dynamic distributional levers. Our method operates in a shared environment where prompts are dynamically categorized by difficulty (Section \ref{sec:method_grouping}), serving as the basis for distinct adversarial processes:
\begin{enumerate}[nosep,leftmargin=*]
    \item \textbf{Adversarial Sampler ($q_{\text{prompt}}$):} Optimizes the probability of sampling prompts from different difficulty bins to expose model weaknesses (Section \ref{sec:method_reweighting}).
    \item \textbf{Adversarial Budgeter ($n_{\text{rollout}}$):} Optimizes the allocation of rollout counts $n_b$ to different bins to maximize gradient information under a global compute constraint (Section \ref{sec:method_budget}).
\end{enumerate}

\subsection{Dynamic Grouping via Online Pass@k}
\label{sec:method_grouping}

A critical prerequisite for GDRO is the definition of domains (groups) $\mathcal{D}_b$. Relying on static dataset metadata (e.g., ``Level 5'') is suboptimal because theoretical difficulty often diverges from empirical model capability. Furthermore, reliance on explicit labels limits applicability to datasets with rich metadata.

\subsubsection{Data-Agnostic Difficulty Estimation}
Our method removes the dependency on static annotations by defining groups solely through \textbf{empirical interaction}. We employ an \textbf{Online Difficulty Classifier} to construct dynamic groups based on the model's real-time performance. By utilizing the policy's own error rate as the grouping criterion, the framework becomes \textit{data-agnostic}. Whether the underlying data is math, code, or creative writing, the ``difficulty'' is emergent, allowing the pipeline to automatically discover and upweight the subsets of data that are currently challenging for the policy.

\subsubsection{Implementation Details}
We assign each prompt a unique identifier (UID) and track an online pass@$k$ statistic, where $k$ is a fixed hyperparameter that defines the difficulty scale (e.g., $k=8$ in our experiments).
At training step $t$, we sample $k$ rollouts $\{y_j\}_{j=1}^{k}$ for prompt $x$ and define an any-of-$k$ correctness indicator $c_t(x)\triangleq \mathbb{I}\{\exists j\in\{1,\dots,k\}: r(x,y_j)=1\}$, which equals $1$ iff at least one of the $k$ rollouts is correct.
We then maintain a moving estimate of pass@$k$ using a sliding window of length $H$,
\begin{equation}
    \widehat{\mathrm{pass@}k}_t(x) \;\triangleq\; \frac{1}{H} \sum_{s=t-H+1}^{t} c_s(x),
    \label{eq:passk_window}
\end{equation}
with the convention that the sum is taken over available history for newly seen UIDs.

We map prompts to discrete accuracy bins (e.g., \texttt{accbin\_0} for $[0,0.1)$, \texttt{accbin\_1} for $[0.1,0.2)$, etc.).
Let $0=a_0<a_1<\cdots<a_{B}=1$ denote bin edges. This induces a partition of the input space
$\mathcal{X} = \bigcup_{b=1}^B \mathcal{D}_b$, where
\begin{equation}
    \mathcal{D}_b \;\triangleq\; \big\{x\in\mathcal{X}: \widehat{\mathrm{pass@}k}_t(x) \in [a_{b-1},a_b)\big\}.
\end{equation}
We define the (time-varying) grouping map $g_t: \mathcal{X} \to \{1,\dots,B\}$ by $g_t(x)=b$ if $x\in\mathcal{D}_b$.

\begin{remark}[Bins as dynamic groups]
We use \emph{bin} and \emph{group} interchangeably: the bin index $b$ corresponds to the GDRO group
$\mathcal{D}_b$ induced by the online pass@$k$ estimate \eqref{eq:passk_window}. Within one training step, we treat
$g_t$ as fixed; across steps, the partition evolves as the policy improves.
\end{remark}

To ensure stability in the optimization landscape, we implement \textbf{hysteresis}: a prompt is only reassigned to a new bin if its moving average accuracy crosses the bin boundary by a margin $\delta$. This prevents prompts from oscillating between groups due to stochastic noise, ensuring that the adversaries target stable difficulty tiers rather than transient fluctuations.

\subsection{Adversarial Prompt Reweighting (Prompt-GDRO)}
\label{sec:method_reweighting}

The goal of the prompt reweighting adversary is to construct a distribution $q_t \in \Delta_B$ over the dynamic groups defined above. The adversary seeks to maximize the expected loss, thereby forcing the policy to improve on the ``Pareto frontier'' of difficulty. In summary, we realize Prompt-GDRO by bin-wise reweighting of GRPO updates (via per-sample weights applied to advantages), rather than by physically resampling prompts.

\subsubsection{EMA-Debiased EXP3P Optimization}
We solve the inner maximization problem using the \textbf{GDRO-EXP3P} algorithm \citep{soma2022optimal}.
Intuitively, the adversary maintains a distribution $q_t \in \Delta_B$ over bins and increases pressure on bins with
high recent loss. A key subtlety is \textbf{frequency bias}: if the adversary were to optimize \emph{cumulative}
(extensive) loss, then naturally frequent bins would dominate the score even if they are easy.
Our design instead targets the \emph{intensive} difficulty margin (mean loss), which is the quantity that indicates
how much a typical prompt in the bin is currently challenging.

To correct this, we propose an \textbf{EMA-Debiased} scoring rule. For each bin $b \in \{1, \dots, B\}$, we maintain a
difficulty score $S_t(b)$ updated via an Exponential Moving Average (EMA) with decay $\beta$:
\begin{equation}
    S_t(b) \leftarrow (1 - \beta) S_{t-1}(b) + \beta \cdot \bar{\ell}_{t}(b),
    \label{eq:ema_score}
\end{equation}
where $\bar{\ell}_{t}(b)$ is the empirical mean prompt-level loss for bin $b$ at step $t$. Let $\mathcal{B}_t$ be the
batch of prompts at step $t$, and let $\mathcal{B}_{b,t} = \{x \in \mathcal{B}_t \mid g_t(x) = b\}$ denote the
subset in bin $b$. We compute
\begin{equation}
    \bar{\ell}_{t}(b) = \frac{1}{|\mathcal{B}_{b,t}|} \sum_{x \in \mathcal{B}_{b,t}} \ell(x; \theta).
\end{equation}

Let $\hat q_t(b) \triangleq |\mathcal{B}_{b,t}|/|\mathcal{B}_t|$ denote the (realized) prompt share of bin $b$ in the
current batch. In practice we optionally normalize the update by $\hat q_t(b)$ (with a small floor) to ensure that
rare but consistently high-loss bins can compete with common bins in the EXP3P update.

Crucially, $S_t(b)$ tracks \textit{intensive} difficulty (mean loss) rather than extensive loss, which decouples the
adversarial signal from static dataset frequency.

The adversarial (unnormalized) bin weights $\omega_t(b)$ are derived by exponentiating these scores:
\begin{equation}
    \omega_t(b) = \exp\left(\eta_q \cdot \text{clip}(S_t(b), -C, C)\right),
    \label{eq:exp_update}
\end{equation}
where $\eta_q$ is the learning rate (sharpness). The final sampling probability includes a uniform mixing term $\gamma$ to guarantee exploration:
\begin{equation}
    q_t(b) = (1 - \gamma) \frac{\omega_t(b)}{\sum_{j=1}^B \omega_t(j)} + \frac{\gamma}{B}.
    \label{eq:final_prob}
\end{equation}
Rather than physically resampling the dataset, we realize $q_t$ by reweighting the GRPO gradient contribution of each
prompt. Concretely, for a prompt $x_i$ we scale its rollout advantages as
\begin{equation}
    A_{i,j} \leftarrow A_{i,j} \cdot \min\{\omega_t(g_t(x_i)),\, \omega_{\max}\},
\end{equation}
where $\omega_{\max}$ is a cap for numerical stability. Note that using the unnormalized score $\omega_t(b)$ (rather than the normalized probability $q_t(b)$) preserves the same relative weighting across bins; the omitted normalization constant is a common scalar factor that is absorbed into the effective step size of the GRPO update. Intuitively, this increases the effective step size on bins the
adversary considers difficult, which is equivalent (in expectation) to optimizing the GDRO objective with
adversarial weights.

\subsection{Adversarial Rollout Budgeting (Rollout-GDRO)}
\label{sec:method_budget}

In standard GRPO, the number of rollouts $n$ is fixed (e.g., $n=4$). However, ``easy'' prompts yield diminishing
returns from extra rollouts, while ``hard'' prompts require more exploration to reduce gradient variance and to
stabilize group statistics (\eg, the mean and standard deviation in \eqref{eq:grpo_advantage}).
We formulate the choice of rollout counts as a second GDRO-style resource allocation problem over the same
dynamic bins induced by $g_t$.

\subsubsection{Constrained Maximization Formulation}
We treat the number of rollouts for bin $b$ as a discrete variable $n_b \in \{n_{\min}, \dots, n_{\max}\}$.
Let $\hat q_t(b)$ denote the realized bin share in the prompt batch at step $t$.

\noindent\textbf{Bin-level utility from $n_b$ rollouts.}
Recall from Section~\ref{sec:preli} that GRPO provides a per-response loss signal $\ell_{i,j}(\theta)$
(Eq.~\eqref{eq:grpo_loss_signal}) and aggregates it into a prompt-level loss by averaging over $n$ rollouts,
$\ell(x_i;\theta)=\frac{1}{n}\sum_{j=1}^{n}\ell_{i,j}(\theta)$.
When using a bin-specific rollout count $n_b$, we compute the same quantity with $n_b$ samples:
\begin{equation}
    \ell(x_i;\theta, n_b) \;\triangleq\; \frac{1}{n_b} \sum_{j=1}^{n_b} \ell_{i,j}(\theta).
\end{equation}
Given the current batch $\mathcal{B}_t$ and the induced bin partition $\{\mathcal{B}_{b,t}\}_{b=1}^B$, we define the
empirical bin-level \emph{utility} (negative loss) under $n_b$ rollouts as
\begin{equation}\label{eq:analysis_budget_obj}
    \hat{J}_b(\theta; n_b)
    \;\triangleq\;
    -\frac{1}{|\mathcal{B}_{b,t}|} \sum_{x_i \in \mathcal{B}_{b,t}} \ell(x_i;\theta, n_b),
\end{equation}
where the hat emphasizes that this is a finite-sample estimate computed on the current batch.

\noindent\textbf{Mean-rollout constraint.}
The budgeter chooses $\{n_b\}$ to concentrate rollouts on bins where they are most valuable, subject to a strict global
budget on the \emph{mean} rollouts per prompt $\bar{n}$:
\begin{equation}\label{eq:analysis_budget_lagrangian}
    \max_{\{n_b\}} \sum_{b=1}^B \hat q_t(b) \, \hat{J}_b(\theta; n_b)
    \quad \text{s.t.} \quad \sum_{b=1}^B \hat q_t(b)\, n_b = \bar{n}.
\end{equation}
This formulation incentivizes allocating more compute to bins where additional rollouts improve gradient quality the
most: more rollouts both (i) reduce Monte Carlo noise and (ii) increase the number of informative samples
contributing to the update.
The global budget $\bar{n}$ (typically set to the baseline rollout count, e.g., $\bar{n}=4$) ensures the total
computational cost matches the uniform baseline.

\subsubsection{Dual Ascent Solver}
We solve the constrained maximization above via a Lagrangian relaxation with a single multiplier $\mu$ for the mean-rollout
constraint. For a candidate rollout count $n$ in bin $b$, we define the corresponding penalized \emph{bandit loss} as:
\begin{equation}
    L_b(n) \;=\; -\hat{J}_b(\theta; n) + \mu\, n.
    \label{eq:rollout_arm_loss}
\end{equation}
We maintain a separate \textbf{EXP3P} instance \citep{soma2022optimal} where the ``arms'' are the discrete integers in $[n_{\min}, n_{\max}]$. At each step, we select the configuration $\{n_b\}_{b=1}^B$ that maximizes the joint probability while strictly satisfying the global budget equality constraint $\sum_{b=1}^{B} \hat q_t(b) n_b = \bar{n}$. This exact matching is implemented via a dynamic programming selection step over the active bins.

After the batch is processed and the actual rollout consumption $\hat{n}_{\text{realized}}$ is observed, the dual variable $\mu$ is updated:
\begin{equation}\label{eq:budget_constraint}
    \mu \leftarrow \mu + \alpha_\mu (\hat{n}_{\text{realized}} - \bar{n}),
\end{equation}
where $\alpha_\mu$ is the dual learning rate. This mechanism ensures the method remains compute-neutral compared to the baseline, dynamically shifting resources from easy to hard groups based on the shadow price of compute $\mu$.

\section{Analysis}
\label{sec:analysis}

This section provides a first-principles interpretation of why \emph{static uniformity} can be structurally inefficient
in RL post-training for reasoning, and how our two controllers implement targeted robustness improvements.
We intentionally avoid re-defining Prompt-GDRO and Rollout-GDRO (Section~\ref{sec:method}); here we connect the method to
robust objectives and state the core theoretical messages.
Detailed derivations and proofs are deferred to Appendix~\ref{sec:main_theory}.

\subsection{Motivation: Why uniformity can fail in reasoning post-training}

Reasoning datasets are heterogeneous: prompts belong to latent sub-domains (topics, formats, reasoning styles) with
widely varying difficulty. Uniform sampling and a fixed rollout budget implicitly assume that (a) all prompts are equally
informative to train on and (b) the same amount of exploration is warranted everywhere.
Both assumptions are brittle in practice. In particular, once a large fraction of prompts become ``nearly solved,'' their
rollouts yield low-variance gradients and diminishing marginal learning signal, while the remaining frontier prompts
continue to exhibit high uncertainty.
Under uniform sampling, optimization is dominated by the most frequent, often easier patterns, so learning signal
concentrates on the ``easy core'' while errors persist in a long, difficult tail \citep{bengio2009curriculum,sagawa2020distributionally}.

\paragraph{GDRO lens.}
We use GDRO (Section~\ref{sec:preli_gdro}, Eq.~\eqref{eq:gdro_general}) as a compact way to reason about
worst-bin robustness. In our setting the ``groups'' are the online difficulty bins induced by $g_t$
(Section~\ref{sec:method_grouping}); within a training step we treat $g_t$ as fixed and interpret the group losses
$L_b(\theta)$ in Eq.~\eqref{eq:gdro_general} as bin-conditional GRPO prompt losses.
Under this lens, Prompt-GDRO is an online approximation to the inner adversary over $q$, while Rollout-GDRO is a second
adversary that reallocates rollout compute across the same bins to improve the signal-to-noise ratio of the update under a
compute-neutral budget.

\subsection{Adversary I (Prompt-GDRO): Robustness via EMA-debiased difficulty pressure}
\label{sec:analysis_adv1}

\paragraph{From GDRO to a bandit adversary.}
If bins were fixed and losses were observed noiselessly, the inner maximization in
Eq.~\eqref{eq:gdro_general} could be solved by concentrating $q$ on the current worst-loss bin.
In post-training, however, (i) losses are stochastic Monte Carlo estimates, and (ii) the grouping $g_t$ is online and
non-stationary as the model improves.
Following \citet{soma2022optimal}, we view the adversary as an online learner that updates a distribution over bins
using no-regret bandit-style updates.

\paragraph{Why ``EMA-debiased''? Frequency bias vs.\ intensive difficulty.}
Methodologically, Prompt-GDRO is fully specified in Section~\ref{sec:method_reweighting}.
The key point for the analysis is that the adversary is driven by an \emph{intensive} statistic (mean prompt loss per bin)
rather than an extensive cumulative loss, and that this statistic is smoothed over time.
Concretely, Prompt-GDRO maintains an EMA difficulty score $S_t(b)$ (Eq.~\eqref{eq:ema_score}), exponentiates the (clipped)
scores to obtain unnormalized weights $\omega_t(b)$ (Eq.~\eqref{eq:exp_update}), and forms the adversarial bin distribution
$q_t$ with an exploration mixture (Eq.~\eqref{eq:final_prob}).
Optionally normalizing by the realized bin share $\hat q_t(b)$ mitigates frequency bias so that persistently hard but rare
bins can remain competitive.

\paragraph{How $q_t$ acts on GRPO updates (intuition).}
Although $q_t$ can be interpreted as a sampling distribution over bins, we realize it in a compute-neutral way by
reweighting gradient contributions (Section~\ref{sec:method_reweighting}).
Scaling a prompt's rollout advantages by a bin-dependent multiplier increases its effective learning rate, which
implements the same ``pay more attention to hard bins'' principle as Eq.~\eqref{eq:gdro_general}.

\subsubsection{Entropic GDRO view: log-sum-exp surrogate and softmax best response}
\label{sec:analysis_entropic_gdro}

Prompt-GDRO implements the adversary via exponential weights (Eqs.~\eqref{eq:exp_update}--\eqref{eq:final_prob}),
which corresponds to \emph{entropic mirror ascent} on the simplex.
A key consequence is that the hard inner maximum in Eq.~\eqref{eq:gdro_general} is implicitly replaced by an
entropy-regularized one, yielding a smooth ``soft worst-group'' objective.

To keep notation uncluttered, for the remainder of this subsection we treat $g_t$ as fixed within a step and omit the
subscript $t$.

\begin{lemma}[Entropic GDRO surrogate and softmax best response]
\label{lem:analysis_entropic_surrogate}
For any $\eta>0$, define the entropy-regularized inner problem
\begin{equation}
\label{eq:analysis_entropic_risk}
\mathcal{R}_{\eta}(\theta)
\;\triangleq\;
\max_{q\in\Delta_B}\left\{\sum_{b=1}^{B} q(b)L_b(\theta)+\frac{1}{\eta}H(q)\right\},
\qquad
H(q)\triangleq -\sum_{b=1}^{B} q(b)\log q(b).
\end{equation}
Then
\begin{equation}
\label{eq:analysis_entropic_risk_lse}
\mathcal{R}_{\eta}(\theta)
=
\frac{1}{\eta}\log\!\left(\sum_{b=1}^{B} e^{\eta L_b(\theta)}\right),
\end{equation}
and the (unique) maximizer is the softmax distribution
\begin{equation}
\label{eq:analysis_q_eta_def}
q_{\eta}(b;\theta)
\;=\;
\frac{\exp(\eta L_b(\theta))}{\sum_{j=1}^{B} \exp(\eta L_j(\theta))}.
\end{equation}
Moreover, $\mathcal{R}_\eta$ approximates the hard worst-group loss up to $\log B/\eta$:
\begin{equation}
\label{eq:analysis_entropic_surrogate_bounds}
\max_{b\in[B]} L_b(\theta)
\;\le\;
\mathcal{R}_\eta(\theta)
\;\le\;
\max_{b\in[B]} L_b(\theta) + \frac{\log B}{\eta}.
\end{equation}
\end{lemma}

\noindent\emph{Proof.} Appendix~\ref{sec:theory_lse}.

\paragraph{Gradient interpretation.}
Differentiating Eq.~\eqref{eq:analysis_entropic_risk_lse} yields a weighted mixture of group gradients,
\begin{equation}
\label{eq:analysis_entropic_grad}
\nabla_\theta \mathcal{R}_\eta(\theta)
=
\sum_{b=1}^{B} q_{\eta}(b;\theta)\,\nabla_\theta L_b(\theta),
\end{equation}
so applying larger weights to higher-loss bins can be read as taking a gradient step on a smooth approximation to the
worst-group objective.

\paragraph{Connection to Prompt-GDRO.}
Eq.~\eqref{eq:analysis_q_eta_def} is the entropy-regularized best response of the adversary to $\theta$.
In an idealized full-information setting where the adversary observes $L_b(\theta_t)$ directly, exponential weights
tracks this softmax best response online.
In our implementation, the EMA score $S_t(b)$ (Eq.~\eqref{eq:ema_score}) is a smoothed (and optionally debiased) estimator
of the current bin losses, so the bin distribution $q_t$ (Eq.~\eqref{eq:final_prob}) can be interpreted as tracking
$q_{\eta_q}(\cdot;\theta_t)$ up to bandit noise and exploration.

\subsubsection{No-regret interpretation: why this game structure is sensible}
\label{sec:analysis_noregret_prompt}

The entropic GDRO view explains \emph{what} objective the exponential-weights adversary is tracking.
A complementary (standard) lens explains \emph{why} coupling this adversary with gradient-based policy updates is sensible:
in a convex--concave game, if both players have sublinear regret, time-averaged iterates converge to an approximate saddle
point.
Concretely, letting $f(\theta,q)\triangleq \sum_{b=1}^{B} q(b)L_b(\theta)$ and writing $\bar\theta=\frac{1}{T}\sum_{t=1}^{T}\theta_t$
and $\bar q=\frac{1}{T}\sum_{t=1}^{T}q_t$, one obtains the generic bound
\begin{equation}
\label{eq:analysis_saddle_gap_regret}
\max_{q\in\Delta_B} f(\bar\theta,q)\;-\;\min_{\theta\in\Theta} f(\theta,\bar q)
\;\le\;
\frac{\mathrm{Regret}_\theta(T)+\mathrm{Regret}_q(T)}{T},
\end{equation}
where $\mathrm{Regret}_\theta(T)$ and $\mathrm{Regret}_q(T)$ are the cumulative regrets of the learner and adversary,
respectively.
Appendix~\ref{sec:theory_noregret} states a precise version (with step sizes) for the convex bounded regime.
In deep RL, we read Eq.~\eqref{eq:analysis_saddle_gap_regret} as an \emph{optimization principle}: if (i) the GRPO updates
behave like a low-regret learner and (ii) the Prompt-GDRO adversary behaves like a low-regret reweighting rule, then the
training dynamics push down the robust (worst-group) objective over time.

\subsubsection{Toy validation on MATH: entropy and worst-group robustness}
\label{sec:analysis_toy_math}

The analysis above predicts that a well-behaved Prompt-GDRO adversary should (i) avoid collapsing onto a tiny set of
bins due to noisy estimates and (ii) improve worst-group performance.
We validate these qualitative predictions on the MATH benchmark as a \emph{toy} study (Prompt-GDRO only).

The standard GRPO baseline achieves a Worst-Group Pass@1 of roughly $33.92\%$.
While a class-based GDRO baseline improves this to $37.7\%$, it exhibits instability in the adversary entropy,
collapsing to a narrow effective support of $\sim 12$ groups.
In contrast, our EMA-debiased formulation achieves the highest robustness with Worst-Group Pass@1 of
\textbf{$39.6\%$}, and maintains a broader active support of $\sim 24$ groups throughout training.
This broader support is consistent with the ``intensive difficulty'' objective in Eq.~\eqref{eq:ema_score}:
instead of chasing a few high-variance bins, the adversary sustains a diversified portfolio of challenging bins, forcing
the policy to improve along a wider Pareto frontier of hard prompts.

\subsection{Adversary II (Rollout-GDRO): Compute allocation as an economic control problem}
\label{sec:analysis_adv2}

Prompt reweighting changes \emph{which bins receive learning pressure}; rollout budgeting changes \emph{how much
exploration} is used to estimate and optimize that pressure.
The key empirical fact is that rollouts are not equally useful everywhere: once a bin is nearly solved, additional
rollouts mostly repeat correct solutions and contribute low-variance gradients; on frontier bins, additional rollouts
reduce Monte Carlo noise and improve both reward estimation and GRPO's within-prompt normalization.

\paragraph{Recap: compute-neutral rollout budgeting.}
Rollout-GDRO is defined in Section~\ref{sec:method_budget}.
At each step, the controller selects a bin-specific rollout count $n_b\in\{n_{\min},\dots,n_{\max}\}$ under a strict
mean-rollout budget $\bar n$.
Operationally, it uses a Lagrangian relaxation with multiplier $\mu$ (a shadow price of compute): the penalized rollout
``arm loss'' $L_b(n)$ is defined in Eq.~\eqref{eq:rollout_arm_loss}, and $\mu$ is updated by dual ascent
(Eq.~\eqref{eq:budget_constraint}).
This closed-loop control is what produces the ``budget frontier'' and ``multiplier effect'' patterns in
Section~\ref{sec:rollout_analysis}.

\subsubsection{A variance proxy for why allocating more rollouts helps}
\label{sec:analysis_variance_proxy}

The rollout allocator is most useful when extra rollouts reduce the Monte Carlo noise of the GRPO update.
A simple proxy is to focus on the stochastic variance of the per-prompt gradient estimate.
Even though GRPO normalizes advantages within a rollout group (coupling the rollouts for a prompt), a mild bounded-differences
condition on the resulting prompt-level gradient estimator implies the conditional variance still contracts at rate $O(1/n_b)$ with
the number of rollouts.
Abstractly, we can write a bin-dependent ``intrinsic variance'' parameter $v_b(\theta)$ so that the per-prompt conditional variance
obeys $\mathrm{Var}[\hat g(x;\theta,n_b)\mid g_t(x)=b]\le v_b(\theta)/n_b$, and approximate the batch-level noise by
\begin{equation}
\label{eq:analysis_var_proxy}
\mathrm{VarProxy}(\mathbf n;\theta)
\;\triangleq\;
\sum_{b=1}^{B} \bar q_b\,\frac{v_b(\theta)}{n_b},
\qquad
\bar q_b \triangleq \hat q_t(b),
\end{equation}
where $\bar q_b$ are the realized bin fractions in the current batch.
Appendix~\ref{sec:theory_variance_proxy} makes this precise (via a bounded-differences / Efron--Stein argument)
under i.i.d.\ rollout sampling, and notes that the condition accommodates GRPO's within-prompt normalization.

\subsubsection{The variance-optimal allocation obeys a square-root law}
\label{sec:analysis_sqrt_law}

Motivated by Eq.~\eqref{eq:analysis_var_proxy}, consider the continuous relaxation of a variance-aware compute-neutral
allocation:
\begin{equation}
\label{eq:analysis_var_opt_problem}
\min_{\mathbf n\in\mathbb{R}_+^{B}}\;\sum_{b=1}^{B} \bar q_b\,\frac{v_b(\theta)}{n_b}
\qquad\text{s.t.}\qquad
\sum_{b=1}^{B} \bar q_b\,n_b = \bar n.
\end{equation}

\begin{theorem}[Square-root law for variance-optimal allocation]
\label{thm:analysis_sqrt_allocation}
Let $A\triangleq\{b\in[B]: \bar q_b>0\}$ denote the set of bins present in the batch.
If $v_b(\theta)>0$ for all $b\in A$, then the minimizer of Eq.~\eqref{eq:analysis_var_opt_problem} is unique on the active
coordinates and equals
\begin{equation}
\label{eq:analysis_nb_star}
n_b^\star
= \bar n\cdot\frac{\sqrt{v_b(\theta)}}{\sum_{j=1}^{B} \bar q_j\,\sqrt{v_j(\theta)}},
\qquad b\in A.
\end{equation}
\end{theorem}

\noindent\emph{Proof.} Appendix~\ref{sec:theory_sqrt_law}.

\paragraph{Implications for Rollout-GDRO.}
Eq.~\eqref{eq:analysis_nb_star} is the classical Neyman/square-root allocation: bins with higher intrinsic variance receive
more rollouts, with $n_b^\star\propto \sqrt{v_b(\theta)}$.
The associated KKT condition can be written as a shadow-price rule: for a multiplier $\mu>0$, the per-bin best response of
the continuous relaxation solves
\begin{equation}
\label{eq:analysis_shadow_price_best_response}
n_b(\mu)\;=\;\arg\min_{n>0}\;\frac{v_b(\theta)}{n}+\mu n\;=\;\sqrt{\frac{v_b(\theta)}{\mu}},
\end{equation}
and $\mu$ is chosen to satisfy the mean-budget constraint.
This clarifies why Rollout-GDRO naturally admits an ``economic'' interpretation: the dual variable $\mu$ trades off variance
reduction against compute.

In practice we choose $n_b$ from a discrete set of rollout arms and only observe bandit feedback (the value of the chosen
arm).
Rollout-GDRO's EXP3P updates and DP-based selection step (Section~\ref{sec:method_budget}) can be read as an online
approximation to the shadow-price solution above; the resulting best-response map is piecewise constant in $\mu$, which
matches the staircase/threshold transitions in our rollout heatmaps.
We defer the corresponding online-learning formalization (and extensions beyond the variance proxy) to the appendix.

\section{Experiments}
\label{sec:experiments}

In this section, we present the empirical evaluation of our \textbf{Multi-Adversary GDRO} framework. We use the DAPO 14.1k English dataset\footnote{\url{https://huggingface.co/datasets/open-r1/DAPO-Math-17k-Processed}} for all training runs, following a standard post-training pipeline. 

\noindent\textbf{Metrics.}
Unless stated otherwise, we report \emph{mean@8} for benchmark accuracies: each prompt is evaluated with 8 sampled completions
and we average the binary correctness indicator across those 8 trials.
We additionally report \emph{pass@8}, the probability that at least one of the 8 completions is correct (estimated from the same samples).
This distinction matters for reasoning: mean@8 reflects typical performance, while pass@8 captures ``best-of-$k$'' robustness under limited search.

\noindent\textbf{Compute neutrality.}
Prompt-GDRO keeps the rollout budget fixed and changes training pressure through bin-wise reweighting of GRPO updates.
Rollout-GDRO keeps the \emph{mean} rollout budget fixed (e.g., $\bar n=4$) and redistributes rollouts across bins.
Thus, improvements reflect better use of the same overall training compute rather than a larger sampling budget.

We first report the quantitative improvements on standard mathematical reasoning benchmarks. Subsequently, we provide a rigorous qualitative analysis of the dynamic difficulty landscape, interpreting how the interaction between model capacity and data heterogeneity drives the emergent curriculum observed in our prompt reweighting distribution and rollout allocation strategies.

\subsection{Main Results}
We evaluate our method across three model scales: \textbf{Qwen3-1.7B-Base}, \textbf{Qwen3-4B-Base}, and \textbf{Qwen3-8B-Base}. We compare the standard GRPO baseline against our two proposed mechanisms: \textbf{Prompt-GDRO} (adversarial prompt reweighting driven by online difficulty bins) and \textbf{Rollout-GDRO} (adversarial compute budgeting). Table \ref{tab:gdro_results} summarizes the performance at the peak checkpoint for each stabilized run.

\begin{table*}[t]
\centering
\caption{Results on mathematical reasoning benchmarks for Prompt Reweighting GDRO and Rollout Budgeting GDRO vs GRPO Baseline. \textbf{Bold} values indicate methods that outperform the GRPO Baseline (independent comparison). The \textbf{AIME} column reports the average accuracy of AIME 2024 and AIME 2025. The percentage improvement for pass@8 is shown in brackets. All other metrics reported are \textbf{mean@8}.}
\vspace{-0.1in}
\label{tab:gdro_results}
\scalebox{0.88}{\setlength{\tabcolsep}{1.5mm}{
\begin{tabular}{lcccccccc>{\columncolor{gray!15}}c@{}}
\toprule
\textbf{Models} & \textbf{MATH 500} & \textbf{AIME} & \textbf{AMC} & \textbf{MINERVA} & \textbf{OLYMPIAD} & \textbf{GPQA} & \textbf{pass@8} \\
\midrule
\multicolumn{8}{@{}l}{\textit{Qwen3-1.7B-Base}} \\
\quad GRPO (Baseline) & 50.62 & 5.42 & 34.69 & 14.56 & 23.07 & 26.39 & 50.74 \\
\quad Prompt-GDRO & \textbf{63.20} & \textbf{6.88} & \textbf{39.38} & 14.61 & \textbf{25.07} & \textbf{29.29} & \textbf{55.68} \textcolor{blue}{(+9.74\%)} \\
\quad Rollout-GDRO & \textbf{63.98} & \textbf{7.50} & \textbf{36.87} & \textbf{17.28} & \textbf{26.71} & \textbf{27.72} & \textbf{56.14} \textcolor{blue}{(+10.64\%)} \\
\midrule
\multicolumn{8}{@{}l}{\textit{Qwen3-4B-Base}} \\
\quad GRPO (Baseline) & 72.05 & 11.25 & 60.94 & 17.79 & 30.48 & 35.54 & 56.31 \\
\quad Prompt-GDRO & \textbf{75.78} & \textbf{12.92} & \textbf{64.06} & \textbf{26.72} & \textbf{40.98} & \textbf{40.88} & \textbf{63.70} \textcolor{blue}{(+13.13\%)} \\
\quad Rollout-GDRO & \textbf{75.20} & \textbf{13.96} & \textbf{67.50} & \textbf{26.47} & \textbf{39.28} & \textbf{38.51} & \textbf{62.27} \textcolor{blue}{(+10.59\%)} \\
\midrule
\multicolumn{8}{@{}l}{\textit{Qwen3-8B-Base}} \\
\quad GRPO (Baseline) & 73.45 & 14.38 & 66.56 & 28.17 & 36.92 & 42.25 & 62.04 \\
\quad Prompt-GDRO & \textbf{76.18} & \textbf{16.04} & \textbf{70.94} & \textbf{32.17} & \textbf{42.43} & \textbf{43.81} & \textbf{67.60} \textcolor{blue}{(+8.96\%)} \\
\quad Rollout-GDRO & \textbf{77.88} & \textbf{15.63} & \textbf{66.88} & \textbf{29.55} & \textbf{43.62} & \textbf{43.31} & \textbf{67.75} \textcolor{blue}{(+9.20\%)} \\
\bottomrule
\end{tabular}}}
\end{table*}

Our framework demonstrates robust gains across all settings. \textbf{Prompt-GDRO} consistently improves performance by explicitly targeting hard data groups, achieving a peak gain of \textbf{+13.13\%} on the 4B model. Remarkably, \textbf{Rollout-GDRO} achieves comparable or superior results---improving the 1.7B model by \textbf{+10.64\%} and the 8B model by \textbf{+9.20\%}---without altering the data distribution. This validates our hypothesis that \textit{compute allocation} is an equally powerful lever for robustness: by dynamically assigning more rollouts to high-variance prompts, the adversary reduces gradient variance exactly where the model is most uncertain, yielding gains that rival direct data curriculum learning.

\begin{keyfindingbox}{Key Finding 1}
Both prompt reweighting (Prompt-GDRO) and adaptive compute allocation (Rollout-GDRO) independently yield significant performance gains, improving pass@8 accuracy by up to 13.13\% and 10.64\% respectively across model scales.
\end{keyfindingbox}

\subsection{Qualitative Analysis: The Dynamics of Difficulty (Prompt-GDRO)}
\label{sec:qualitative_analysis}

To understand the mechanism behind these performance gains, we visualize the temporal evolution of the training distribution. Figure \ref{fig:mechanism} presents a comprehensive triptych tracking the causal chain of our adversarial mechanism: from \textit{data availability} (prompt share) to \textit{adversarial pressure} (weights) to \textit{learning payoff} (reward).

\begin{figure*}[t]
    \centering
    \includegraphics[width=\textwidth]{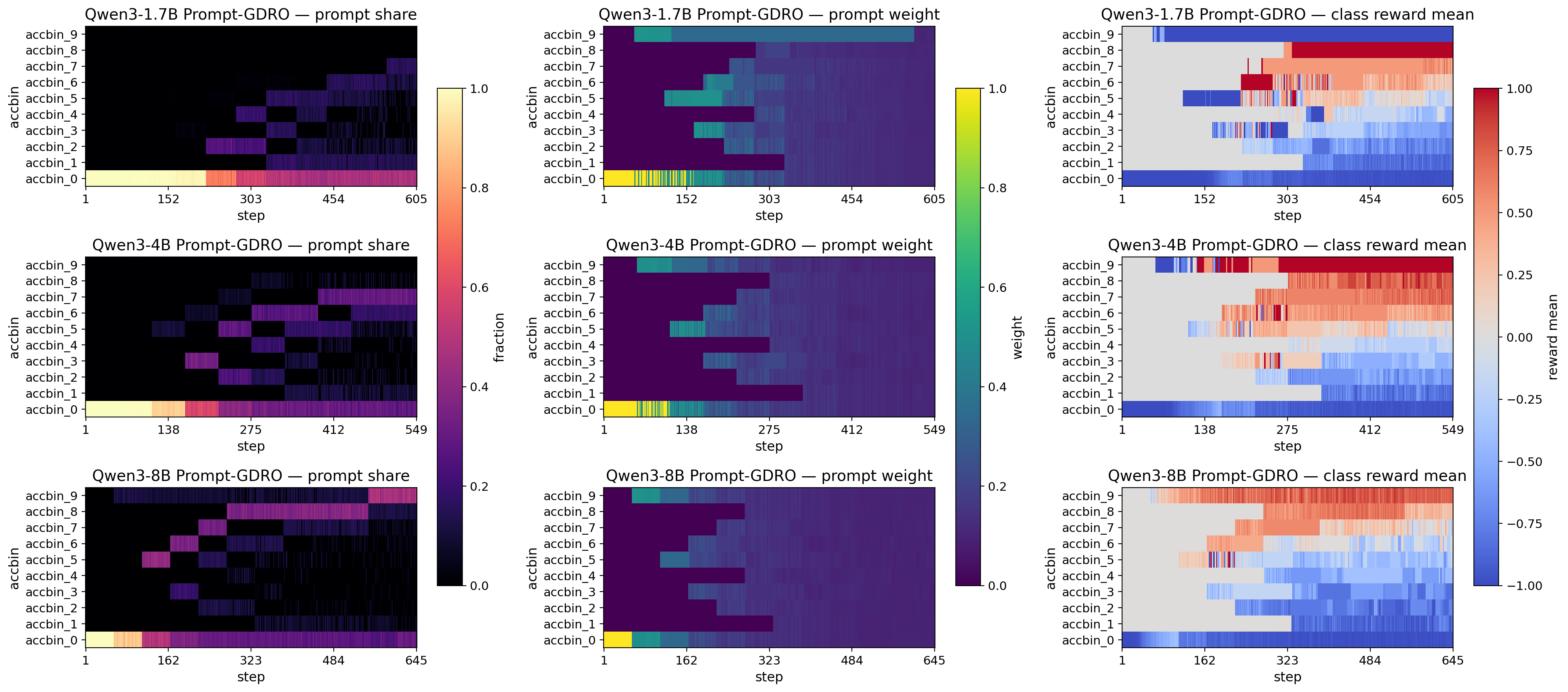}
    \caption{\textbf{The Causal Chain of Curriculum.} A triptych comparing the \textbf{Prompt Distribution} (Left), \textbf{Adversarial Weights} (Middle), and \textbf{Realized Reward} (Right) for 1.7B, 4B, and 8B models. This visualizes the mechanism: even when hard bins are rare in the data (dark regions in Left), the adversary applies disproportionate pressure (bright bands in Middle), forcing the model to eventually crack these problems and yield positive rewards (emergence of red/blue bands in Right). This effectively proves that Prompt-GDRO decouples the learning signal from dataset frequency.}
    \label{fig:mechanism}
\end{figure*}

\paragraph{Capacity-Dependent Distribution Shift.}
The training dynamics reveal a stark correlation between model capacity and the ``velocity'' of learning, as quantified by the macro-level metrics in Figure \ref{fig:dynamics_metrics}.
\begin{itemize}[nosep,leftmargin=*]
    \item \textbf{Qwen3-1.7B (Top Row, Fig \ref{fig:mechanism}):} The model exhibits high inertia. While the dataset mass often lags in \texttt{accbin\_0}, the scalar summary (Figure \ref{fig:dynamics_metrics}, top panel) shows the mean bin index steadily climbing past 2.0. The \textbf{mass in bins $\ge 3$} trace reveals that even this smaller model is successfully pushed out of the trivial zone, preventing the stagnation typical of uniform baselines.

    \item \textbf{Qwen3-4B (Middle Row, Fig \ref{fig:mechanism}):} This scale occupies a ``Goldilocks zone.'' The data distribution shows a steady migration to intermediate bins. The adversarial weights form a distinct, high-intensity \textbf{diagonal frontier} that leads the data distribution. By Step 366, the weight entropy stabilizes, indicating the adversary has locked into a high-precision curriculum targeting specific intermediate failure modes.

    \item \textbf{Qwen3-8B (Bottom Row, Fig \ref{fig:mechanism}):} The largest model exhibits a rapid collapse of the unsolved mass. The scalar summaries show the \textbf{mass in bins $\ge 8$} (dashed lines) spiking early, confirming that for capable models, the adversary aggressively focuses on the ``last mile'' of robustness---solving the few remaining hard cases---rather than wasting capacity on solved arithmetic.
\end{itemize}

\begin{figure}[t]
    \centering
    \includegraphics[width=0.95\linewidth]{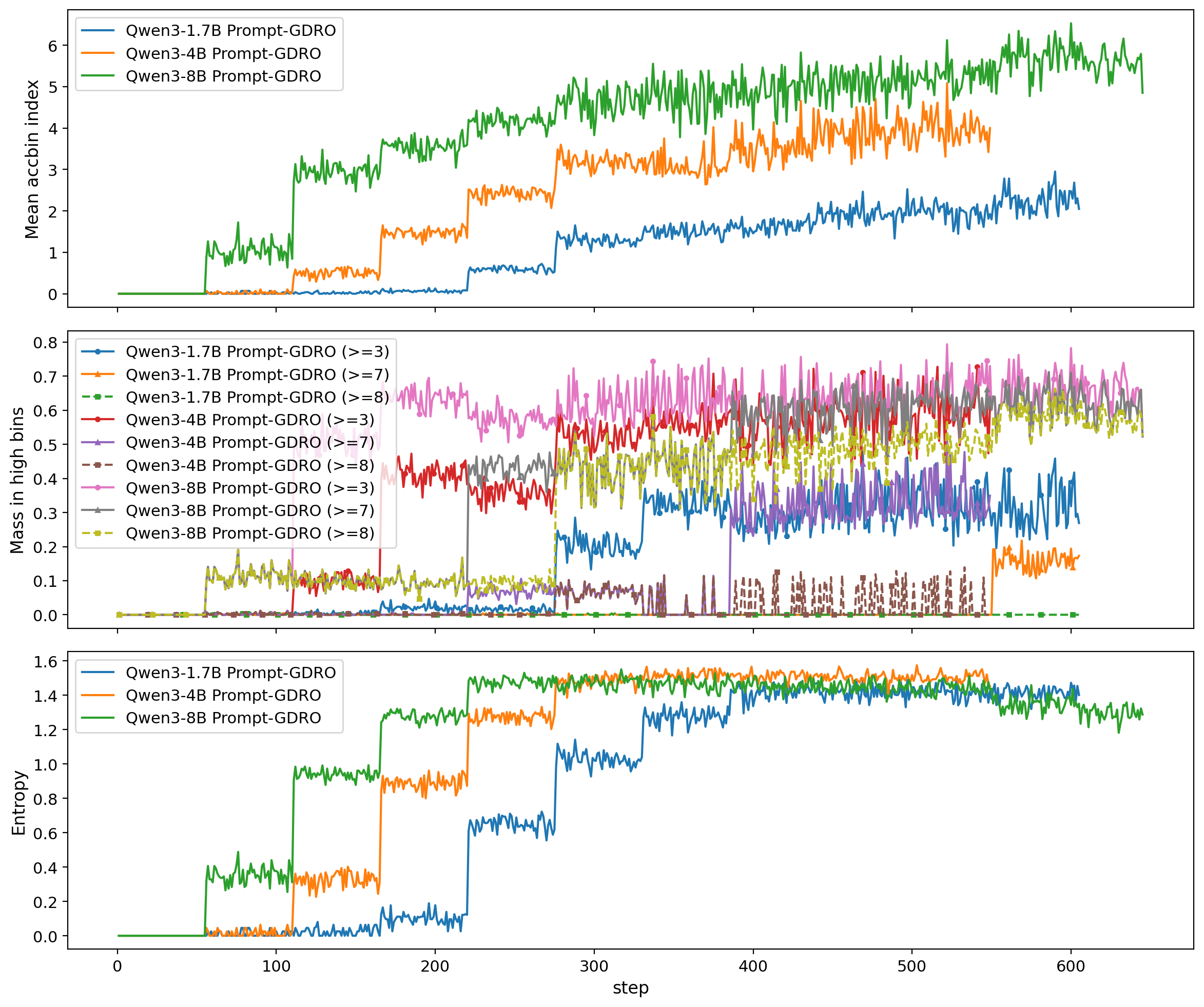}
    \caption{\textbf{Training Dynamics Quantified (Prompt-GDRO).} (Top) The Mean Accuracy Bin Index tracks the rising difficulty of targeted prompts. (Middle) The fraction of prompts in difficulty bands $\ge 3$ (solid) and $\ge 8$ (dashed) highlights scaling laws: 1.7B struggles to clear the $\ge 3$ bar, while 8B rapidly saturates even the $\ge 8$ band. (Bottom) Entropy metrics confirm that the adversary maintains a diverse portfolio of difficulty, preventing mode collapse to a single bin.}
    \label{fig:dynamics_metrics}
\end{figure}

\begin{figure*}[t]
    \centering
    \includegraphics[width=\textwidth]{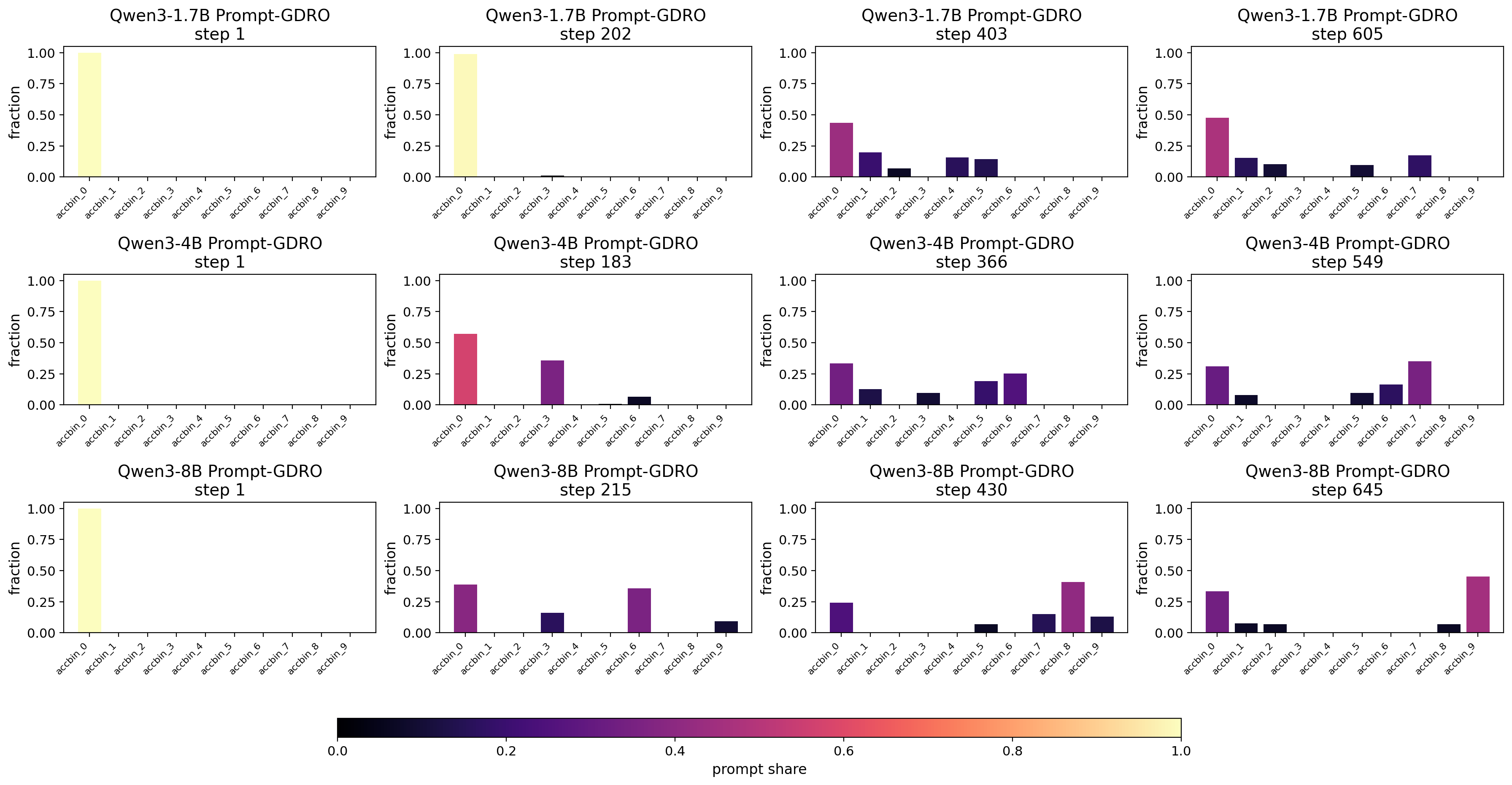}
    \caption{\textbf{Snapshots of the Learning Distribution.} The probability mass of the training set across difficulty bins at four distinct training stages (Start, Early-Mid, Late-Mid, End). These publication-friendly checkpoints reveal the exact shape of the curriculum: note how the 4B model (Middle Row) transitions from a uniform start to a heavy emphasis on \texttt{accbin\_5}--\texttt{accbin\_7} by mid-training, whereas the 8B model (Bottom Row) shifts almost its entire mass to the hardest bins by the final step.}
    \label{fig:snapshots}
\end{figure*}

\paragraph{Visualizing the Wave of Progress.}
Figure \ref{fig:snapshots} offers discrete checkpoints that clarify the exact distributional shape at key training intervals.
\begin{itemize}[nosep,leftmargin=*]
    \item \textbf{The Heavy Tail (1.7B):} Even at the final step (Step 605), the 1.7B model retains a significant plurality of mass ($\approx 45\%$) in \texttt{accbin\_0}. The distribution remains right-skewed, indicating the ``reasoning frontier'' is anchored in fundamental difficulties.
    \item \textbf{The High-Entropy Plateau (4B):} By Step 366, the 4B model allocates over 50\% of its probability mass to the intermediate \texttt{accbin\_5}--\texttt{accbin\_7} range. This ``plateau'' represents a diverse curriculum where the model simultaneously refines intermediate concepts and attempts advanced problems.
    \item \textbf{The Traveling Peak (8B):} The 8B model demonstrates a clear ``wave'' dynamic. The peak of the distribution physically travels from left to right. By Step 430, the mass in \texttt{accbin\_0} has virtually vanished, and the bulk of the sampling budget is dedicated to \texttt{accbin\_8} and \texttt{accbin\_9}. This confirms that for sufficient capacity, the primary challenge shifts rapidly from \textit{correctness} to \textit{robustness}, necessitating the dynamic budget reallocation our method provides.
\end{itemize}

\paragraph{Adversary lead--lag.}
The heatmaps in Figure~\ref{fig:mechanism} suggest that the adversarial weights form a ``frontier'' that \emph{precedes} the empirical data distribution.
We quantify this intuition with a lightweight lead--lag proxy computed from logged bin statistics.
Let $q_t\in\Delta_B$ denote the empirical \emph{prompt share} over bins at training step $t$, and let $\hat w_t\in\Delta_B$ denote the \emph{normalized} Prompt-GDRO weights across bins at the same step (the weight-only distribution, not reweighted by $q_t$).
Define the mean bin index under the data and under the weights as
\begin{equation}
    \mu_{\mathrm{data}}(t) \triangleq \sum_{b=1}^{B} b\,q_t(b),
    \qquad
    \mu_{\mathrm{weight}}(t) \triangleq \sum_{b=1}^{B} b\,\hat w_t(b),
\end{equation}
and the lead--lag gap $\Delta\mu(t)\triangleq \mu_{\mathrm{weight}}(t)-\mu_{\mathrm{data}}(t)$.
A positive $\Delta\mu(t)$ indicates that the adversary's weight distribution is shifted toward \,higher-index\, bins than the empirical prompt share.
Under our convention (Section~\ref{sec:method_grouping}) that \texttt{accbin\_0} corresponds to the \emph{lowest} pass@8 interval and larger indices correspond to \emph{higher} pass@8 (more solvable) prompts, this means the adversary emphasizes the current \emph{learnable frontier} rather than simply mirroring the bulk of unsolved mass.
Figure~\ref{fig:extra_diagnostics} (left) shows that $\Delta\mu(t)$ is strongly positive early in training and decays over time, with faster decay at larger model scales.
For Qwen3-8B, $\Delta\mu(t)$ eventually becomes slightly negative, consistent with the data distribution migrating quickly into high pass@8 bins while the adversary maintains pressure on the remaining low-pass@8 cases.
Overall, the decay of $\Delta\mu(t)$ provides a compact scalar signature of the ``traveling wave'' curriculum: the adversary leads and the data distribution catches up as the policy improves.

\begin{figure*}[t]
    \centering
    \begin{subfigure}[t]{0.495\textwidth}
        \centering
        \includegraphics[width=\linewidth]{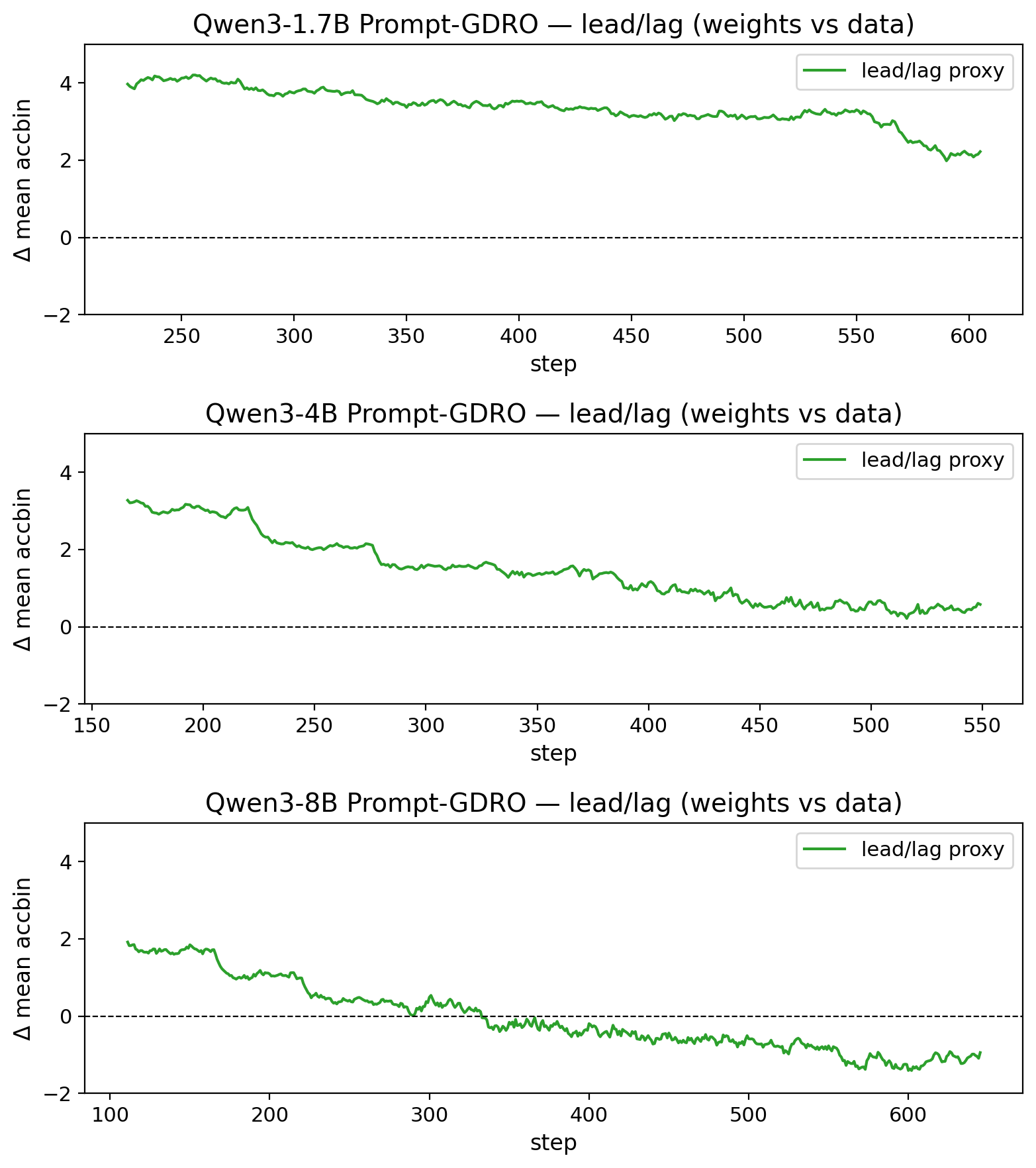}
        \caption{\textbf{Prompt-GDRO lead--lag proxy.} $\Delta\mu(t)=\mu_{\mathrm{weight}}(t)-\mu_{\mathrm{data}}(t)$, where $\mu_{\mathrm{data}}(t)=\sum_b b\,q_t(b)$ and $\mu_{\mathrm{weight}}(t)=\sum_b b\,\hat w_t(b)$.}
        \label{fig:lead_lag}
    \end{subfigure}\hfill
    \begin{subfigure}[t]{0.495\textwidth}
        \centering
        \includegraphics[width=\linewidth]{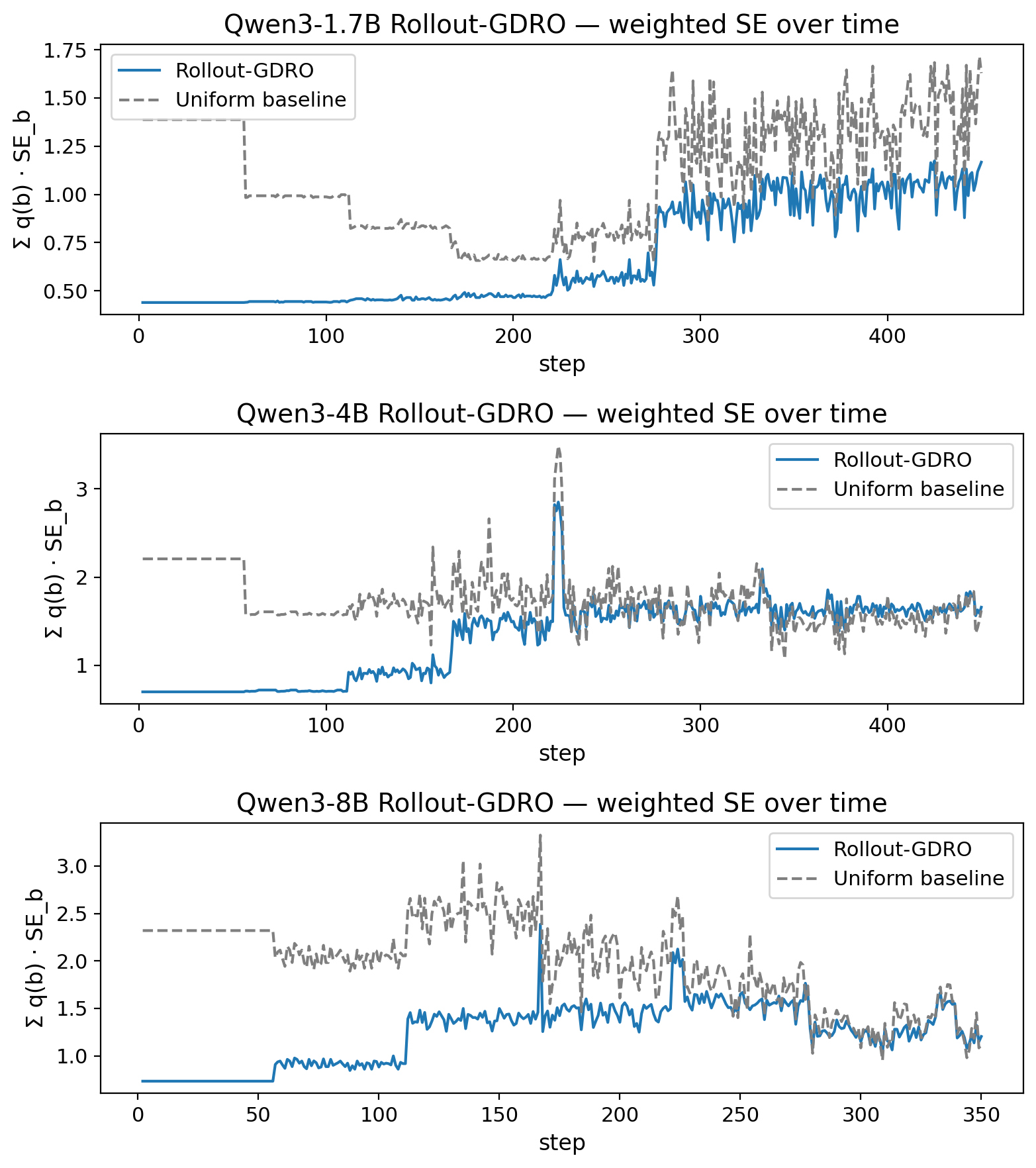}
        \caption{\textbf{Rollout-GDRO weighted SE proxy.} $\mathrm{WSE}(t)=\sum_b q_t(b)\hat\sigma_b/\sqrt{n_b(t)}$ compared to a compute-matched uniform baseline.}
        \label{fig:wse_proxy}
    \end{subfigure}
    \caption{\textbf{Two diagnostics for the two adversaries.} (Left) Prompt-GDRO weights initially lead the empirical data distribution in bin index (a ``learnable frontier'') and progressively align as the prompt-share distribution catches up. (Right) Rollout-GDRO reduces an offline weighted standard-error proxy relative to a uniform allocation with the same mean rollout budget, consistent with variance-aware compute allocation.}
    \label{fig:extra_diagnostics}
\end{figure*}

\begin{keyfindingbox}{Key Finding 2}
Prompt-GDRO generates an emergent curriculum that decouples the learning signal from dataset frequency. It applies disproportionate pressure to rare, hard bins, creating a ``traveling wave'' of difficulty that adapts to the model's capacity.
\end{keyfindingbox}

\subsection{Qualitative Analysis: Adaptive Compute Allocation (Rollout-GDRO)}
\label{sec:rollout_analysis}

While Prompt-GDRO improves robustness by altering \textit{what} data the model sees, Rollout-GDRO improves robustness by altering \textit{how deeply} the model explores that data. To understand this mechanism, we analyze the adversarial budgeter's behavior through three complementary visualizations: the continuous budget frontier, macro-level scalar dynamics, and discrete allocation snapshots.

\subsubsection{The Budget Frontier}
Figure \ref{fig:rollout_mechanism} visualizes the direct contrast between data frequency and resource allocation. The left column shows the prompt share (dataset mass), while the right column shows the allocated rollout budget per prompt.

\begin{figure*}[t]
    \centering
    \includegraphics[width=\textwidth]{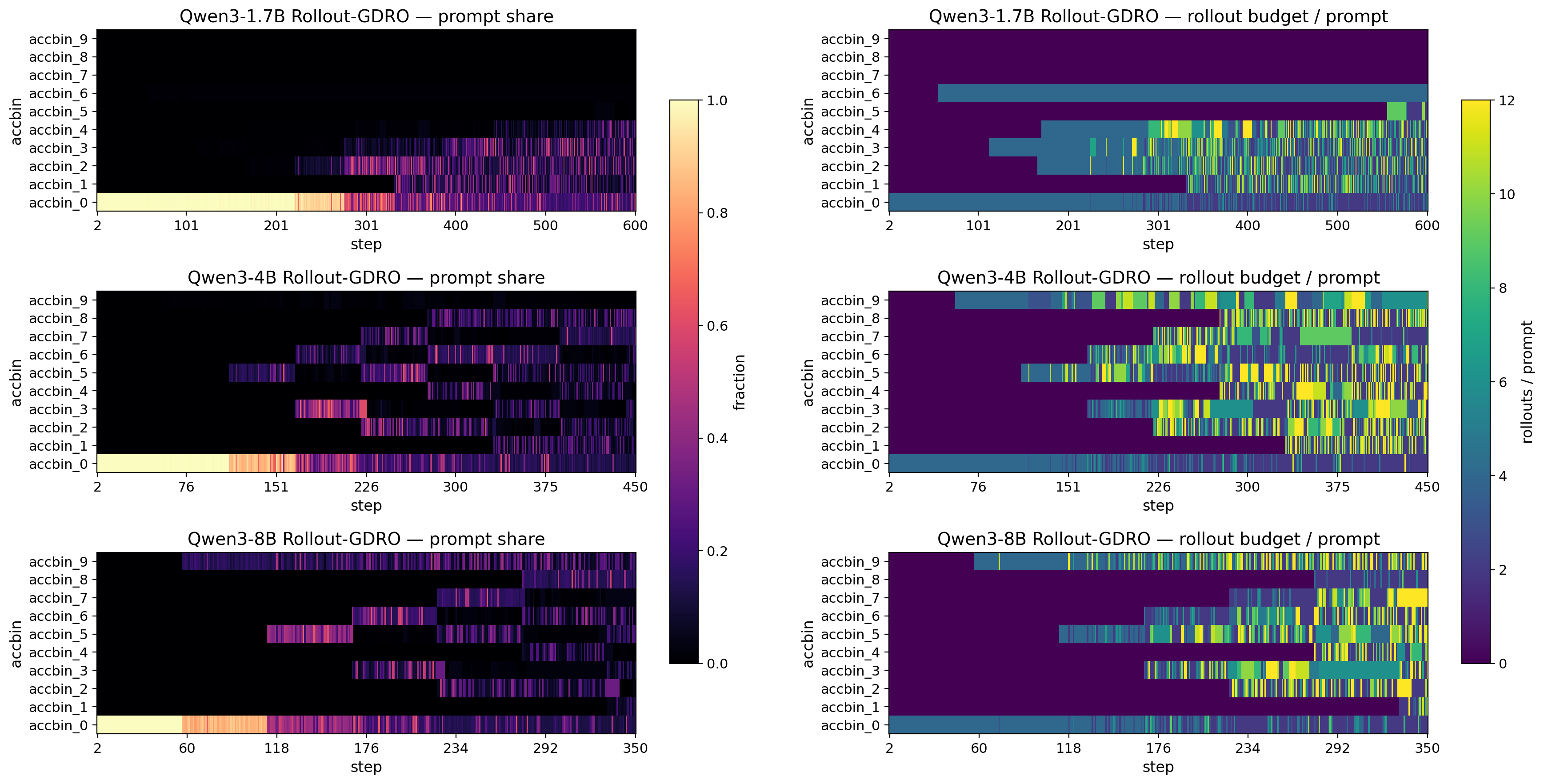}
    \caption{\textbf{The Budget Frontier.} A comparison of the dataset's natural difficulty distribution (Left) versus the adversarial rollout allocation (Right) for 1.7B, 4B, and 8B models. The color intensity in the Right column represents the number of rollouts assigned per prompt (from purple $\approx 2$ to yellow $\approx 12$). Note how the budgeter shifts compute intensity to the ``transition zone,'' decoupling resource allocation from data frequency: even when hard bins are rare (dark left), they receive maximum compute (bright right).}
    \label{fig:rollout_mechanism}
\end{figure*}

This comparison offers the clearest qualitative proof of our method's economic policy. For the 8B model (bottom row), the dataset mass (left) remains concentrated in easier bins for hundreds of steps. However, the budgeter (right) rapidly identifies the emerging capability in \texttt{accbin\_6} and above, locking high-compute resources onto these rare prompts. In contrast, the 1.7B model (top row) takes significantly longer to leave \texttt{accbin\_0}, and the budget frontier moves sluggishly, confirming that the adversary adapts the curriculum pace to the model's intrinsic scaling laws.

\subsubsection{Macro-Dynamics of Allocation}
To quantify these trends, Figure \ref{fig:rollout_scalars} presents the macro-level training dynamics. The four stacked traces---Mean Accuracy Bin Index, High-Bin Mass, High-Bin Budget Share, and Entropy---provide a unified view of how Rollout-GDRO migrates compute toward harder domains.

\begin{figure}[t]
    \centering
    \includegraphics[width=0.95\linewidth]{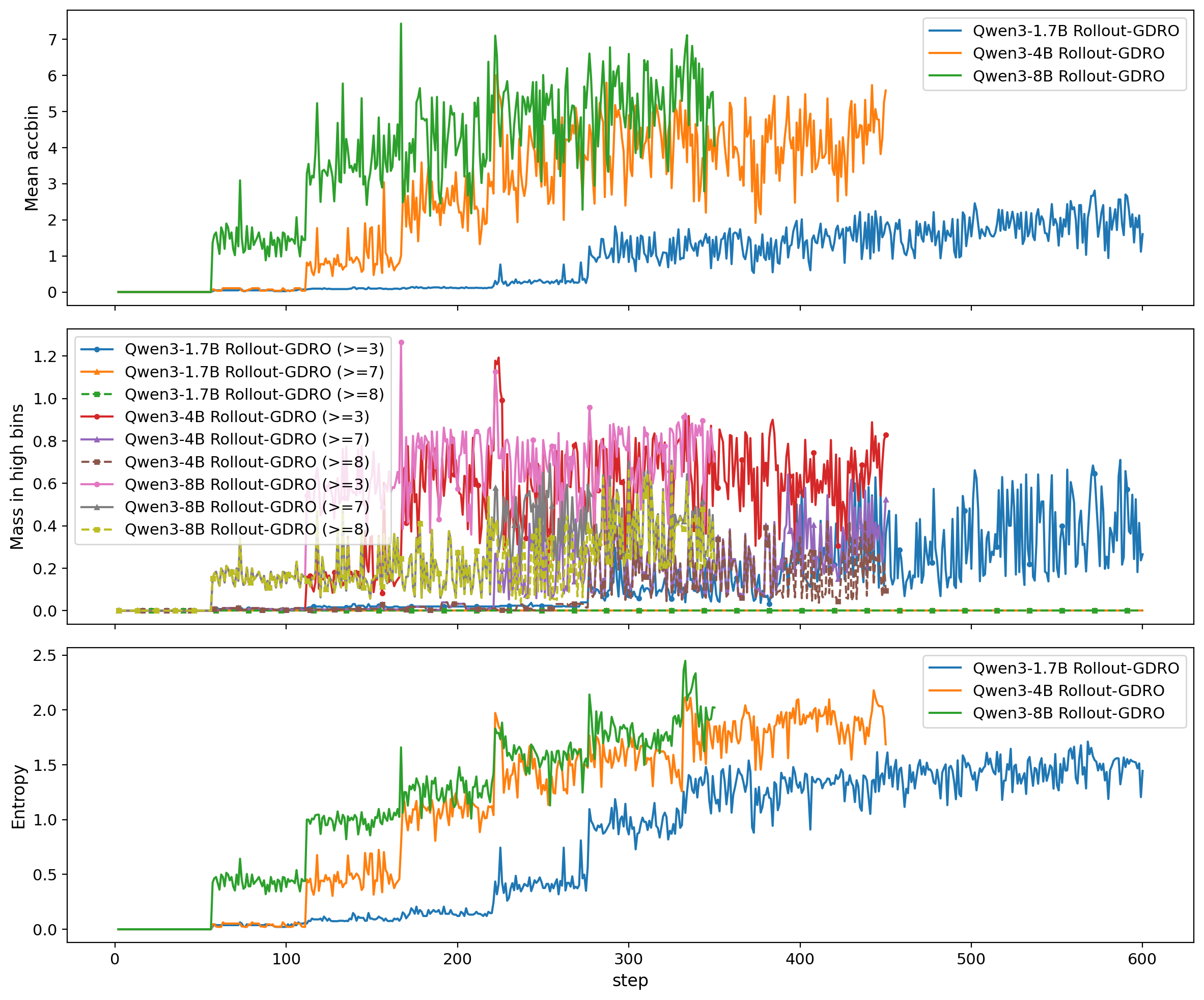}
    \caption{\textbf{Macro-Level Allocation Dynamics (Rollout-GDRO).} (Top) The Mean Accuracy Bin Index tracks the rising difficulty of targeted prompts. (Middle) The ``Mass in High Bins'' trace serves as quantitative evidence of the rollout adversary: it shows the dual variable mechanism keeping the total budget fixed while aggressively reweighting toward hard groups ($\ge \texttt{accbin\_8}$). (Bottom) Entropy metrics confirm that the 8B model (green) sustains a diverse allocation strategy even as it conquers lower difficulties, contrasting with the slower migration of the 1.7B model (blue).}
    \label{fig:rollout_scalars}
\end{figure}

The ``Mass in High Bins'' trace is particularly revealing. It serves as direct evidence of the DRO objective in action: as the models improve, the adversary steadily reallocates the fixed global budget toward bins $\ge 7$. This reallocation correlates directly with the inflection points observed in the pass@8 accuracy tables. Furthermore, the shared axes highlight distinct scaling trends: the 8B model (green line) learns to push its budget into high-difficulty bins significantly earlier than the 1.7B model (blue line), validating that larger capacity enables more aggressive curriculum acceleration.

\noindent\textbf{Variance-aware compute efficiency.}
Beyond shifting budget toward harder bins, we can directly test whether the rollout adversary reduces the \emph{uncertainty} of the bin-wise training signal at fixed compute.
Let $n_b(t)$ denote the realized number of rollouts per prompt allocated to bin $b$ at step $t$, and let $q_t(b)$ denote the prompt share.
Using an offline bin-wise variability proxy $\hat\sigma_b$ (estimated once from the training logs as the empirical standard deviation of a per-prompt GRPO signal within each bin, e.g., the prompt-level loss $\ell(x;\theta)$), we define a weighted standard-error proxy
\begin{equation}
    \mathrm{WSE}(t)\triangleq \sum_{b=1}^{B} q_t(b)\,\frac{\hat\sigma_b}{\sqrt{n_b(t)}}.
\end{equation}
We compare against a compute-matched uniform-rollout baseline by setting $n_b(t)\equiv \bar n$ for all bins $b$, which satisfies the mean-rollout constraint $\sum_{b=1}^B q_t(b)\, n_b(t)=\bar n$ and therefore matches overall sampling compute at each step.
As shown in Figure~\ref{fig:extra_diagnostics} (right), Rollout-GDRO consistently attains lower $\mathrm{WSE}(t)$ than the uniform baseline throughout training.
Averaged over the plotted horizon, this corresponds to relative reductions of \textbf{37.1\%} (1.7B), \textbf{22.6\%} (4B), and \textbf{33.4\%} (8B), supporting the interpretation that Rollout-GDRO improves gradient information efficiency by allocating more rollouts to high-variance bins.

\subsubsection{Discrete Economic Phases}
Finally, Figure \ref{fig:rollout_snapshots} decomposes the continuous training process into discrete ``chapters'' of the curriculum. These snapshots clarify the magnitude of the adversarial intervention at key training stages (Start, Early, Mid, Late).

\begin{figure*}[t]
    \centering
    \includegraphics[width=\textwidth]{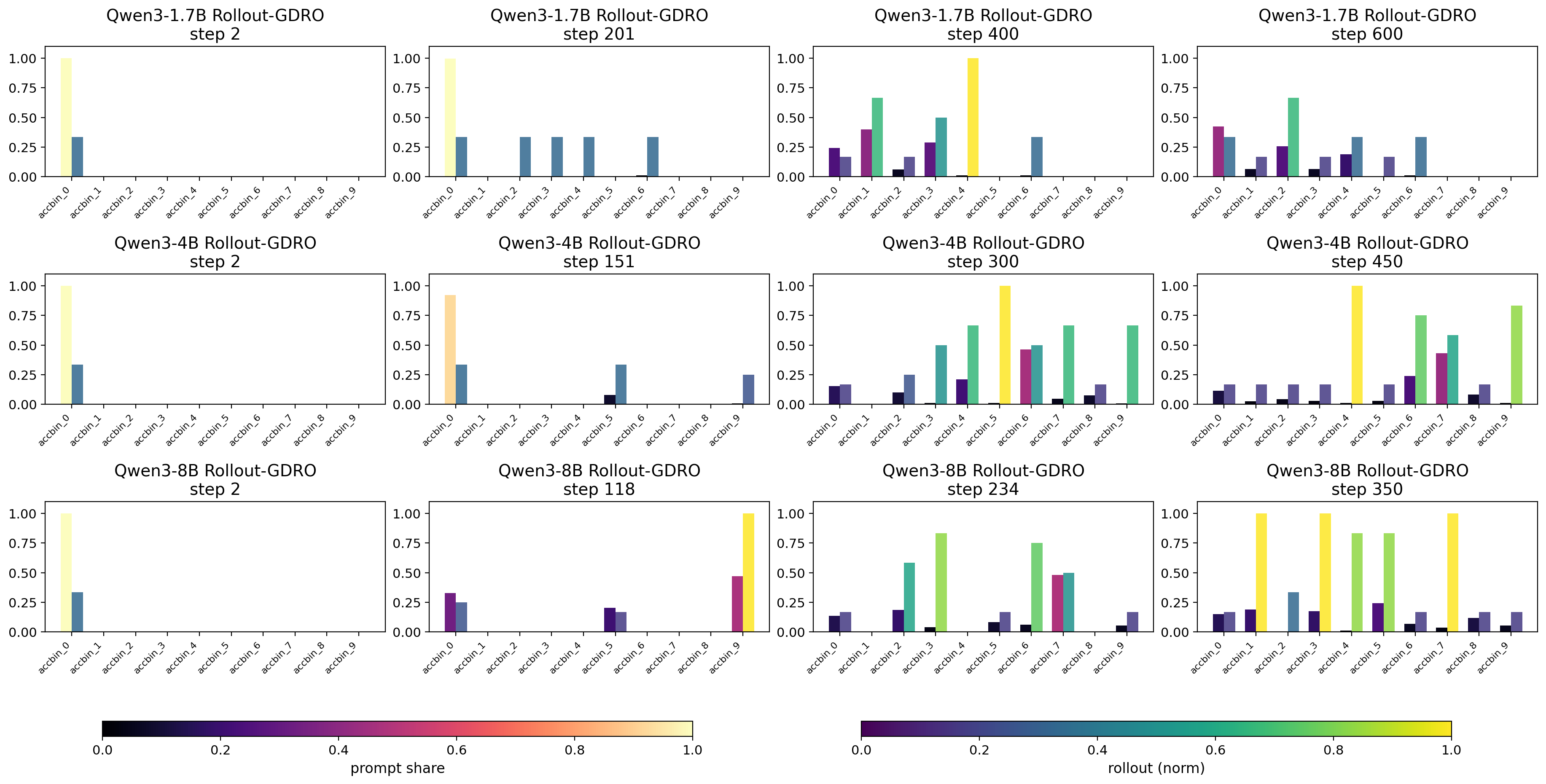}
    \caption{\textbf{Snapshots of the Allocation Economy.} Paired bars at four canonical steps showing the dataset share (dark blue) versus the normalized rollout budget (light blue) for each bin. This visualizes the \textbf{``Multiplier Effect''}: by Step 300, the 4B model allocates $>80\%$ of its budget to \texttt{accbin\_5+}, even though these bins contain $<20\%$ of the data. This explicitly demonstrates how Rollout-GDRO amplifies the signal from rare, high-value prompts.}
    \label{fig:rollout_snapshots}
\end{figure*}

The paired bars (Dataset Share vs. Rollout Budget) illustrate a massive \textbf{Multiplier Effect}. For example, at Step 300, the 4B model allocates over $80\%$ of its compute budget to bins $\ge 5$, despite these bins constituting less than $20\%$ of the training data. This confirms that our method creates a highly non-uniform economic policy that gives 5--10$\times$ more rollouts to the ``reasoning frontier'' than a uniform baseline would. This behavior is model-specific: the 8B row shows an even faster shift, embracing high-bin budgeting early in training (Step 118), which aligns with the rapid saturation of easy tasks observed in our qualitative analysis.

\begin{keyfindingbox}{Key Finding 3}
Rollout-GDRO autonomously identifies the ``transition zone'' of difficulty and concentrates computational resources there. This results in a highly non-uniform economic policy where rare, high-value prompts receive up to 10$\times$ more rollouts than a uniform baseline, maximizing gradient information efficiency.
\end{keyfindingbox}

\section{Additional Related Work}
\label{sec:related_work}

Our work studies reasoning post-training at the intersection of (i) RLVR/GRPO-style policy optimization, (ii) distributionally robust optimization (DRO) and group robustness, and (iii) adaptive allocation of training-time compute.

\subsection{RLVR and Post-training for Reasoning}

RL-based post-training has become central for improving reasoning behaviors in LLMs. PPO \citep{schulman2017proximal} underpins RLHF-style alignment \citep{ouyang2022training}, while GRPO \citep{shao2024deepseekmath} offers a value-free, group-normalized alternative that has proven effective for verifiable-reward domains such as math. Parallel research improves reward quality and credit assignment through process supervision and step-level signals \citep{lightman2023let,wang2023mathshepherd}, as well as iterative refinement and self-improvement mechanisms \citep{gulcehre2023reinforced}. More recently, large-scale RLVR has enabled open reasoning models trained primarily from verifiable rewards, highlighting the potential of pure RL to elicit sophisticated behaviors such as self-reflection and verification \citep{guo2025deepseekr1}.

However, a growing body of work suggests that \emph{where} RLVR learns and \emph{how} compute is spent are both highly non-uniform. Token-level analyses indicate that RLVR gains can concentrate on a small fraction of high-entropy ``forking'' tokens that control reasoning branches \citep{wang2025beyond8020}, while other studies argue RLVR may implicitly incentivize correct reasoning patterns already latent in the base model \citep{wen2025rlvr_incentivizes}. At inference time, test-time scaling via longer ``thinking'' traces can be non-monotonic and may induce overthinking \citep{ghosal2025thinkingmore}, and long-CoT reasoning models can exhibit looping pathologies under low-temperature decoding \citep{pipis2025looping}. In parallel, several works explore alternative ways to trade off accuracy and compute, including budget-aware evaluation and compute-normalized comparisons \citep{wang2024token_economies}, and inference-time orchestration strategies that decouple accuracy from raw CoT length \citep{madaan2025rethinking_thinking}. Our work is complementary: rather than proposing a new reward or inference strategy, we focus on \emph{training-time} mechanisms that adaptively steer (a) which prompts are sampled and (b) how many rollouts are allocated, with the goal of improving robustness and compute efficiency.

\subsection{Distributionally Robust Optimization in Supervised Learning}

DRO formalizes robustness by minimizing worst-case risk over an ambiguity set of distributions around the empirical training distribution \citep{ben1999robust,rahimian2019dro_review}. In supervised learning, a common choice is divergence-based ambiguity sets, yielding objectives that emphasize performance under distribution shift and provide statistical guarantees \citep{namkoong2016stochastic,duchi2021statistics}. Wasserstein-based DRO offers an alternative geometry with tractable reformulations and strong finite-sample guarantees \citep{esfahani2018wasserstein}. Within deep learning, GDRO \citep{sagawa2020distributionally} operationalizes robustness to hidden stratification and spurious correlations by optimizing the maximum loss across pre-defined groups, and has become a standard tool for improving worst-group accuracy. On the algorithmic side, group-robust learning admits a natural game-theoretic and online learning interpretation; in particular, \citet{soma2022optimal} connect GDRO to no-regret dynamics (e.g., EXP3 variants), which directly motivates our use of \textbf{GDRO-EXP3P} for adversarial prompt reweighting.

Our setting departs from classical supervised GDRO in two ways. First, we \emph{do not assume static group labels}; instead we use an \emph{online} difficulty classifier based on pass@k to define groups that evolve with the policy. Second, we introduce a second adversary that controls \emph{compute allocation} (rollouts) under a budget constraint, extending the DRO perspective beyond data distribution shifts to \emph{training-time resource shifts}.

\subsection{Robust and Distributionally Robust Reinforcement Learning}

Robust RL traditionally models uncertainty in environment dynamics and seeks policies that perform well under worst-case transition perturbations, e.g., robust Markov decision processes \citep{iyengar2005robust,nilim2005robust} and distributionally robust MDP formulations \citep{xu2012distributionally}. Recent work develops statistical and computational characterizations of robust RL \citep{panaganti2022robust}. In contrast, our work keeps the underlying environment fixed and instead treats \emph{prompt difficulty} and \emph{compute allocation} as adversarially controlled quantities during LLM post-training. This yields a robustness lens that is closer to group robustness over tasks/prompts than to worst-case transition uncertainty.

\subsection{Curriculum, Adaptive Compute, and Data Value}

Adaptive curricula and compute allocation are increasingly recognized as first-class components of reasoning systems. Curriculum-based post-training pipelines such as Light-R1 \citep{wen2025lightr1} explicitly stage data difficulty and objectives (SFT/DPO/RL) to elicit long-CoT behaviors. At inference time, scaling laws and compute-allocation policies highlight that difficult instances require more budget, but that naive increases in ``thinking'' can be inefficient or unstable \citep{snell2024scaling,ghosal2025thinkingmore}. Recent work proposes learning policies to allocate test-time compute \citep{setlur2024learning} and explores search-based or tree-structured generation to improve exploration of the reasoning space \citep{xie2024mcts,hou2025treerl}. Our \textbf{Rollout-GDRO} can be viewed as moving this idea to \emph{training time}: we allocate rollout budgets to groups to improve gradient estimator quality under a global constraint, akin to adaptive variance reduction \citep{rubinstein2016simulation}.

Recent work has also begun to articulate compute-optimal ``RL scaling'' workflows for LLM post-training by empirically studying how to allocate a fixed sampling budget across (i) the number of problems per batch, (ii) the number of parallel rollouts per problem (GRPO group size), and (iii) the number of sequential update steps. In particular, the IsoCompute Playbook reports that compute-optimal rollout parallelism can often be summarized by simple sigmoidal/logit fits as total sampling compute grows \citep{cheng2026isocompute}. In contrast, our approach is more algorithmic: Prompt-GDRO and Rollout-GDRO adaptively reshape the effective prompt distribution and per-group compute online, without assuming a pre-fit scaling law. Concretely, we hold the mean sampling budget fixed and redistribute it across online-defined difficulty subgroups within each batch, rather than optimizing compute allocation across global axes such as problems-per-batch, rollouts-per-problem, or number of update steps. Relatedly, \citet{qi2025brpo} propose Budget Relative Policy Optimization (BRPO) to optimize anytime reasoning performance across varying token budgets, complementing our training-time allocation perspective.

Finally, recent theoretical discussion emphasizes that the \emph{value of data} depends on the learner's computational constraints and even on data ordering. The notion of \emph{epiplexity} formalizes this perspective and motivates principled data selection and dataset interventions \citep{finzi2026epiplexity}; see also community discussion \citep{finzi2026epiplexity_thread}. Our work is exploratory in this broader direction: we study whether simple, online, adversarial control loops over prompt reweighting and compute allocation can serve as a practical ``steering subsystem'' for reasoning post-training, complementing concurrent efforts that analyze RLVR mechanisms \citep{wang2025beyond8020,wen2025rlvr_incentivizes}, rethink thinking-token tradeoffs \citep{madaan2025rethinking_thinking}, and diagnose instabilities such as looping \citep{pipis2025looping}.

\section{Limitations and Future Work}
\label{sec:limitations}

\paragraph{Bridge: from two adversaries to open questions.}
Our empirical results suggest that the two adversaries introduced in this work---\emph{Prompt-GDRO} (adaptive prompt reweighting over online difficulty bins) and \emph{Rollout-GDRO} (adaptive rollout allocation under a global compute budget)---capture complementary levers for improving reasoning post-training.
Both adversaries are driven by the same online difficulty signal (stable binning via empirical pass@k),
but intervene at different points in the pipeline: Prompt-GDRO shapes the \emph{data distribution} presented to the learner, while Rollout-GDRO shapes the \emph{per-sample signal-to-noise ratio} of policy-gradient updates by modulating rollout counts.
This coupling is central to the ``beyond uniform'' thesis of our framework, yet our current study is necessarily exploratory and leaves several important questions unresolved.

\paragraph{Empirical scope and missing full-factorial ablations.}
This paper is intended as an exploratory report to the community: we demonstrate that \emph{distribution-shaping} tools from GDRO-style thinking can be productively instantiated inside a modern reasoning post-training stack, but we do not claim to have exhaustively optimized the design space.
In particular, we have not yet performed a full factorial ablation over \emph{all} distribution-shaping components and their interactions (e.g., prompt selection/reweighting, compute allocation, and online binning choices).
Concrete future work includes:
\begin{itemize}[nosep,leftmargin=*]
    \item \textbf{Binning and classifier hyperparameters.} We used a fixed binning scheme in most experiments; broader sweeps over the number of bins, smoothing horizons, and bin-stability heuristics are needed. In preliminary sweeps we often observed performance peaking around $\approx 6$ bins, but this is not yet a robust conclusion.
    \item \textbf{Joint training with multiple adversaries.} Most experiments isolate Prompt-GDRO or Rollout-GDRO; a systematic study of their \emph{joint} behavior (and staged curricula) is still missing.
    \item \textbf{Rollout allocator design.} We only explored a limited set of discrete rollout arms and budget schedules; future work should examine broader arm sets, alternative constrained optimizers, and adaptive rollout bounds.
\end{itemize}

\paragraph{RL scaling and compute-optimal post-training.}
Our experiments are conducted at a single (moderate) scale, and we do not yet understand how adversarial prompt reweighting and rollout budgeting interact with emerging ``RL scaling'' behavior as total post-training compute grows. Recent work \citep{liu2025tricks_traps,liu2025prorl,khatri2025scaling_rl_compute,cheng2026isocompute} study compute-optimal allocation rules and scaling behavior for RL of LLMs (e.g., by fitting simple sigmoidal/logit trends for optimal rollouts-per-problem under larger budgets). A natural next step is to evaluate whether Prompt-GDRO and Rollout-GDRO shift these compute-optimal frontiers (or their saturation points) across larger compute budgets, model sizes, and prompt mixtures.

\paragraph{Adversarial game computation and systems overhead.}
A practical limitation of our approach is that it introduces additional online machinery beyond standard GRPO:
(i) computing and maintaining a difficulty classifier (pass@k statistics, stable bin assignment),
(ii) updating adversarial distributions (EXP3P-style weight updates),
and (iii) (for Rollout-GDRO) enforcing compute constraints while selecting discrete rollout arms. As an illustration for trade-offs, we measure the driver-side advantage stage time (\texttt{timing\_s/adv}, in sec/step) and find that for the Qwen3-4B runs (mean over the last 100 logged steps after warmup) GRPO requires \texttt{0.043} sec/step, Prompt-GDRO requires \texttt{0.355} sec/step, and Rollout-GDRO requires \texttt{0.446} sec/step. This overhead is distinct from our \emph{sampling-compute neutrality} claims, which refer to the mean rollout budget.

We find GDRO’s adversary/advantage-side bookkeeping can materially increase the driver-side advantage stage, and is one contributor to end-to-end slowdown.
While each component is lightweight in isolation, their combination can create nontrivial systems overhead at scale.
An important engineering direction is to design \emph{asynchronous} and \emph{streaming} variants (e.g., delayed bin updates, batched adversary steps, or partially offloaded bookkeeping) that preserve the core objective while minimizing training slowdown.

\paragraph{Sensitivity to online binning and reward noise.}
Our ``group'' notion is induced by an online estimator of difficulty, rather than static metadata.
While this avoids reliance on brittle human-defined group labels, it also introduces potential noise sources:
estimated pass@k may have high variance early in training, and bin boundaries can induce discontinuous group reassignment.
These effects can bias both adversaries if not handled carefully.
Future work should explore principled uncertainty-aware binning (e.g., Bayesian estimators or confidence-bound assignment rules),
as well as robustness to verifier calibration drift and reward-model nonstationarity.
It may also be valuable to incorporate \emph{process-level} or \emph{stepwise} supervision signals (when available) to stabilize difficulty estimation, rather than relying solely on outcome-only pass@k.

\paragraph{Generalization and evaluation beyond the current pipeline.}
Although we report improvements on advanced benchmarks, we do not yet provide a systematic evaluation protocol for distribution shift.
A natural next step is to test whether the adversarial curricula learned on one training mixture transfer to new mixtures,
or whether they overfit to dataset-specific artifacts.
More broadly, our method suggests a future direction where GDRO-style training is used not only to improve average in-distribution metrics,
but to target robustness to hidden stratification and distribution shift \citep{sagawa2020distributionally,oakden2020hidden}.
Crucially, this should be framed as future work: \emph{out-of-distribution generalization is not a primary motivation of this paper}, but an important downstream question enabled by the methodology.

\paragraph{Toward learning from the model's own experience.}
A key longer-term direction is to couple our adversarial distribution shaping with \emph{experience generation} and \emph{continual post-training}.
For example, one can imagine a closed loop where the model:
(i) proposes new problems or perturbations, (ii) evaluates its own failures, and (iii) uses Prompt-GDRO/Rollout-GDRO to prioritize the resulting frontier.
This connects naturally to self-training and self-improvement paradigms such as STaR-style bootstrapping \citep{zelikman2022star}, Quiet-STaR-style implicit ``thinking'' training \citep{zelikman2024quietstar}, self-rewarding / judge-based optimization \citep{yuan2024selfrewarding}, and RLAIF-style scalable feedback \citep{lee2023rlaif}.
Realizing such a pipeline requires addressing continual-learning issues (e.g., catastrophic forgetting \citep{rolnick2019experience}),
maintaining replay buffers \citep{schaul2015prioritized}, and preventing reward hacking under self-generated supervision.

\paragraph{Beyond exponential-weights GDRO: richer ambiguity sets and scalable solvers.}
Our current instantiation uses exponential-weights style updates over a discrete set of groups/arms.
However, the broader DRO literature offers many alternative ambiguity sets and solution methods that may be better suited for future scaling:
$f$-divergence DRO admits stochastic-gradient formulations \citep{namkoong2016stochastic},
and recent work studies computationally efficient large-scale solvers for DRO objectives such as CVaR and $\chi^2$-based uncertainty sets \citep{levy2020large}.
Wasserstein DRO provides a complementary metric-based robustness lens with tractable reformulations and finite-sample guarantees \citep{esfahani2018wasserstein}.
On the RL side, robust MDP formulations \citep{iyengar2005robust,nilim2005robust} and scalable $\phi$-divergence regularization approaches \citep{panaganti2024phidivergence} suggest additional ways to model uncertainty and allocate resources under environment shift.
A concrete research agenda is to identify which ambiguity sets best correspond to the \emph{operational} failure modes of reasoning post-training (e.g., verifier mismatch, data mixture drift, or hard-sample scarcity), and to develop scalable solvers compatible with modern LLM training.

\paragraph{Beyond robustness: adversarial reasoning objectives for safety and risk.}
Finally, our adversaries currently act on \emph{what} is trained (prompt distribution) and \emph{how intensively} it is trained (rollout budgets),
but not on richer forms of adversarial reasoning objectives.
An important future direction is to design adversaries that target \emph{specific} reasoning desiderata beyond accuracy, such as safety, risk avoidance, and constraint satisfaction.
This connects to alignment frameworks that rely on rule-based or AI-generated feedback \citep{bai2022constitutional,lee2023rlaif},
as well as risk-sensitive optimization objectives.
We view this as a distinct problem from classical DRO: rather than protecting against distributional uncertainty alone,
the goal becomes to adversarially surface and correct \emph{undesirable reasoning behaviors} (e.g., unsafe tool use, brittle shortcuts, or overconfident hallucinations) under realistic deployment constraints.

\section{Conclusion}
\label{sec:conclusion}

We introduced two multi-adversary GDRO frameworks for reasoning post-training that move beyond the static uniformity of standard GRPO by defining \emph{dynamic} groups via an online difficulty classifier (stable pass@k binning). On this shared grouping layer, Prompt-GDRO uses an EMA-debiased GDRO-\textsc{EXP3P} reweighting to concentrate updates on persistently hard bins without frequency artifacts, and Rollout-GDRO uses a GDRO-\textsc{EXP3P} adversary with a shadow-price controller to redistribute rollouts under a fixed mean compute budget. In this work, we evaluate Prompt-GDRO and Rollout-GDRO independently (no coupling); studying their joint dynamics is left to future work (\S\ref{sec:limitations}). Across Qwen3-Base scales, both mechanisms are compute-neutral in the sense of \S\ref{sec:experiments} yet improve pass@8 over GRPO by up to 13.13\% (Prompt-GDRO) and 10.64\% (Rollout-GDRO), and diagnostics indicate an emergent curriculum as sampling weights and rollout budgets track the evolving reasoning frontier. We hope these results motivate further work on dynamic grouping and DRO-style training games as principled components of future reasoning post-training pipelines.

\newpage
\bibliography{ref,gdro-ref}
\bibliographystyle{colm}
\newpage
\appendix
\section{Experiment Details}
\label{app:experiment_details}

In this section, we detail the experimental setup, including the shared optimization hyperparameters and the specific configurations for our adversarial mechanisms. All experiments were conducted using the Qwen3-Base model family (1.7B, 4B, 8B) using BFloat16 (BF16) mixed precision and FlashAttention 2. The code and configuration files used to reproduce these results are available at \url{https://github.com/kishanpb/verl-gdro}.

\subsection{Shared Training Hyperparameters}
All methods (GRPO Baseline, Prompt-GDRO, and Rollout-GDRO) utilize a common post-training foundation based on the Group Relative Policy Optimization (GRPO) objective.

\subsubsection{Optimization \& Architecture}
\begin{itemize}
    \item \textbf{Global Train Batch Size}: 256
    \item \textbf{Global Validation Batch Size}: 128
    \item \textbf{Total Training Steps}: 1000
    \item \textbf{Optimizer}: AdamW
    \item \textbf{Actor Learning Rate}: $1 \times 10^{-6}$
    \item \textbf{KL Penalty Coefficient} ($\beta_{\text{KL}}$): 0.001
    \item \textbf{PPO Clip Range}: $[1 - \epsilon_{\text{low}}, 1 + \epsilon_{\text{high}}]$ where $\epsilon_{\text{low}}=0.2$, $\epsilon_{\text{high}}=0.28$
    \item \textbf{Advantage Normalization}: Yes (Normalized by group standard deviation)
    \item \textbf{Advantage Clipping}: $[-5, 5]$
\end{itemize}

\subsubsection{Rollout Generation}
\begin{itemize}
    \item \textbf{Inference Engine}: vLLM
    \item \textbf{Training Rollouts per Prompt} ($G$): 4 (Base setting)
    \item \textbf{Validation Rollouts per Prompt}: 8
    \item \textbf{Sampling Temperature}: 0.6 (Training)
    \item \textbf{Top-p}: 0.8
    \item \textbf{Top-k}: 20
    \item \textbf{Reward}: Verifiable math correctness with $r(x,y)\in\{-1,+1\}\subset[-1,1]$, implemented via the \texttt{math/math-dapo} modules in \texttt{verl}.
\end{itemize}

\subsection{Adversarial Configuration}
Our Multi-Adversary framework introduces specific hyperparameters for the \textbf{EXP3P} algorithms governing data sampling and compute allocation.

\subsubsection{Prompt-GDRO (The Data Adversary)}
This mechanism reweights the prompt distribution based on the intensive difficulty of online groups.
\begin{itemize}
    \item \textbf{Grouping Mechanism}: Online Pass@k (10 bins)
    \item \textbf{Adversary Algorithm}: EMA-Debiased GDRO-EXP3P
    \item \textbf{Adversary Learning Rate} ($\eta_q$): 0.65
    \item \textbf{Exploration Rate} ($\gamma$): 0.01
    \item \textbf{Score EMA Decay} ($\beta$): 0.12
    \item \textbf{Max Class Weight Cap}: 15.0
    \item \textbf{Loss Normalization}: Normalized by class share (to prevent frequency bias)
\end{itemize}

\subsubsection{Rollout-GDRO (The Compute Adversary)}
This mechanism allocates discrete rollout counts $n_b$ to minimize gradient variance under a global budget constraint.
\begin{itemize}
    \item \textbf{Grouping Mechanism}: Online Pass@k (10 bins, edges at $0.1, 0.2, \dots, 0.9$)
    \item \textbf{Rollout Arm Range}: $n \in [n_{\min}, n_{\max}]$ where $n_{\min}=2, n_{\max}=12$ (Multiplier $3.0\times$ base)
    \item \textbf{Global Budget Constraint} ($\bar{n}$): 4 rollouts (average per prompt)
    \item \textbf{Dual Learning Rate} ($\alpha_{\mu}$): 0.05
    \item \textbf{Arm Learning Rate} ($\eta$): 0.65
    \item \textbf{Arm Exploration Rate} ($\gamma$): 0.01
    \item \textbf{Arm Score EMA Decay}: 0.4
    \item \textbf{Budget Matching}: Exact constrained selection via Dynamic Programming
\end{itemize}
\newpage
\section{Main Theoretical Results: A Game-and-Variance View}
\label{sec:main_theory}

This section develops a unified theoretical lens for the \emph{two adversarial controllers} in our framework:
(i) \textbf{Prompt-GDRO}, which adaptively reshapes the prompt distribution to emphasize difficult bins, and
(ii) \textbf{Rollout-GDRO}, which adaptively reallocates rollouts across bins to reduce estimation noise under a compute budget.
Our goal is \emph{not} to provide deep-network convergence guarantees for GRPO, but rather to (a) formalize the
surrogate objectives implicitly optimized by these controllers, and (b) connect their update rules to standard
no-regret / mirror-descent analyses that explain the qualitative behaviors observed empirically (e.g., ``traveling waves'' and staircase compute allocation).

A condensed statement of these results (omitting most proofs) appears in Section~\ref{sec:analysis};
this appendix provides complete statements and proofs for reference.

Throughout, we adopt the GDRO formulation from \Cref{sec:preli} (Eq.~\eqref{eq:gdro_general}).
Let $g(x)\in\{1,\dots,B\}$ be the (online) grouping rule from \Cref{sec:analysis}, and define the group losses below. We use $\{L_b\}_{b=1}^B$ primarily in the Prompt-GDRO analysis; Rollout-GDRO instead optimizes a separate budgeted variance objective.
\begin{equation}
    L_b(\theta)\;\triangleq\;\mathbb{E}\!\left[\ell(x;\theta)\mid g(x)=b\right],
    \qquad b\in\{1,\dots,B\},
    \label{eq:Lb_def_main_theory}
\end{equation}
where $\ell(x;\theta)$ is the prompt-level GRPO loss from \Cref{sec:preli_grpo}.%
\footnote{In reward maximization form, one may take $L_b(\theta)=-J_b(\theta)$, where $J_b$ is the group-conditional expected reward (this sign convention is standard in RL; see, e.g., \citealp{agarwal2019reinforcement}).}

\subsection{Prompt-GDRO as Entropic GDRO and No-Regret Game Dynamics}
\label{sec:theory_prompt_gdro}

We start from the canonical finite-group robust objective
\begin{equation}
    \min_{\theta}\;\max_{q\in\Delta_B}\; f(\theta,q),
    \qquad
    f(\theta,q)\;\triangleq\;\sum_{b=1}^B q(b)\,L_b(\theta),
    \label{eq:prompt_gdro_game_main}
\end{equation}
where $\Delta_B$ is the probability simplex over $B$ groups.
The inner maximization selects a worst-case mixture of groups, while the outer minimization trains a policy robust to this mixture.

\subsubsection{Entropy-regularized inner maximization and the log-sum-exp surrogate}
\label{sec:theory_lse}

A recurring theme in this paper is that our adversaries are implemented by \emph{exponential-weights} updates (i.e., entropic mirror descent/ascent). 
A key consequence is that the exact max/min over the simplex is replaced by an \emph{entropy-regularized} version, yielding a smooth
``soft'' worst-group objective.

The next lemma is a standard variational identity (often called the Gibbs variational principle / the convex conjugacy between log-sum-exp and negative entropy); see, e.g., \citet{boyd2004convex} or \citet{wainwright2008graphical}. We include a proof for completeness.
\begin{lemma}[Entropy-regularized maximum and log-sum-exp]
\label{lem:lse_variational}
For any $v\in\mathbb{R}^m$ and $\eta>0$,
\begin{equation}
\label{eq:lse_variational}
\max_{p\in\Delta_m} \left\{\langle p, v\rangle + \frac{1}{\eta} H(p)\right\}
\;=\;
\frac{1}{\eta}\log\!\left(\sum_{i=1}^m e^{\eta v_i}\right),
\end{equation}
where $H(p)\triangleq -\sum_{i=1}^m p_i\log p_i$ is the Shannon entropy.
Moreover, the maximizer is unique and equals the softmax distribution
\begin{equation}
\label{eq:softmax_solution}
p^\star_i(v) \;=\; \frac{e^{\eta v_i}}{\sum_{j=1}^m e^{\eta v_j}}
\qquad (i=1,\dots,m).
\end{equation}
\end{lemma}

\begin{proof}
Consider the Lagrangian for maximizing $\sum_i p_i v_i - \frac{1}{\eta}\sum_i p_i\log p_i$ subject to $\sum_i p_i=1$ and $p_i\ge 0$.
At an interior optimum (which holds here because the entropy term makes the objective \emph{strictly concave} and favors $p_i>0$), stationarity gives
\[
v_i - \frac{1}{\eta}(\log p_i + 1) = \lambda
\qquad\text{for all } i,
\] 
Here $\lambda\in\mathbb{R}$ is the Lagrange multiplier for the equality constraint $\sum_i p_i=1$ (so its sign is unconstrained); rearranging yields $p_i \propto e^{\eta v_i}$. Normalizing yields \eqref{eq:softmax_solution}.
Plugging \eqref{eq:softmax_solution} into the objective gives
\[
\sum_i p^\star_i v_i + \frac{1}{\eta}H(p^\star)
=
\frac{1}{\eta}\log\!\left(\sum_{i=1}^m e^{\eta v_i}\right),
\]
which proves \eqref{eq:lse_variational}. Uniqueness follows from strict concavity of $H(\cdot)$ on $\Delta_m$.
\end{proof}

\begin{corollary}[Smooth approximation quality]
\label{cor:lse_max_gap}
For any $v\in\mathbb{R}^m$ and $\eta>0$,
\begin{equation}
\label{eq:lse_gap}
\max_i v_i
\;\le\;
\frac{1}{\eta}\log\!\left(\sum_{i=1}^m e^{\eta v_i}\right)
\;\le\;
\max_i v_i + \frac{\log m}{\eta}.
\end{equation}
\end{corollary}

\begin{proof}
Let $v_{\max}=\max_i v_i$. Then $\sum_i e^{\eta v_i}\ge e^{\eta v_{\max}}$, giving the lower bound.
Also $\sum_i e^{\eta v_i}\le m e^{\eta v_{\max}}$, giving the upper bound.
\end{proof}

\begin{remark}[KL-robust view and ``implicit'' regularization]
\label{rem:kl_robust_view}
Let $u$ be the uniform distribution over $[m]$.
Using $\mathrm{KL}(p\|u)=\sum_i p_i\log p_i+\log m$ and $H(p)=-\sum_i p_i\log p_i$,
\eqref{eq:lse_variational} is equivalently
\begin{equation}
\label{eq:kl_form}
\frac{1}{\eta}\log\!\left(\sum_{i=1}^m e^{\eta v_i}\right)
=
\frac{\log m}{\eta}
+
\max_{p\in\Delta_m}\left\{\langle p,v\rangle - \frac{1}{\eta}\mathrm{KL}(p\|u)\right\}.
\end{equation}
Thus, the softmax distribution in \eqref{eq:softmax_solution} can be interpreted as the optimizer of a \emph{KL-penalized} DRO objective.
In our implementation, this entropy/KL term is not added explicitly to \eqref{eq:prompt_gdro_game_main};
it appears implicitly because the adversary is implemented by entropic mirror ascent (exponential weights).
\Cref{cor:lse_max_gap} quantifies the resulting max-gap: the soft objective approximates the hard max within $\log m/\eta$.
\end{remark}

\paragraph{Entropic GDRO surrogate.}
Applying \Cref{lem:lse_variational} with $m=B$ and $v=L(\theta)\triangleq(L_1(\theta),\dots,L_B(\theta))$ shows that
the entropy-regularized inner problem in \eqref{eq:prompt_gdro_game_main} yields the smooth robust surrogate
\begin{equation}
    \mathcal{R}_\eta(\theta)
    \;\triangleq\;
    \max_{q\in\Delta_B}\left\{\sum_{b=1}^B q(b)L_b(\theta)+\frac{1}{\eta}H(q)\right\}
    \;=\;
    \frac{1}{\eta}\log\!\left(\sum_{b=1}^B e^{\eta L_b(\theta)}\right).
    \label{eq:entropic_risk_prompt}
\end{equation}

\begin{corollary}[Entropic GDRO as a surrogate for worst-group loss]
\label{cor:entropic_surrogate}
For any $\theta$, the entropic surrogate in \eqref{eq:entropic_risk_prompt} satisfies
\begin{equation}
\label{eq:entropic_surrogate_bounds}
\max_{b\in[B]} L_b(\theta)
\;\le\;
\mathcal{R}_\eta(\theta)
\;\le\;
\max_{b\in[B]} L_b(\theta) + \frac{\log B}{\eta}.
\end{equation}
Equivalently, $\mathcal{R}_\eta(\theta)$ is the value of the entropy-regularized variant of the inner maximization in \eqref{eq:prompt_gdro_game_main}.
\end{corollary}

\begin{proof}
Apply \Cref{cor:lse_max_gap} to $v=L(\theta)$ with $m=B$ and note that $\max_i v_i = \max_{b\in[B]} L_b(\theta)$.
\end{proof}

By \Cref{cor:entropic_surrogate}, $\mathcal{R}_\eta(\theta)$ is a differentiable approximation to $\max_b L_b(\theta)$. 

\begin{lemma}[Gradient of the entropic worst-group surrogate]
\label{lem:entropic_grad}
Assume each $L_b(\theta)$ is differentiable.
Define $q_\eta(\cdot;\theta)\in\Delta_B$ by
\begin{equation}
    q_\eta(b;\theta)\;\triangleq\;\frac{\exp(\eta L_b(\theta))}{\sum_{j=1}^B \exp(\eta L_j(\theta))}.
    \label{eq:q_eta_def}
\end{equation}
Then $\nabla_\theta \mathcal{R}_\eta(\theta)=\sum_{b=1}^B q_\eta(b;\theta)\,\nabla_\theta L_b(\theta)$.
\end{lemma}

\begin{proof}
This identity follows directly from \emph{Danskin's theorem} applied to the inner maximization in \eqref{eq:entropic_risk_prompt}.
Indeed, $\Delta_B$ is compact and the entropy regularizer makes the inner problem strictly concave in $q$, hence the maximizer $q_\eta(\cdot;\theta)$ is unique. We include the short closed-form differentiation of \eqref{eq:entropic_risk_prompt} below for completeness:
\[
\nabla_\theta \mathcal{R}_\eta(\theta)
=
\frac{1}{\eta}\cdot
\frac{\sum_{b=1}^B e^{\eta L_b(\theta)}\,\eta\,\nabla_\theta L_b(\theta)}{\sum_{j=1}^B e^{\eta L_j(\theta)}}
=
\sum_{b=1}^B q_\eta(b;\theta)\,\nabla_\theta L_b(\theta).
\]
\end{proof}
\noindent\textbf{Interpretation.}
The distribution $q_\eta(\cdot;\theta)$ in \eqref{eq:q_eta_def} is exactly the entropy-regularized best response of the adversary to $\theta$: it solves 
$\arg\max_{q\in\Delta_B}\{\sum_b q(b)L_b(\theta)+\frac{1}{\eta}H(q)\}$.
Thus $q_\eta(\cdot;\theta)$ is a smooth proxy for the hard worst-group selector in \eqref{eq:prompt_gdro_game_main}:
as $\eta\to\infty$, $q_\eta(\cdot;\theta)$ concentrates on (ties among) $\arg\max_b L_b(\theta)$, whereas for finite $\eta$ it spreads mass across near-worst groups.
This is the same ``soft'' best response tracked online by the exponential-weights update used in Prompt-GDRO.

\subsubsection{No-regret dynamics imply approximate robust optimality}
\label{sec:theory_noregret}

We now connect Prompt-GDRO to standard min--max optimization via no-regret dynamics.
The main message is classical: if the learner (policy) and adversary (group distribution) each run a no-regret algorithm,
then their time averages approach an approximate saddle point of \eqref{eq:prompt_gdro_game_main}.
We provide a self-contained theorem in an idealized convex bounded regime, mainly to make the core maximin/no-regret logic behind Prompt-GDRO explicit.

\begin{assumption}[Convex bounded regime]
\label{ass:convex_regime}
Assume $\Theta\subset\mathbb{R}^d$ is convex and compact with diameter $D$ in $\ell_2$.
Assume each group loss $L_b(\theta)$ is convex in $\theta$, differentiable, and $G$-Lipschitz:
$\|\nabla_\theta L_b(\theta)\|_2\le G$ for all $\theta\in\Theta$.
Assume the losses are bounded: $0\le L_b(\theta)\le M$ for all $b,\theta$.
\end{assumption}

\begin{theorem}[No-regret dynamics yield an approximate GDRO solution]
\label{thm:noregret_gdro}
Consider the zero-sum game \eqref{eq:prompt_gdro_game_main} with payoff
$f(\theta,q)\triangleq \sum_{b=1}^B q(b)L_b(\theta)$. Let \Cref{ass:convex_regime} hold.
Suppose we run $T$ rounds of:
\begin{enumerate}[nosep,leftmargin=*]
\item \textbf{(Learner)} Online gradient descent on $\theta$ with step size $\eta_\theta$:
$\theta_{t+1}=\Pi_{\Theta}(\theta_t-\eta_\theta g_t)$,
where $g_t$ is a (possibly stochastic) subgradient satisfying $\mathbb{E}[g_t\mid \theta_t,q_t]=\nabla_\theta f(\theta_t,q_t)$ and a conditional second-moment bound $\mathbb{E}[\|g_t\|_2^2\mid \theta_t,q_t]\le G_{\mathrm{sg}}^2$.
\item \textbf{(Adversary)} Exponentiated-gradient mirror ascent on $q$ with step size $\eta_q$:
$q_{t+1}(b)\propto q_t(b)\exp(\eta_q \hat{L}_{t,b})$,
where $\hat{L}_{t,b}$ satisfies $\mathbb{E}[\hat{L}_{t,b}\mid \theta_t]=L_b(\theta_t)$ and $0\le \hat{L}_{t,b}\le M$ a.s.
\end{enumerate}
Let $\bar{\theta}\triangleq \frac{1}{T}\sum_{t=1}^T \theta_t$ and $\bar{q}\triangleq \frac{1}{T}\sum_{t=1}^T q_t$.
Then the expected saddle-point gap satisfies
\begin{equation}
\label{eq:saddle_gap_bound}
\mathbb{E}\Big[\max_{q\in\Delta_B} f(\bar{\theta},q) - \min_{\theta\in\Theta} f(\theta,\bar{q})\Big]
\;\le\;
\frac{D^2}{2\eta_\theta T}+\frac{\eta_\theta G_{\mathrm{sg}}^2}{2}
+\frac{\log B}{\eta_q T}+\frac{\eta_q M^2}{8}.
\end{equation}
Consequently, since $\max_{q\in\Delta_B} f(\bar{\theta},q)=\max_b L_b(\bar{\theta})$ and $\min_{\theta}\max_q f(\theta,q)$ is the GDRO optimum,
\begin{equation}
\label{eq:worst_group_subopt}
\mathbb{E}\big[\max_b L_b(\bar{\theta})\big]
\;\le\;
\min_{\theta\in\Theta}\max_b L_b(\theta)
+\frac{D^2}{2\eta_\theta T}+\frac{\eta_\theta G_{\mathrm{sg}}^2}{2}
+\frac{\log B}{\eta_q T}+\frac{\eta_q M^2}{8}.
\end{equation}
\end{theorem}

\begin{proof}
We prove \eqref{eq:saddle_gap_bound} by combining standard regret bounds for OGD (learner) and exponential-weights on the simplex (adversary); see, e.g., \citealp{bubeck2015convex,hazan2016introduction,cesa2006prediction}. We follow these textbook proofs and include the derivation here mainly to track constants.

\textbf{Step 1: learner regret.}
By non-expansiveness of Euclidean projection and the standard OGD one-step inequality,
for any $\theta\in\Theta$,
\begin{align*}
\|\theta_{t+1}-\theta\|_2^2
&\le
\|\theta_t-\eta_\theta g_t-\theta\|_2^2
=
\|\theta_t-\theta\|_2^2
-2\eta_\theta\langle g_t,\theta_t-\theta\rangle
+\eta_\theta^2\|g_t\|_2^2.
\end{align*}
Rearranging and summing over $t=1,\dots,T$ yields
\[
\sum_{t=1}^T \langle g_t,\theta_t-\theta\rangle
\le
\frac{\|\theta_1-\theta\|_2^2}{2\eta_\theta}
+\frac{\eta_\theta}{2}\sum_{t=1}^T\|g_t\|_2^2
\le
\frac{D^2}{2\eta_\theta}
+\frac{\eta_\theta}{2}\sum_{t=1}^T\|g_t\|_2^2.
\]
Taking conditional expectation and using $\mathbb{E}[g_t\mid \theta_t,q_t]=\nabla_\theta f(\theta_t,q_t)$, convexity of $f(\cdot,q_t)$, and the conditional second-moment bound $\mathbb{E}[\|g_t\|_2^2\mid \theta_t,q_t]\le G_{\mathrm{sg}}^2$ gives
\[
\mathbb{E}\!\left[\sum_{t=1}^T f(\theta_t,q_t) - f(\theta,q_t)\right]
\le
\frac{D^2}{2\eta_\theta}
+\frac{\eta_\theta G_{\mathrm{sg}}^2 T}{2}.
\]
Dividing by $T$ gives the learner's average regret bound.

\textbf{Step 2: adversary regret.}
Let $u_t\in\mathbb{R}^B$ denote the (possibly estimated) payoff vector with entries $u_{t,b}\triangleq \hat L_{t,b}$.
Exponentiated-gradient mirror ascent on the simplex with entropy regularizer satisfies the standard bound:
for any $q\in\Delta_B$,
\[
\sum_{t=1}^T \langle q, u_t\rangle - \langle q_t, u_t\rangle
\le
\frac{\log B}{\eta_q} + \frac{\eta_q}{8}\sum_{t=1}^T \|u_t\|_\infty^2,
\]
where the constant $1/8$ follows from Hoeffding's lemma when $u_{t,b}\in[0,M]$.
Since $\|u_t\|_\infty\le M$ a.s., we have
\[
\mathbb{E}\!\left[\sum_{t=1}^T \langle q, \hat L_t\rangle - \langle q_t, \hat L_t\rangle\right]
\le
\frac{\log B}{\eta_q} + \frac{\eta_q M^2 T}{8}.
\]
Using unbiasedness $\mathbb{E}[\hat L_{t,b}\mid \theta_t]=L_b(\theta_t)$ yields the same bound with $L(\theta_t)$ in place of $\hat L_t$.

\textbf{Step 3: combine regrets into a saddle-point gap.}
The learner regret implies
\[
\frac{1}{T}\sum_{t=1}^T f(\theta_t,q_t)
\le
\min_{\theta\in\Theta}\frac{1}{T}\sum_{t=1}^T f(\theta,q_t)
+\frac{D^2}{2\eta_\theta T}+\frac{\eta_\theta G_{\mathrm{sg}}^2}{2}.
\]
The adversary regret implies
\[
\max_{q\in\Delta_B}\frac{1}{T}\sum_{t=1}^T f(\theta_t,q)
\le
\frac{1}{T}\sum_{t=1}^T f(\theta_t,q_t)
+\frac{\log B}{\eta_q T}+\frac{\eta_q M^2}{8}.
\]
Combining and using Jensen's inequality,
\[
\max_{q\in\Delta_B} f(\bar\theta,q)
\le
\max_{q\in\Delta_B}\frac{1}{T}\sum_{t=1}^T f(\theta_t,q)
\le
\min_{\theta\in\Theta}\frac{1}{T}\sum_{t=1}^T f(\theta,q_t)
+\frac{D^2}{2\eta_\theta T}+\frac{\eta_\theta G_{\mathrm{sg}}^2}{2}
+\frac{\log B}{\eta_q T}+\frac{\eta_q M^2}{8}.
\]
Finally, convexity of $f(\theta,\cdot)$ in $q$ implies
$\min_{\theta\in\Theta} f(\theta,\bar q)\le \min_{\theta\in\Theta}\frac{1}{T}\sum_{t=1}^T f(\theta,q_t)$.
Rearranging yields \eqref{eq:saddle_gap_bound}.
The suboptimality bound \eqref{eq:worst_group_subopt} follows since $\max_{q\in\Delta_B} f(\bar\theta,q)=\max_b L_b(\bar\theta)$.
\end{proof}

\begin{remark}[Reading \Cref{thm:noregret_gdro} in the deep RL regime]
\label{rem:sanity_theorem}
Assumption~\ref{ass:convex_regime} is not satisfied by neural policies trained with GRPO, so \Cref{thm:noregret_gdro} should be read as an \emph{idealized online-learning lens} rather than a literal convergence guarantee.
It formalizes the conceptual statement that if (i) the policy updates behave like a low-regret learner and (ii) the group reweighting behaves like a low-regret adversary (implemented via exponential weights), then the time averages approximate a GDRO saddle point.

Empirically, several qualitative predictions of this lens appear in our Prompt-GDRO dynamics: the adversary maintains a non-degenerate, entropy-regularized weight distribution over bins (see the entropy traces in \Cref{fig:dynamics_metrics}), and the weights form a ``traveling wave frontier that leads the empirical prompt-share distribution early in training (visible in the triptych of \Cref{fig:mechanism} and summarized by the lead--lag proxy in \Cref{fig:lead_lag}).
In practice, Prompt-GDRO uses bandit feedback (we only observe losses for sampled prompts/bins), hence our implementation uses the GDRO-EXP3P estimator/updates of \citet{soma2022optimal}, which preserve no-regret guarantees under partial information.
\end{remark}

\begin{remark}[Where Rollout-GDRO enters the saddle-gap bound]
\label{rem:rollout_enters_bound}
The learner term $\tfrac{\eta_\theta}{2}\sum_{t=1}^T \|g_t\|_2^2$ in \Cref{thm:noregret_gdro} makes explicit that
\emph{estimation noise} impacts the constants in our no-regret analysis.
In GRPO, $g_t$ is formed from Monte Carlo rollouts, and is therefore stochastic even when conditioning on $(\theta_t,q_t)$.

To separate optimization from estimation effects, write
$g_t=\nabla_\theta f(\theta_t,q_t)+\xi_t$ with $\mathbb{E}[\xi_t\mid \theta_t,q_t]=0$.
Then
\[
\mathbb{E}\|g_t\|_2^2
=
\|\nabla_\theta f(\theta_t,q_t)\|_2^2+\mathbb{E}\|\xi_t\|_2^2.
\]
Now consider a batch of $M$ prompts whose (empirical) bin fractions are $\bar q_{t,1:B}$ and whose per-bin rollout allocation is $\mathbf n_t=(n_{t,1},\dots,n_{t,B})$.
Under \Cref{lem:batch_var}, the stochastic component of the batch gradient obeys the proxy bound
\[
\mathbb{E}\|\xi_t\|_2^2
\;\le\;
\frac{1}{M}\sum_{b=1}^B \bar q_{t,b}\,\frac{v_b(\theta_t)}{n_{t,b}},
\]
where $v_b(\theta_t)$ is the intrinsic per-bin gradient-variance term.
Rollout-GDRO is designed to \emph{adaptively choose} $\mathbf n_t$ to reduce this variance proxy under a mean compute constraint.
In \Cref{sec:theory_rollout_noregret}, we show that the rollout controller can itself be viewed as a no-regret primal--dual algorithm
for a budgeted variance objective, making precise (in an idealized model) how variance-aware compute allocation tightens the
$\sum_t \|g_t\|_2^2$ term relative to uniform rollouts.
\end{remark}

\subsection{Rollout-GDRO as Variance-Aware Allocation Under a Budget and No-Regret Game Dynamics}
\label{sec:theory_rollout_gdro}

We now analyze the second controller: allocating the number of rollouts per prompt as a function of group/bin.
The key idea is classical: if some bins induce higher intrinsic variance (more stochastic rewards / longer reasoning / more fragile completions),
then allocating \emph{more} rollouts to these bins reduces gradient variance and improves stability.

\subsubsection{A variance proxy for GRPO rollouts}
\label{sec:theory_variance_proxy}

\paragraph{Reference equations from the main paper.}
In the main paper (\Cref{sec:analysis}), Rollout-GDRO is posed as a compute-neutral allocation problem driven by empirical bin utilities.
At a training step $t$, the allocator observes the realized bin fractions $\hat q_t\in\Delta_B$ in the current prompt batch and chooses
discrete rollout counts $n_b\in\{n_{\min},\dots,n_{\max}\}$ for each bin $b$ under the mean rollout budget constraint in \eqref{eq:analysis_budget_lagrangian} (``compute neutrality'').
The bin-level quantity $\hat J_b(\theta;n_b)$ in \eqref{eq:analysis_budget_obj} is itself estimated from $n_b$ rollouts per prompt,
so changing $n_b$ affects both the quality (variance) and possibly the bias of the signal provided to the learner.

\paragraph{What is guaranteed by the rollout controller vs. what is explained by the variance proxy.}
Equation~\eqref{eq:analysis_budget_lagrangian} is the \emph{implemented} allocation problem: at each step, the controller selects discrete arms $\{n_b\}$ to maximize an empirical utility under a strict mean-rollout constraint.
This can be cast as an online constrained bandit problem by defining an (arm) loss
\(V_b(n;\theta)\triangleq -\hat J_b(\theta;n)\) (or any bounded monotone transform thereof), and augmenting it with the linear budget penalty $\mu n$ in \eqref{eq:rollout_arm_loss}.
Our no-regret result in \Cref{sec:theory_rollout_noregret} applies to this generic primal--dual EXP3P controller for \eqref{eq:analysis_budget_lagrangian}.
The variance-proxy analysis below is a \emph{separate} guarantee: it upper-bounds the conditional variance of GRPO's Monte Carlo gradient estimator as a function of the allocation $\mathbf n$, yielding the proxy objective $\sum_b \bar q_b\,v_b(\theta)/n_b$ (Eq.~\eqref{eq:analysis_var_opt_problem}).
In particular, if the controller is instantiated with feedback that estimates (minus) this proxy cost, then the same no-regret machinery yields a provably near-optimal variance-minimizing allocation.

Our theoretical results in this subsection focus on the stabilizing effect of increasing $n_b$ through variance reduction.
Formally, we derive a simple proxy for the stochastic component of the GRPO batch gradient as a function of the rollout allocation
$\mathbf n=(n_1,\dots,n_B)$, which yields an optimization problem of the form
$\min_{\mathbf n}\sum_b \bar q_b\,v_b(\theta)/n_b$ under the same mean budget.
Informally, the bin-dependent quantity $v_b(\theta)$ measures how noisy the per-rollout GRPO gradient is within bin $b$: larger $v_b(\theta)$ means completions in that bin yield higher-variance gradients, so averaging more rollouts (larger $n_b$) is more valuable there. For the cleanest bound, we define $v_b(\theta)$ as a uniform (worst-case) upper bound over prompts in bin $b$; see \Cref{lem:batch_var}.
This makes explicit how Rollout-GDRO can be interpreted as variance-aware compute redistribution in the sense visualized by our
``weighted standard error'' diagnostics.
For clarity, we fix a training step and write $\bar q_b\triangleq \hat q_t(b)$ for the realized bin fractions. 
That is, for the step-$t$ prompt batch $\{x_i\}_{i=1}^M$, \(\hat q_t(b) \triangleq \frac{1}{M}\sum_{i=1}^M \mathbf{1}\{g(x_i)=b\}\), so $\bar q_b$ is the empirical fraction of bin-$b$ prompts in the batch.

Fix a prompt $x$ and policy parameters $\theta$.
Let $\{y_j\}_{j=1}^{n_b}$ denote $n_b$ rollouts sampled from $\pi_\theta(\cdot\mid x)$.
In GRPO, these rollouts are used to form a \emph{prompt-level empirical loss} (cf.\ \Cref{sec:preli_grpo}), which may involve within-prompt normalization across the rollout group (e.g., the standardized advantages in \eqref{eq:grpo_advantage}).
We write this generic prompt-level estimator as $\hat\ell(x;\theta,n_b)$ and define the corresponding per-prompt gradient estimator
\begin{equation}
    \hat g(x;\theta,n_b)
    \;\triangleq\;
    \nabla_\theta\,\hat\ell(x;\theta,n_b),
    \qquad g(x)=b.
    \label{eq:per_prompt_grad_est}
\end{equation}

The analysis below only requires that the rollouts are i.i.d. and that $\hat g$ is a (possibly coupled) function of these rollouts.

\begin{assumption}[Rollout sampling model and differentiation under the expectation]
\label{ass:rollout_sampling_model}
Conditional on a prompt $x$ and parameters $\theta$, the rollouts $y_1,\dots,y_{n_b}$ are drawn i.i.d. from $\pi_\theta(\cdot\mid x)$.
Moreover, the prompt-level estimator $\hat\ell(x;\theta,n_b)$ is differentiable in $\theta$ and we may interchange gradient and conditional expectation:
\(\nabla_\theta\,\mathbb{E}[\hat\ell(x;\theta,n_b)\mid x]=\mathbb{E}[\nabla_\theta \hat\ell(x;\theta,n_b)\mid x]\).
\end{assumption}

\begin{lemma}[Connecting $\hat J_b(\theta;n_b)$ to per-prompt gradient estimators]
\label{lem:jb_to_grad}
Recall that $\hat\ell(x;\theta,n_b)$ denotes the prompt-level empirical loss estimator computed from $n_b$ rollouts, and $\hat g(x;\theta,n_b)=\nabla_\theta\hat\ell(x;\theta,n_b)$.
If a bin-level estimator $\hat J_b(\theta;n_b)$ (as in \eqref{eq:analysis_budget_obj}) is formed by averaging $\hat\ell(x;\theta,n_b)$ over the prompts $x$ in bin $b$ in the current batch (up to the sign convention $L_b=-J_b$), then $\nabla_\theta \hat J_b(\theta;n_b)$ is the corresponding average of the per-prompt estimators $\hat g(x;\theta,n_b)$.
\end{lemma}

\begin{proof}
Both claims follow immediately from linearity of averaging and differentiation.
\end{proof}

\begin{assumption}[Bounded differences for the prompt-gradient estimator]
\label{ass:rollout_bounded_diff}
For any prompt $x$ with $g(x)=b$ and any rollout count $n_b$, the estimator $\hat g(x;\theta,n_b)$ is a measurable function of the rollout group $(y_1,\dots,y_{n_b})$.
Assume there exists a (bin-dependent) constant $C_b(\theta)\ge 0$ such that, for every coordinate $j\in\{1,\dots,n_b\}$ and any replacement rollout $y_j'$, the estimator satisfies the bounded-differences property
\begin{equation}
\label{eq:rollout_bd}
\big\|\hat g(x;\theta,n_b; y_1,\dots,y_j,\dots,y_{n_b})-\hat g(x;\theta,n_b; y_1,\dots,y_j',\dots,y_{n_b})\big\|_2
\;\le\;
\frac{C_b(\theta)}{n_b}.
\end{equation}
\end{assumption}

\begin{lemma}[GRPO prompt-gradient satisfies bounded differences under within-group normalization]
\label{lem:grpo_bounded_diff}
Consider a fixed prompt $x$ in bin $b$ and $n_b$ i.i.d. rollouts $y_1,\dots,y_{n_b}\sim \pi_\theta(\cdot\mid x)$ with bounded rewards $r_j\triangleq r(x,y_j)\in[0,R]$.
Let
\(\bar r\triangleq \frac{1}{n_b}\sum_{j=1}^{n_b} r_j\),
\(s\triangleq \sqrt{\frac{1}{n_b}\sum_{j=1}^{n_b} (r_j-\bar r)^2}\),
and define standardized advantages $A_j\triangleq (r_j-\bar r)/(s+\varepsilon)$ for some fixed $\varepsilon>0$.
Suppose the (per-rollout) score function is uniformly bounded, i.e.,
\begin{equation}
\label{eq:score_bound}
\big\|\nabla_\theta \log \pi_\theta(y\mid x)\big\|_2\le G_\pi\qquad \text{for all }(x,y,\theta).
\end{equation}
Consider the normalized prompt-gradient estimator
\begin{equation}
\label{eq:grpo_grad_simplified}
\hat g(x;\theta,n_b)\;\triangleq\;\frac{1}{n_b}\sum_{j=1}^{n_b} A_j\,\nabla_\theta \log \pi_\theta(y_j\mid x).
\end{equation}
Then Assumption~\ref{ass:rollout_bounded_diff} holds with a constant of the form
\begin{equation}
\label{eq:grpo_bd_const}
C_b(\theta)\;\le\; G_\pi\Big(\frac{3R^2}{\varepsilon^2}+\frac{5R}{\varepsilon}\Big).
\end{equation}
The same conclusion holds (up to replacing $G_\pi$ by an appropriate bound on the per-rollout GRPO score term) when additional bounded multiplicative factors are present, such as PPO ratio clipping.
\end{lemma}

\begin{proof}
Fix an index $k\in\{1,\dots,n_b\}$ and let $r=(r_1,\dots,r_{n_b})$ and $r'$ denote the reward vectors that differ only at coordinate $k$ (corresponding to replacing $y_k$ by some alternative rollout $y_k'$).
Let $(\bar r,s)$ and $(\bar r',s')$ denote the corresponding means and standard deviations, and define $A_j$ and $A_j'$ from $r$ and $r'$.

First, the mean changes by at most
\begin{equation}
\label{eq:mean_lip}
|\bar r-\bar r'|\le \frac{|r_k-r_k'|}{n_b}\le \frac{R}{n_b}.
\end{equation}
Next, writing the (biased) sample variance as $v\triangleq s^2=\frac{1}{n_b}\sum_{j} r_j^2-\bar r^2$ (and similarly $v'\triangleq {s'}^2$; biased vs.\ unbiased only changes constants), we have
\begin{align}
|v-v'|
&\le \frac{1}{n_b}|r_k^2-{r_k'}^2|+|\bar r^2-\bar r'^2|
\le \frac{R^2}{n_b}+|\bar r-\bar r'|\,|\bar r+\bar r'|\nonumber\\
&\le \frac{R^2}{n_b}+\frac{R}{n_b}\cdot 2R
\le \frac{3R^2}{n_b}.
\label{eq:var_lip}
\end{align}
By the identity $|\sqrt{v}-\sqrt{v'}|=|v-v'|/(s+s')$ (interpreting the right-hand side as $0$ when $s=s'=0$),
\begin{equation}
\label{eq:std_lip}
|s-s'|\le \frac{|v-v'|}{s+s'}\le \frac{3R^2}{n_b\,(s+s')}.
\end{equation}

Now decompose the change in the estimator \eqref{eq:grpo_grad_simplified} as
\begin{equation}
\label{eq:grpo_bd_decomp}
\big\|\hat g-\hat g'\big\|_2
\le \frac{1}{n_b}\sum_{j=1}^{n_b} \big\| (A_j-A_j')\nabla_\theta\log\pi_\theta(y_j\mid x)\big\|_2
\; +\;\frac{1}{n_b}\big\|A_k'\big(\nabla_\theta\log\pi_\theta(y_k\mid x)-\nabla_\theta\log\pi_\theta(y_k'\mid x)\big)\big\|_2.
\end{equation}
Using \eqref{eq:score_bound} and $|A_k'|\le \frac{|r_k'-\bar r'|}{\varepsilon}\le \frac{R}{\varepsilon}$, the second term in \eqref{eq:grpo_bd_decomp} is at most $\frac{2RG_\pi}{\varepsilon n_b}$.

For the first term, we bound the total change in standardized advantages.
For any $j\ne k$ (so $r_j=r_j'$), we can write
\begin{align}
|A_j-A_j'|
&=\left|\frac{r_j-\bar r}{s+\varepsilon}-\frac{r_j-\bar r'}{s'+\varepsilon}\right|\nonumber\\
&\le \frac{|\bar r-\bar r'|}{s+\varepsilon}
\; +\; |r_j-\bar r|\,\left|\frac{1}{s+\varepsilon}-\frac{1}{s'+\varepsilon}\right|.
\label{eq:Aj_diff_bound}
\end{align}
Summing \eqref{eq:Aj_diff_bound} over $j\ne k$ and using \eqref{eq:mean_lip} gives
\begin{equation}
\label{eq:mean_term_sum}
\sum_{j\ne k}\frac{|\bar r-\bar r'|}{s+\varepsilon}
\le (n_b-1)\cdot \frac{R/n_b}{\varepsilon}
\le \frac{R}{\varepsilon}.
\end{equation}
For the denominator term, note that
\(|\frac{1}{s+\varepsilon}-\frac{1}{s'+\varepsilon}| = \frac{|s-s'|}{(s+\varepsilon)(s'+\varepsilon)}\le \frac{|s-s'|}{\varepsilon^2}\).
Also, by Cauchy--Schwarz,
\(\sum_{j=1}^{n_b} |r_j-\bar r|\le \sqrt{n_b\sum_j (r_j-\bar r)^2}=n_b s\).
Combining with \eqref{eq:std_lip} and \eqref{eq:var_lip} yields
\begin{equation}
\label{eq:denom_term_sum}
\sum_{j\ne k} |r_j-\bar r|\left|\frac{1}{s+\varepsilon}-\frac{1}{s'+\varepsilon}\right|
\le \frac{n_b s}{\varepsilon^2}\cdot \frac{3R^2}{n_b(s+s')}
\le \frac{3R^2}{\varepsilon^2}.
\end{equation}
Finally, using the crude bound $|A_k-A_k'|\le |A_k|+|A_k'|\le 2R/\varepsilon$ for the changed coordinate, we obtain
\begin{equation}
\label{eq:sum_A_diff}
\sum_{j=1}^{n_b} |A_j-A_j'|\le \frac{3R^2}{\varepsilon^2}+\frac{3R}{\varepsilon}.
\end{equation}
Plugging \eqref{eq:sum_A_diff} into the first term of \eqref{eq:grpo_bd_decomp} and combining with the second-term bound gives
\(\|\hat g-\hat g'\|_2\le \frac{G_\pi}{n_b}(\frac{3R^2}{\varepsilon^2}+\frac{5R}{\varepsilon})\), which is the stated bounded-differences form.
\end{proof}

\begin{lemma}[A $1/n$ variance bound (covers within-prompt normalization in GRPO)]
\label{lem:rollout_avg}
Assume \Cref{ass:rollout_sampling_model,ass:rollout_bounded_diff} hold.
Then for any prompt $x$ with $g(x)=b$,
\begin{equation}
\label{eq:1_over_n}
\mathbb{E}\!\left[\left\|\hat g(x;\theta,n_b)-\mathbb{E}[\hat g(x;\theta,n_b)\mid x]\right\|_2^2 \,\Big|\, x\right]
\;\le\;
\frac{C_b(\theta)^2}{2n_b}.
\end{equation}
In particular, defining $v_b(\theta)\triangleq C_b(\theta)^2/2$ yields the convenient form
$\mathbb{E}[\|\hat g(x;\theta,n_b)-\mathbb{E}[\hat g(x;\theta,n_b)\mid x]\|_2^2\mid x]\le v_b(\theta)/n_b$.
\end{lemma}

\begin{proof}
Fix $x$ and write $Y=(y_1,\dots,y_{n_b})$ for the rollout group.
Let $Y^{(j)}$ denote the vector obtained by replacing only the $j$-th coordinate $y_j$ with an independent copy $y_j'\sim\pi_\theta(\cdot\mid x)$.
\emph{Fact (vector-valued Efron--Stein).} The Efron--Stein inequality extends to $\mathbb{R}^d$-valued estimators with $\|\cdot\|_2$ (more generally, Hilbert space-valued) by applying the scalar inequality coordinatewise and summing; see, e.g., \citet[Ch.~3]{boucheron2013concentration}.
The Efron--Stein inequality (vector-valued Efron--Stein / Hilbert space version) implies
\[
\mathbb{E}\!\left[\left\|\hat g(Y)-\mathbb{E}[\hat g(Y)\mid x]\right\|_2^2\,\Big|\,x\right]
\;\le\;
\frac12\sum_{j=1}^{n_b}\mathbb{E}\!\left[\left\|\hat g(Y)-\hat g\!\left(Y^{(j)}\right)\right\|_2^2\,\Big|\,x\right].
\]
By \Cref{eq:rollout_bd}, each summand is at most $C_b(\theta)^2/n_b^2$.
Summing over $j$ gives \eqref{eq:1_over_n}.
\end{proof}

Next, consider a batch of $M$ prompts $\{x_i\}_{i=1}^M$.
Let $\mathbf n=(n_1,\dots,n_B)$ denote the vector of per-bin rollout counts, and define the batch gradient estimator
\begin{equation}
    \hat g(\theta;\mathbf n)
    \;\triangleq\;
    \frac{1}{M}\sum_{i=1}^M \hat g(x_i;\theta,n_{g(x_i)}).
    \label{eq:batch_grad_est}
\end{equation}

\begin{lemma}[A variance proxy decomposes over groups]
\label{lem:batch_var}
Assume that, conditioned on the batch prompts $\{x_i\}_{i=1}^M$, the rollout groups used to form the per-prompt estimators
$\hat g(x_i;\theta,n_{g(x_i)})$ are independent across $i$.
Fix a batch of $M$ prompts $\{x_i\}_{i=1}^M$ and define the empirical bin fractions
$\bar q_b \triangleq \frac{1}{M}\sum_{i=1}^M \mathbf{1}\{g(x_i)=b\}$.
Then, conditioned on the batch prompts $\{x_i\}_{i=1}^M$,
\begin{equation}
    \mathbb{E}\!\left[\big\|\hat g(\theta;\mathbf n) - \mathbb{E}[\hat g(\theta;\mathbf n)\mid \{x_i\}_{i=1}^M]\big\|_2^2 \,\Big|\, \{x_i\}_{i=1}^M\right]
    \;\le\;
    \frac{1}{M}\sum_{b=1}^B \bar q_b\,\frac{v_b(\theta)}{n_b},
    \label{eq:batch_var_bound}
\end{equation}
where $v_b(\theta)$ is any uniform bin-wise constant such that the per-prompt conditional variance obeys
$\mathbb{E}[\|\hat g(x;\theta,n_b)-\mathbb{E}[\hat g(x;\theta,n_b)\mid x]\|_2^2\mid x]\le v_b(\theta)/n_b$ for all prompts $x$ with $g(x)=b$.
Under \Cref{ass:rollout_bounded_diff} and \Cref{lem:rollout_avg}, we may take $v_b(\theta)=C_b(\theta)^2/2$.
\end{lemma}

\begin{proof}
Write $\hat g(\theta;\mathbf n)-\mathbb{E}[\hat g(\theta;\mathbf n)\mid \{x_i\}_{i=1}^M] = \frac{1}{M}\sum_{i=1}^M Z_i$,
where $Z_i \triangleq \hat g(x_i;\theta,n_{g(x_i)})-\mathbb{E}[\hat g(x_i;\theta,n_{g(x_i)})\mid x_i]$.
Conditioned on the prompts, $\{Z_i\}_{i=1}^M$ are independent and zero-mean.
Thus,
\[
\mathbb{E}\!\left[\left\|\frac{1}{M}\sum_{i=1}^M Z_i\right\|_2^2 \,\Big|\, \{x_i\}_{i=1}^M\right]
=
\frac{1}{M^2}\sum_{i=1}^M \mathbb{E}\!\left[\|Z_i\|_2^2\mid x_i\right]
\le
\frac{1}{M^2}\sum_{i=1}^M \frac{v_{g(x_i)}(\theta)}{n_{g(x_i)}},
\]
where the last inequality uses \Cref{lem:rollout_avg} and the definition of $v_b(\theta)$.
Grouping terms by bins yields
$\frac{1}{M^2}\sum_{b=1}^B \sum_{i:g(x_i)=b} \frac{v_b(\theta)}{n_b}
= \frac{1}{M}\sum_{b=1}^B \bar q_b \frac{v_b(\theta)}{n_b}$,
which proves \eqref{eq:batch_var_bound}.
\end{proof}

\begin{remark}[From variance proxy to the plotted ``weighted standard error'']
\label{rem:se_proxy}
The quantity on the right-hand side of \eqref{eq:batch_var_bound} is a \emph{variance proxy}.
In experiments we additionally visualize the more interpretable ``weighted standard error'' proxy
$\sum_b \bar q_b \sqrt{v_b(\theta)/n_b}$.
This proxy is often easier to interpret because each term $\sqrt{v_b(\theta)/n_b}$ behaves like a standard error: it is a within-bin scale of the prompt-gradient noise, which shrinks at the canonical $1/\sqrt{n_b}$ rate when using $n_b$ rollouts (\Cref{lem:rollout_avg}), and the weights $\bar q_b$ simply average these standard-error contributions across the observed batch composition (see \Cref{fig:wse_proxy}).
By Cauchy--Schwarz,
\[
\left(\sum_{b=1}^B \bar q_b \sqrt{\frac{v_b(\theta)}{n_b}}\right)^2
\le
\left(\sum_{b=1}^B \bar q_b\right)\left(\sum_{b=1}^B \bar q_b\frac{v_b(\theta)}{n_b}\right)
=
\sum_{b=1}^B \bar q_b\frac{v_b(\theta)}{n_b}.
\]
Thus, minimizing the variance proxy also controls the weighted standard-error proxy we plot.
\end{remark}

\subsubsection{The variance-optimal allocation obeys a square-root law}
\label{sec:theory_sqrt_law}

\Cref{lem:batch_var} suggests that, under compute neutrality, the rollout allocation $\mathbf n$ controls a leading-order proxy for
the stochastic component of the learner's update through the term $\sum_b \bar q_b\,v_b(\theta)/n_b$.
This motivates studying a variance-aware relaxation of the allocator objective \eqref{eq:analysis_budget_obj}, in which the
\emph{marginal value} of additional rollouts is captured by the reduction in estimator variance.
Concretely, we fix $\theta$ and treat $v_b(\theta)$ and $\bar q_b$ as constants for a given step (with $\bar q_b=\hat q_t(b)$),
and we analyze the continuous relaxation
\begin{equation}
\label{eq:var_opt_problem}
\min_{\mathbf n\in\mathbb{R}_+^B}\;\sum_{b=1}^B \bar q_b\,\frac{v_b(\theta)}{n_b}
\qquad\text{s.t.}\qquad
\sum_{b=1}^B \bar q_b\,n_b = \bar n.
\end{equation}
This objective is convex in $\mathbf n$ (each term is $1/n_b$), and the constraint is linear.

\begin{theorem}[Square-root law for variance-optimal allocation]
\label{thm:sqrt_allocation}
Let $A\triangleq\{b\in[B]: \bar q_b>0\}$ denote the set of bins present in the batch.
If $v_b(\theta)>0$ for all $b\in A$, then the minimizer of \eqref{eq:var_opt_problem} is unique on the active coordinates and is
\begin{equation}
    n_b^\star
    = \bar n\cdot\frac{\sqrt{v_b(\theta)}}{\sum_{j=1}^B \bar q_j\,\sqrt{v_j(\theta)}},
\qquad b\in A.
    \label{eq:nb_star}
\end{equation}
For bins with $\bar q_b=0$ (i.e., $b\notin A$), the variable $n_b$ does not affect the objective nor the budget constraint and may be chosen arbitrarily.

Moreover, the minimal objective value equals
\begin{equation}
    \sum_{b=1}^B \bar q_b\,\frac{v_b(\theta)}{n_b^\star}
    = \frac{\left(\sum_{b=1}^B \bar q_b\,\sqrt{v_b(\theta)}\right)^2}{\bar n}.
    \label{eq:var_star_value}
\end{equation}
Compared to the uniform allocation $n_b\equiv \bar n$, the optimal allocation never increases the variance proxy:
\begin{equation}
    \sum_{b=1}^B \bar q_b\,\frac{v_b(\theta)}{n_b^\star}
    \;\le\;
    \sum_{b=1}^B \bar q_b\,\frac{v_b(\theta)}{\bar n}
    =
    \frac{\sum_{b=1}^B \bar q_b v_b(\theta)}{\bar n},
    \label{eq:optimal_le_uniform}
\end{equation}
with equality iff $v_b(\theta)$ is constant across the active bins $b\in A$.
\end{theorem}

\begin{proof}
If $v_b(\theta)=0$ for some active bin $b\in A$, then its term $\bar q_b v_b(\theta)/n_b$ is identically zero, so allocating additional rollouts to that bin cannot reduce the variance proxy; in that degenerate case one may treat such bins as ``free'' and apply the same square-root allocation to the remaining bins with positive $v_b(\theta)$ (in our discrete implementation, this corresponds to choosing the smallest rollout arm).
Form the Lagrangian (with multiplier $\mu$, matching the ``shadow price'' interpretation in \eqref{eq:analysis_budget_lagrangian})
\[
\mathcal{L}(\mathbf n,\mu)
=
\sum_{b=1}^B \bar q_b\frac{v_b(\theta)}{n_b}
+\mu\left(\sum_{b=1}^B \bar q_b n_b-\bar n\right).
\]
Bins with $\bar q_b=0$ do not appear in the objective nor the constraint in \eqref{eq:var_opt_problem}, so we may ignore them and restrict attention to the active set $A$.
Stationarity gives, for each active bin $b\in A$ (so $\bar q_b>0$),
\[
-\bar q_b\frac{v_b(\theta)}{n_b^2}+\mu \bar q_b = 0
\quad\Longrightarrow\quad
n_b = \sqrt{\frac{v_b(\theta)}{\mu}}.
\]
Enforcing the constraint yields
$\sum_b \bar q_b \sqrt{v_b(\theta)/\mu}=\bar n$, hence
$\sqrt{\mu}=\frac{\sum_j \bar q_j\sqrt{v_j(\theta)}}{\bar n}$,
and substituting gives \eqref{eq:nb_star}.
Uniqueness on the active coordinates follows because the objective in \eqref{eq:var_opt_problem} is strictly convex in $\{n_b\}_{b\in A}$ when $v_b(\theta)>0$.

Plugging \eqref{eq:nb_star} into the objective yields \eqref{eq:var_star_value}.

Finally, to show \eqref{eq:optimal_le_uniform}, apply Cauchy--Schwarz:
\[
\left(\sum_{b=1}^B \bar q_b\sqrt{v_b(\theta)}\right)^2
\le
\left(\sum_{b=1}^B \bar q_b\right)\left(\sum_{b=1}^B \bar q_b v_b(\theta)\right)
=
\sum_{b=1}^B \bar q_b v_b(\theta).
\]
Divide by $\bar n$ and use \eqref{eq:var_star_value} to obtain \eqref{eq:optimal_le_uniform}.
Equality holds iff $\sqrt{v_b(\theta)}$ is constant across $b\in A$.
\end{proof}

\begin{remark}[Interpretation]
\label{rem:sqrt_interp}
\Cref{thm:sqrt_allocation} shows that the variance-optimal allocation equalizes the \emph{marginal} variance proxy across bins:
at the optimum, $v_b(\theta)/n_b^{\star 2}=\mu$ is constant in $b$ (the KKT multiplier / shadow price).
Consequently, bins with larger intrinsic variance $v_b(\theta)$ receive more rollouts, with
\emph{per-prompt} rollouts scaling as $n_b^\star\propto \sqrt{v_b(\theta)}$.
This square-root law is the same structure as the classical Neyman allocation in stratified sampling and Monte Carlo budgeting
(see, e.g., \citealp{rubinstein2016simulation}).
Notably, the bin fraction $\bar q_b$ cancels from the stationarity condition, meaning that the \emph{marginal} value of extra rollouts
depends on $v_b(\theta)$ rather than frequency.
At the same time, the budget is weighted by $\bar q_b$, so rare bins can receive large $n_b^\star$ without violating compute neutrality
$\sum_b \bar q_b n_b=\bar n$.
See the rollout heatmaps and time snapshots in \Cref{fig:rollout_mechanism,fig:rollout_snapshots} for the corresponding piecewise-constant ("staircase") transitions between discrete rollout arms as the shadow price $\mu$ adapts.
This square-root allocation intuition is consistent with the empirically observed rollout-allocation patterns in Rollout-GDRO (see, e.g., \Cref{fig:rollout_mechanism,fig:rollout_snapshots}), where bins deemed noisier receive larger rollout arms under a fixed mean budget.
\end{remark}

\subsubsection{Discrete rollout arms and an entropic primal--dual view}
\label{sec:theory_primal_dual}

The continuous analysis in \Cref{sec:theory_variance_proxy}--\Cref{sec:theory_sqrt_law} yields a closed-form allocation for the idealized
variance proxy objective when the rollout counts $\{n_b\}$ are allowed to be any positive reals.
Rollout-GDRO, however, must choose \emph{discrete} rollout counts $n_b\in\mathcal{N}=\{n_{\min},\dots,n_{\max}\}$ and enforce compute neutrality online.
Mirroring the Lagrangian perspective in the main paper (cf.\ \eqref{eq:analysis_budget_lagrangian}), we view the allocator as maintaining a
dual variable $\mu$ that acts as a \emph{shadow price of compute}.
To connect the discrete implementation to the variance proxy derived above, we consider an entropy-regularized relaxation of the same constrained problem.

Concretely, for each bin $b$ we maintain a distribution $p_b\in\Delta_{\mathcal{N}}$ over rollout ``arms'' $n\in\mathcal{N}$.
We model the per-bin \emph{cost} of choosing arm $n$ as a function $V_b(n;\theta)$ that decreases with $n$.
In the variance-proxy setting of \Cref{sec:theory_variance_proxy}, a canonical choice is $V_b(n;\theta)=v_b(\theta)/n$.
(Equivalently, one may interpret the utility in \eqref{eq:analysis_budget_obj} as a monotone transform of $-V_b$; our analysis isolates the
variance-reduction component that is most directly controlled by $n$.)
The entropy-regularized Lagrangian for a fixed $\theta$ takes the form
\begin{equation}
\label{eq:rollout_lagrangian}
\mathcal{L}(\{p_b\},\mu;\theta)
=
\sum_{b=1}^B \bar q_b\left(
\mathbb{E}_{n\sim p_b}[V_b(n;\theta)]
-\frac{1}{\eta}H(p_b)
\right)
+\mu\left(\sum_{b=1}^B \bar q_b\,\mathbb{E}_{n\sim p_b}[n]-\bar n\right),
\end{equation}
where $\mu\ge 0$ is the shadow price and the entropy term encourages exploration over arms (stability).

\begin{lemma}[Soft-min solution for rollout arms]
\label{lem:softmin_rollout}
For fixed $\mu$ and bin $b$, the minimizer of \eqref{eq:rollout_lagrangian} over $p_b\in\Delta_{\mathcal{N}}$ is
\begin{equation}
\label{eq:rollout_softmin_dist}
p^\star_{b,\mu}(n)\;=\;\frac{\exp\!\big(-\eta\,[V_b(n;\theta)+\mu n]\big)}
{\sum_{n'\in\mathcal{N}}\exp\!\big(-\eta\,[V_b(n';\theta)+\mu n']\big)}.
\end{equation}
Moreover, the optimal value of the inner minimization equals a soft-min:
\begin{equation}
\label{eq:softmin_value}
\min_{p_b\in\Delta_{\mathcal{N}}}\left\{
\mathbb{E}_{n\sim p_b}[V_b(n;\theta)+\mu n] - \frac{1}{\eta}H(p_b)
\right\}
=
-\frac{1}{\eta}\log\!\left(\sum_{n\in\mathcal{N}} e^{-\eta(V_b(n;\theta)+\mu n)}\right).
\end{equation}
\end{lemma}

\begin{proof}
Apply \Cref{lem:lse_variational} to the vector $v\in\mathbb{R}^{|\mathcal{N}|}$ with entries
$v_n \triangleq -\big(V_b(n;\theta)+\mu n\big)$.
Then
\[
\max_{p\in\Delta_{\mathcal{N}}}\left\{\langle p,v\rangle + \frac{1}{\eta}H(p)\right\}
=
\frac{1}{\eta}\log\sum_{n\in\mathcal{N}} e^{\eta v_n}
=
\frac{1}{\eta}\log\sum_{n\in\mathcal{N}} e^{-\eta (V_b(n;\theta)+\mu n)}.
\]
Negating both sides turns the maximization into the minimization in \eqref{eq:softmin_value}.
The maximizer in \Cref{lem:lse_variational} becomes the minimizer here and equals \eqref{eq:rollout_softmin_dist}.
\end{proof}

\begin{corollary}[Soft-min approximation quality]
\label{cor:softmin_gap}
Let $\mathrm{smin}_\eta(z)\triangleq -\frac{1}{\eta}\log\sum_{i=1}^K e^{-\eta z_i}$.
Then for any $z\in\mathbb{R}^K$,
\begin{equation}
\label{eq:softmin_gap}
\mathrm{smin}_\eta(z)\;\le\;\min_i z_i\;\le\;\mathrm{smin}_\eta(z)+\frac{\log K}{\eta}.
\end{equation}
\end{corollary}

\begin{proof}
Apply \Cref{cor:lse_max_gap} to $v=-z$.
\end{proof}

\begin{remark}[Shadow price and the ``staircase'' compute pattern]
\label{rem:shadow_price}
\Cref{lem:softmin_rollout} makes explicit that $\mu$ acts as a \emph{shadow price} on rollouts:
increasing $\mu$ shifts probability mass in \eqref{eq:rollout_softmin_dist} toward smaller rollout counts $n$.
In our implementation, $\mu$ is updated by (projected) dual ascent on the budget violation, so the system dynamically finds a price at which the expected
rollout budget is approximately satisfied.
This primal--dual viewpoint also clarifies the connection to the square-root law.
If we set $V_b(n;\theta)=v_b(\theta)/n$ and temporarily relax $n$ to be continuous, then the unregularized best response to a fixed price $\mu$ solves
\(\min_{n>0}\, v_b(\theta)/n + \mu n\), whose minimizer is $n_b(\mu)=\sqrt{v_b(\theta)/\mu}$.
Choosing $\mu$ to satisfy the mean-budget constraint recovers the closed-form allocation in \Cref{thm:sqrt_allocation}.
With a finite discrete arm set $\mathcal{N}$, the minimizing arm is the nearest available rollout count to $\sqrt{v_b(\theta)/\mu}$, so as $\mu$
changes the optimizer can switch abruptly between arms.
This explains the ``staircase'' patterns observed in the rollout-allocation figures.
See \Cref{fig:rollout_mechanism,fig:rollout_snapshots} for these staircase-like transitions between discrete rollout arms as the learned shadow price $\mu$ changes.
\end{remark}

\subsubsection{No-regret primal--dual analysis for Rollout-GDRO}
\label{sec:theory_rollout_noregret}

The previous subsection gave a (regularized) Lagrangian view of the rollout controller.
Here we show that the same structure admits a clean \emph{no-regret} interpretation, mirroring the Prompt-GDRO analysis in \Cref{sec:theory_prompt_gdro}.

Throughout this subsection, we idealize the rollout controller as operating on a \emph{fixed} rollout-cost landscape,
i.e., we condition on a fixed policy parameter $\theta$ and fixed group fractions $\bar q\in\Delta_B$.
This subsection analyzes Rollout-GDRO as an independent online budget-allocation game for fixed $(\theta,\bar q)$.
(We comment on the coupled, multi-time-scale setting in \Cref{rem:coupling_prompt_rollout}.)

\paragraph{A truncated Lagrangian game.}
Let $\mathcal{N}=\{n_{\min},\dots,n_{\max}\}$ be the discrete set of rollout arms and denote $K\triangleq|\mathcal{N}|$.
For each bin $b$, the rollout controller chooses a distribution $p_b\in\Delta_{\mathcal{N}}$ over arms; write $p=(p_1,\dots,p_B)$.
Let $V_b(n;\theta)\ge 0$ denote the (idealized) variance proxy for choosing $n$ rollouts in bin $b$ at policy $\theta$ (e.g., $V_b(n;\theta)=v_b(\theta)/n$ as in \Cref{eq:var_opt_problem}).
Let $a(n)\triangleq n$ denote the compute cost of arm $n$.
We consider the convex--concave ``truncated'' Lagrangian
\begin{equation}
\label{eq:rollout_truncated_game}
L_{\mathrm{roll}}(p,\mu;\theta)
\;\triangleq\;
\sum_{b=1}^B \bar q_b\,\mathbb{E}_{n\sim p_b}\!\left[V_b(n;\theta)\right]
\;+\;
\mu\left(\sum_{b=1}^B \bar q_b\,\mathbb{E}_{n\sim p_b}\!\left[a(n)\right]-\bar n\right),
\qquad
\mu\in[0,\mu_{\max}],
\end{equation}
where $\mu$ is the shadow price and $\mu_{\max}$ is a chosen truncation level (projection radius).
When $\mu_{\max}$ is large, maximizing over $\mu$ approximately enforces the budget constraint.

\paragraph{Primal--dual updates.}
A natural no-regret dynamics for \eqref{eq:rollout_truncated_game} is:
\begin{align}
p_{t+1,b}
&=
\argmin_{p_b\in\Delta_{\mathcal{N}}}
\left\{
\left\langle p_b,\, \ell_{t,b}\right\rangle
+\frac{1}{\eta_p}\mathrm{KL}(p_b\,\|\,p_{t,b})
\right\},
\qquad
\ell_{t,b}(n)\triangleq \bar q_b\big(V_b(n;\theta)+\mu_t a(n)\big),
\label{eq:rollout_primal_update}
\\
\mu_{t+1}
&=
\Pi_{[0,\mu_{\max}]}\!\left(\mu_t+\eta_\mu\,g_t\right),
\qquad
g_t\triangleq \sum_{b=1}^B \bar q_b\,\mathbb{E}_{n\sim p_{t,b}}[a(n)]-\bar n,
\label{eq:rollout_dual_update}
\end{align}
where $\Pi_{[0,\mu_{\max}]}$ is Euclidean projection and $\mathrm{KL}(\cdot\|\cdot)$ is the KL divergence.
The update \eqref{eq:rollout_primal_update} is entropic mirror descent (a.k.a.\ exponentiated gradient) on the simplex, applied independently in each bin.
\Cref{lem:softmin_rollout} shows that for fixed $\mu$, the minimizer has a Gibbs / soft-min form; the dynamics \eqref{eq:rollout_primal_update} tracks this solution online as $\mu_t$ evolves.

\begin{theorem}[No-regret guarantee for Rollout-GDRO's primal--dual controller]
\label{thm:rollout_noregret}
Assume that for all $b$ and $n\in\mathcal{N}$ we have $0\le V_b(n;\theta)\le V_{\max}$ and $0\le a(n)\le a_{\max}$.
Let $K=|\mathcal{N}|$ and initialize $p_{1,b}$ to the uniform distribution on $\mathcal{N}$ and $\mu_1=0$.
Run the updates \eqref{eq:rollout_primal_update}--\eqref{eq:rollout_dual_update} for $T$ steps with step sizes $\eta_p,\eta_\mu>0$.

Define the averaged iterates $\bar p_b\triangleq \frac{1}{T}\sum_{t=1}^T p_{t,b}$ and $\bar\mu\triangleq \frac{1}{T}\sum_{t=1}^T \mu_t$.
Then the (truncated) saddle-point gap satisfies
\begin{equation}
\label{eq:rollout_saddle_gap}
\max_{\mu\in[0,\mu_{\max}]}\;L_{\mathrm{roll}}(\bar p,\mu;\theta)
\;-\;
\min_{p\in(\Delta_{\mathcal{N}})^B}\;L_{\mathrm{roll}}(p,\bar\mu;\theta)
\;\le\;
\frac{B\log K}{\eta_p T}
+\frac{\eta_p}{8}\,(V_{\max}+\mu_{\max}a_{\max})^2
+\frac{\mu_{\max}^2}{2\eta_\mu T}
+\frac{\eta_\mu}{2}\,a_{\max}^2.
\end{equation}
In particular, the explicit step sizes
\begin{equation}
\label{eq:rollout_stepsizes_explicit}
\eta_p \;\triangleq\; \frac{\sqrt{8B\log K}}{(V_{\max}+\mu_{\max}a_{\max})\sqrt{T}},
\qquad
\eta_\mu \;\triangleq\; \frac{\mu_{\max}}{a_{\max}\sqrt{T}},
\end{equation}
yield the concrete bound
\begin{equation}
\label{eq:rollout_gap_explicit}
\mathrm{Gap}_T
\;\le\;
(V_{\max}+\mu_{\max}a_{\max})\,\sqrt{\frac{B\log K}{2T}}
\;+
\frac{\mu_{\max}a_{\max}}{\sqrt{T}}.
\end{equation}

Moreover, let $p^\star$ be an optimal solution of the \emph{budgeted variance} problem
\begin{equation}
\label{eq:rollout_constrained_problem}
\min_{p\in(\Delta_{\mathcal{N}})^B}\;\sum_{b=1}^B \bar q_b\,\mathbb{E}_{n\sim p_b}\!\left[V_b(n;\theta)\right]
\qquad\text{s.t.}\qquad
\sum_{b=1}^B \bar q_b\,\mathbb{E}_{n\sim p_b}\!\left[a(n)\right]\le \bar n.
\end{equation}
Then the averaged rollout policy $\bar p$ is nearly optimal and nearly feasible:
\begin{align}
\sum_{b=1}^B \bar q_b\,\mathbb{E}_{n\sim \bar p_b}\!\left[V_b(n;\theta)\right]
-
\sum_{b=1}^B \bar q_b\,\mathbb{E}_{n\sim p_b^\star}\!\left[V_b(n;\theta)\right]
&\le
\mathrm{Gap}_T,
\label{eq:rollout_obj_gap}
\\
\Bigg[\sum_{b=1}^B \bar q_b\,\mathbb{E}_{n\sim \bar p_b}\!\left[a(n)\right]-\bar n\Bigg]_+
&\le
\frac{\sum_{b=1}^B \bar q_b\,\mathbb{E}_{n\sim p_b^\star}\!\left[V_b(n;\theta)\right]+\mathrm{Gap}_T}{\mu_{\max}}
\;\le\;\frac{V_{\max}+\mathrm{Gap}_T}{\mu_{\max}},
\label{eq:rollout_budget_gap}
\end{align}
where $\mathrm{Gap}_T$ denotes the right-hand side of \eqref{eq:rollout_saddle_gap}.
\end{theorem}

\begin{proof}
We apply a standard two-player no-regret-to-equilibrium argument to the convex--concave game \eqref{eq:rollout_truncated_game}; see, e.g., \citealp{cesa2006prediction,hazan2016introduction,bubeck2015convex} for the standard regret bounds and the conversion to saddle-point guarantees. 

\paragraph{Step 1: primal regret bound (entropic mirror descent).}
Fix a bin $b$.
The primal update \eqref{eq:rollout_primal_update} is entropic mirror descent on $\Delta_{\mathcal{N}}$ with loss vector $\ell_{t,b}\in\mathbb{R}^{K}$.
Because $0\le V_b(n;\theta)\le V_{\max}$, $0\le a(n)\le a_{\max}$, and $0\le \mu_t\le \mu_{\max}$, each coordinate satisfies
\(0\le \ell_{t,b}(n)\le \bar q_b\,(V_{\max}+\mu_{\max}a_{\max})\).
The standard entropic-mirror-descent regret bound (via the log-sum-exp potential and Hoeffding's lemma) gives, for any fixed comparator $p_b\in\Delta_{\mathcal{N}}$,
\begin{equation}
\label{eq:rollout_primal_regret_single}
\sum_{t=1}^T \langle p_{t,b},\ell_{t,b}\rangle
-
\sum_{t=1}^T \langle p_b,\ell_{t,b}\rangle
\le
\frac{\log K}{\eta_p}
+\frac{\eta_p}{8}\,T\,\bar q_b^2\,(V_{\max}+\mu_{\max}a_{\max})^2.
\end{equation}
Summing \eqref{eq:rollout_primal_regret_single} over $b=1,\dots,B$ yields
\begin{align}
\label{eq:rollout_primal_regret_total}
\sum_{t=1}^T L_{\mathrm{roll}}(p_t,\mu_t;\theta)
-
\sum_{t=1}^T L_{\mathrm{roll}}(p,\mu_t;\theta)
&\le
\frac{B\log K}{\eta_p}
+\frac{\eta_p}{8}\,T\,(V_{\max}+\mu_{\max}a_{\max})^2\sum_{b=1}^B \bar q_b^2 \\
&\le
\frac{B\log K}{\eta_p}
+\frac{\eta_p}{8}\,T\,(V_{\max}+\mu_{\max}a_{\max})^2,
\end{align}
the last inequality holds due to the algebraic fact that $a^2+\cdots+b^2\le (a+\cdots+b)^2 \le 1$ for any probability measure $(a,\cdots,b)\in\Delta$,
where $p=(p_1,\dots,p_B)$ and $p_t=(p_{t,1},\dots,p_{t,B})$.

\paragraph{Step 2: dual regret bound (projected gradient ascent).}
The dual player maximizes $L_{\mathrm{roll}}(p_t,\mu;\theta)$ over $\mu\in[0,\mu_{\max}]$.
Since $L_{\mathrm{roll}}(p_t,\mu;\theta)$ is linear in $\mu$, the dual gradient is exactly $g_t$ from \eqref{eq:rollout_dual_update}.
Moreover, because $0\le a(n)\le a_{\max}$ and $\bar q\in\Delta_B$, we have $|g_t|\le a_{\max}$.
The standard regret bound for projected online gradient ascent on an interval gives, for any fixed $\mu\in[0,\mu_{\max}]$,
\begin{equation}
\label{eq:rollout_dual_regret}
\sum_{t=1}^T L_{\mathrm{roll}}(p_t,\mu;\theta)
-
\sum_{t=1}^T L_{\mathrm{roll}}(p_t,\mu_t;\theta)
\le
\frac{\mu_{\max}^2}{2\eta_\mu}
+\frac{\eta_\mu}{2}\,T\,a_{\max}^2.
\end{equation}

\paragraph{Step 3: combine regrets and average.}
Adding \eqref{eq:rollout_primal_regret_total} and \eqref{eq:rollout_dual_regret} and dividing by $T$ yields, for all $p$ and $\mu$,
\begin{equation}
\label{eq:rollout_gap_time_average}
\frac{1}{T}\sum_{t=1}^T L_{\mathrm{roll}}(p_t,\mu;\theta)
-
\frac{1}{T}\sum_{t=1}^T L_{\mathrm{roll}}(p,\mu_t;\theta)
\le
\mathrm{Gap}_T.
\end{equation}

Because $L_{\mathrm{roll}}(\cdot,\mu;\theta)$ is convex in $p$ and $L_{\mathrm{roll}}(p,\cdot;\theta)$ is linear in $\mu$,
Jensen's inequality gives
\[
L_{\mathrm{roll}}(\bar p,\mu;\theta)
\le
\frac{1}{T}\sum_{t=1}^T L_{\mathrm{roll}}(p_t,\mu;\theta),
\qquad
\frac{1}{T}\sum_{t=1}^T L_{\mathrm{roll}}(p,\mu_t;\theta)
=
L_{\mathrm{roll}}(p,\bar\mu;\theta).
\]
Substitute into \eqref{eq:rollout_gap_time_average} to obtain \eqref{eq:rollout_saddle_gap}.

\paragraph{Step 4: deduce objective and budget bounds.}
Let $p^\star$ be feasible for \eqref{eq:rollout_constrained_problem}.
For any $\mu\in[0,\mu_{\max}]$, feasibility implies $L_{\mathrm{roll}}(p^\star,\mu;\theta)\le L_{\mathrm{roll}}(p^\star,0;\theta)$.
Thus
\[
\min_{p} L_{\mathrm{roll}}(p,\bar\mu;\theta)
\le
L_{\mathrm{roll}}(p^\star,\bar\mu;\theta)
\le
L_{\mathrm{roll}}(p^\star,0;\theta).
\]
Combining with \eqref{eq:rollout_saddle_gap} yields
\[
\max_{\mu\in[0,\mu_{\max}]} L_{\mathrm{roll}}(\bar p,\mu;\theta)
\le
L_{\mathrm{roll}}(p^\star,0;\theta)+\mathrm{Gap}_T.
\]
Finally, observe that
\[
\max_{\mu\in[0,\mu_{\max}]}\;L_{\mathrm{roll}}(\bar p,\mu;\theta)
=
\sum_{b=1}^B \bar q_b\,\mathbb{E}_{n\sim \bar p_b}[V_b(n;\theta)]
\;+\;
\mu_{\max}\Bigg[\sum_{b=1}^B \bar q_b\,\mathbb{E}_{n\sim \bar p_b}[a(n)]-\bar n\Bigg]_+,
\]
which immediately implies \eqref{eq:rollout_obj_gap} and \eqref{eq:rollout_budget_gap}.
\end{proof}

\begin{remark}[Bandit feedback and EXP3-style updates]
\label{rem:rollout_bandit}
Our implementation uses EXP3P-style updates because Rollout-GDRO only observes the variance proxy corresponding to the chosen arm.
The full-information analysis above can be extended to the bandit setting by replacing $\ell_{t,b}$ with an unbiased importance-weighted estimator and
adding the usual $\sqrt{K}$ factor in the regret term.
We omit the (standard) bandit algebra here, since the qualitative conclusion remains the same: the rollout controller is a no-regret player in a budgeted Lagrangian game.
\end{remark}

\begin{remark}[Decoupled analysis and possible coupling in the full system]
\label{rem:coupling_prompt_rollout}
See the Limitations and Future Work section for further discussion of this coupled setting and open problems it raises.
In this paper we analyze and evaluate \emph{Prompt-GDRO} and \emph{Rollout-GDRO} as \emph{separate} controllers.
Accordingly, in the Rollout-GDRO analysis we treat the group fractions $\bar q\in\Delta_B$ as \emph{exogenous} (a property of the current
training data pipeline and grouping rule), and we study how the rollout allocator responds to group-dependent variance proxies under the
mean-budget constraint $\sum_{b=1}^B \bar q_b\,n_b=\bar n$.
If both controllers are enabled simultaneously, then $\bar q$ becomes endogenous to the Prompt-GDRO sampler and the training loop induces a
coupled multi-time-scale game between the prompt sampler, the rollout allocator, and the learner.
Characterizing stability and convergence of this coupled regime is left to future work.
\end{remark}
\begin{remark}[Practical proxies for $v_b(\theta)$ in GRPO and connection to our implementation]
\label{rem:vb_proxy}
The quantity $v_b(\theta)$ in Lemma~\ref{lem:batch_var} is a group-dependent second-moment / variance proxy for a \emph{single-prompt}
stochastic gradient contribution.
It enters the idealized variance-aware relaxation \eqref{eq:var_opt_problem} only through the characteristic $1/n_b$ scaling
obtained when using $n_b$ rollouts (Lemma~\ref{lem:rollout_avg}).
In our Rollout-GDRO implementation (Section~\ref{sec:analysis}), each bin $b$ selects a discrete rollout arm
$n\in\{n_{\min},\dots,n_{\max}\}$ using a GDRO-EXP3P bandit update driven by the augmented arm loss
$\mathcal{L}_b(n)=\widehat{L}_b(\theta;n)+\mu\,(n-\bar n)$ (Eq.~\eqref{eq:rollout_arm_loss}), while the dual variable $\mu$ is updated by
dual ascent to enforce compute neutrality under the mean-budget constraint (Eq.~\eqref{eq:budget_constraint}).
Operationally, $v_b(\theta)$ can be estimated from the same rollouts used to compute $\widehat{L}_b(\theta;n)$.
For a prompt $x$ with $n$ rollouts $\{y_j\}_{j=1}^n$, let $g_j(x;\theta)$ denote the per-rollout GRPO policy-gradient contribution.
A natural within-prompt proxy is the sample variance (biased vs.\ unbiased only changes constants)
\[
\widehat{\mathrm{Var}}(g\mid x)
\;\triangleq\;
\frac{1}{n-1}\sum_{j=1}^n \big\|g_j(x;\theta)-\bar g(x;\theta)\big\|_2^2,
\qquad
\bar g(x;\theta)=\frac{1}{n}\sum_{j=1}^n g_j(x;\theta),
\]
which can be aggregated over prompts in bin $b$ to form a bin-level proxy $\hat v_b(\theta)$.
In practice we use a monotone scalar surrogate of this signal (e.g., the empirical variance of reward or advantage across rollouts within each bin)
to guide arm selection, consistent with the variance-reduction viewpoint developed in \Cref{sec:theory_variance_proxy}--\Cref{sec:theory_primal_dual}.
\end{remark}

\end{document}